\documentclass[11pt]{article} 

\pdfoutput=1

\usepackage{etoolbox}
\newtoggle{arxiv}
\toggletrue{arxiv}

\usepackage[utf8]{inputenc}
\usepackage{mathrsfs}
\usepackage{microtype}
\usepackage{subfigure}
\usepackage{booktabs} 
\usepackage{longtable,makecell}
\usepackage{amsmath, amsfonts, amsthm, amssymb, graphicx}
\usepackage{mathtools,verbatim}
\usepackage{enumitem}
\usepackage{bm}
\usepackage{thm-restate}
\usepackage{natbib}
\usepackage{xspace,upgreek}
\usepackage[dvipsnames]{xcolor}

\usepackage{color-edits}

\addauthor{df}{ForestGreen}
\addauthor{ms}{orange}
\addauthor{kevin}{ForestGreen}

\usepackage[noend]{algpseudocode}
\usepackage{algorithm}

\iftoggle{arxiv}
{
	\usepackage[scaled=.90]{helvet}
	
	\usepackage[vmargin=1.00in,hmargin=1.00in,centering,letterpaper]{geometry}
	\usepackage{fancyhdr, lastpage}

	\setlength{\headsep}{.10in}
	\setlength{\headheight}{15pt}
	\cfoot{}
	\lfoot{}
	\rfoot{\sc Page\ \thepage\ of\ \protect\pageref{LastPage}}

}

\usepackage{hyperref}
\usepackage{cleveref}
\hypersetup{colorlinks,citecolor=blue}
\newtoggle{upperbound}
\togglefalse{upperbound}

\newcommand{\poly}{\mathrm{poly}}

\numberwithin{equation}{section}

\newcommand{\fro}{\mathrm{F}}
\newcommand{\op}{\mathrm{op}}

{
 \theoremstyle{plain}
      \newtheorem{asm}{Assumption}
}
\creflabelformat{asm}{#2#1#3}

\theoremstyle{plain}
\newtheorem{thm}{Theorem}
\newtheorem{lem}{Lemma}[section]
\newtheorem{cor}{Corollary}
\newtheorem{claim}[lem]{Claim}

\newtheorem{prop}[thm]{Proposition}
\theoremstyle{definition}
\newtheorem{defn}{Definition}[section]

\newtheorem{rem}{Remark}[section]

\newtheorem{protocol}{Protocol}[section]

\renewcommand{\Pr}{\mathbb{P}}
\newcommand{\Exp}{\mathbb{E}}
\newcommand{\Var}{\mathrm{Var}}

\newcommand{\KL}{\mathrm{KL}}

\newcommand{\rmd}{\mathrm{d}}

\newcommand{\tr}{\mathrm{tr}}

\newcommand{\diag}{\mathrm{diag}}

\newcommand{\dist}{\mathrm{dist}}

\newcommand{\R}{\mathbb{R}}
\newcommand{\I}{\mathbb{I}}

\DeclareMathOperator*{\argmax}{arg\,max}
\DeclareMathOperator*{\argmin}{arg\,min}

\def\ddefloop#1{\ifx\ddefloop#1\else\ddef{#1}\expandafter\ddefloop\fi}
\def\ddef#1{\expandafter\def\csname bb#1\endcsname{\ensuremath{\mathbb{#1}}}}
\ddefloop ABCDEFGHIJKLMNOPQRSTUVWXYZ\ddefloop

\def\ddefloop#1{\ifx\ddefloop#1\else\ddef{#1}\expandafter\ddefloop\fi}
\def\ddef#1{\expandafter\def\csname fr#1\endcsname{\ensuremath{\mathfrak{#1}}}}
\ddefloop ABCDEFGHIJKLMNOPQRSTUVWXYZ\ddefloop
\def\ddefloop#1{\ifx\ddefloop#1\else\ddef{#1}\expandafter\ddefloop\fi}
\def\ddef#1{\expandafter\def\csname scr#1\endcsname{\ensuremath{\mathscr{#1}}}}
\ddefloop ABCDEFGHIJKLMNOPQRSTUVWXYZ\ddefloop
\def\ddefloop#1{\ifx\ddefloop#1\else\ddef{#1}\expandafter\ddefloop\fi}
\def\ddef#1{\expandafter\def\csname b#1\endcsname{\ensuremath{\mathbf{#1}}}}
\ddefloop ABCDEFGHIJKLMNOPQRSTUVWXYZ\ddefloop
\def\ddef#1{\expandafter\def\csname c#1\endcsname{\ensuremath{\mathcal{#1}}}}
\ddefloop ABCDEFGHIJKLMNOPQRSTUVWXYZ\ddefloop
\def\ddef#1{\expandafter\def\csname h#1\endcsname{\ensuremath{\widehat{#1}}}}
\ddefloop ABCDEFGHIJKLMNOPQRSTUVWXYZ\ddefloop
\def\ddef#1{\expandafter\def\csname t#1\endcsname{\ensuremath{\widetilde{#1}}}}
\ddefloop ABCDEFGHIJKLMNOPQRSTUVWXYZ\ddefloop

\def\ddefloop#1{\ifx\ddefloop#1\else\ddef{#1}\expandafter\ddefloop\fi}
\def\ddef#1{\expandafter\def\csname mat#1\endcsname{\ensuremath{\mathbf{#1}}}}
\ddefloop abcdefgijklmnopqrstuvwxyzABCDEFGHIJKLMNOPQRSTUVWXYZ\ddefloop

\newcommand{\force}{\textsc{Force}\xspace}

\newcommand{\Vhat}{\widehat{V}}
\newcommand{\Vst}{V^\star}
\newcommand{\Vpi}{V^\pi}

\newcommand{\Qst}{Q^\star}

\newcommand{\pist}{\pi^\star}
\newcommand{\pihat}{\widehat{\pi}}

\newcommand{\wpi}{w^\pi}

\newcommand{\Qpi}{Q^\pi}

\newcommand{\simplex}{\bigtriangleup}

\newcommand{\bLamexp}{\bLambda_{\mathrm{exp}}}

\newcommand{\Delmin}{\Delta_{\min}}

\newcommand{\Pitil}{\widetilde{\Pi}}

\newcommand{\wst}{\bw^\star}

\newcommand{\cOtil}{\widetilde{\cO}}

\newcommand{\pitilst}{\widetilde{\pi}^{\star}}

\newcommand{\frakD}{\mathfrak{D}}

\newcommand{\atil}{\widetilde{a}}

\newcommand{\pitil}{\widetilde{\pi}}

\newcommand{\zst}{\bm{z}^\star}

\newcommand{\algname}{\textsc{Pedel}\xspace}
\newcommand{\tsum}{{\textstyle \sum}}

\newcommand{\inner}[2]{\langle #1, #2 \rangle}
\newcommand{\innerb}[2]{\langle #1, #2 \rangle}

\newcommand{\bx}{\bm{x}}

\newcommand{\Ball}{\mathcal{B}}
\newcommand{\bSigma}{\bm{\Sigma}}

\newcommand{\bvtil}{\widetilde{\bv}}

\newcommand{\bLambda}{\bm{\Lambda}}
\newcommand{\bv}{\bm{v}}
\newcommand{\bphi}{\bm{\phi}}

\newcommand{\btheta}{\bm{\theta}}

\newcommand{\bw}{\bm{w}}

\newcommand{\Fclass}{\mathscr{F}}

\newcommand{\bLamtil}{\widetilde{\bm{\Lambda}}}
\newcommand{\bthetast}{\btheta_{\star}}
\newcommand{\bthetahat}{\widehat{\btheta}}
\newcommand{\bu}{\bm{u}}

\newcommand{\bmu}{\bm{\mu}}

\newcommand{\bwtil}{\widetilde{\bw}}

\newcommand{\xst}{\bx^\star}

\newcommand{\bphihat}{\widehat{\bphi}}

\newcommand{\be}{\bm{e}}
\newcommand{\sbar}{\bar{s}}

\newcommand{\bphitil}{\widetilde{\bphi}}

\newcommand{\bern}{\mathrm{Bernoulli}}

\newcommand{\regmin}{\textsc{RegMin}\xspace}
\newcommand{\piexp}{\pi_{\mathrm{exp}}}

\newcommand{\cThat}{\widehat{\cT}}

\newcommand{\bLamhat}{\widehat{\bLambda}}

\newcommand{\lamminst}{\lambda_{\min}^\star}

\newcommand{\Nst}{N^\star}

\newcommand{\bGamma}{\bm{\Gamma}}

\newcommand{\piw}{\pi^{\bw}}
\newcommand{\piu}{\pi^{\bu}}
\newcommand{\piwst}{\pi^{\wst}}
\newcommand{\lammin}{\lambda_{\min}}

\newcommand{\nur}{\nu}
\newcommand{\bGamhat}{\widehat{\bGamma}}
\newcommand{\optcov}{\textsc{OptCov}\xspace}
\newcommand{\psdmat}{\mathbb{S}_{+}}
\newcommand{\fmin}{f_{\min}}
\newcommand{\ihat}{\widehat{i}}
\newcommand{\Gopt}{\mathsf{XY}_{\mathrm{opt}}}
\newcommand{\Gopts}{\widetilde{\mathsf{XY}}_{\mathrm{opt}}}
\newcommand{\LSE}{\mathrm{LogSumExp}}
\newcommand{\bLambar}{\bar{\bLambda}}
\newcommand{\bLamst}{\bLambda^\star}
\newcommand{\Ktil}{\widetilde{K}}

\newcommand{\fwregret}{\textsc{FWRegret}\xspace}

\newcommand{\bSigtil}{\widetilde{\bSigma}}
\newcommand{\bSigbar}{\bar{\bSigma}}
\newcommand{\bSighat}{\widehat{\bSigma}}

\newcommand{\bOmega}{\bm{\Omega}}

\newcommand{\lamun}{\underline{\lambda}}
\newcommand{\condcov}{\textsc{ConditionedCov}\xspace}
\newcommand{\cEest}{\cE_{\mathrm{est}}}
\newcommand{\cEexp}{\cE_{\mathrm{exp}}}
\newcommand{\Delbar}{\bar{\Delta}}

\newcommand{\epsexp}{\epsilon_{\mathrm{exp}}}

\newcommand{\Pieps}{\Pi_{\epsilon}}
\newcommand{\iotaalg}{\iota_0}

\newcommand{\Delbarmin}{\bar{\Delta}_{\min}}
\newcommand{\cAtil}{\widetilde{\cA}}
\newcommand{\pitilw}{\pitil^{\bw}}
\newcommand{\by}{\bm{y}}
\newcommand{\gamphi}{\gamma_{\Phi}}
\newcommand{\etatil}{\widetilde{\eta}}
\newcommand{\moca}{\textsc{Moca}\xspace}

\usepackage{wrapfig}

\newlength\tindent
\newcommand{\awcomment}[1]{{\color{red}[AW: #1]}}

\newcommand{\finalrev}[1]{{#1}}
\newcommand{\nipscrnew}[1]{{#1}}
\newcommand{\awarxiv}[1]{}
\allowdisplaybreaks
\makeatletter
\providecommand\theHALG@line{\thealgorithm.\arabic{ALG@line}}
\makeatother

\title{Instance-Dependent Near-Optimal Policy Identification in Linear MDPs via Online Experiment Design}

\author{Andrew Wagenmaker\footnote{University of Washington, Seattle. Email: \texttt{ajwagen@cs.washington.edu}} \and Kevin Jamieson\footnote{University of Washington, Seattle. Email: \texttt{jamieson@cs.washington.edu}}}
\date{}

\begin{document}

\maketitle

\begin{abstract}

While much progress has been made in understanding the minimax sample complexity of reinforcement learning (RL)---the complexity of learning on the ``worst-case'' instance---such measures of complexity often do not capture the true difficulty of learning. In practice, on an ``easy'' instance, we might hope to achieve a complexity far better than that achievable on the worst-case instance. In this work we seek to understand the ``instance-dependent'' complexity of learning near-optimal policies (PAC RL) in the setting of RL with linear function approximation. We propose an algorithm, \algname, which achieves a fine-grained instance-dependent measure of complexity, the first of its kind in the RL with function approximation setting, thereby capturing the difficulty of learning on each particular problem instance. 
Through an explicit example, we show that \algname yields provable gains over low-regret, minimax-optimal algorithms and that such algorithms are unable to hit the instance-optimal rate.
Our approach relies on a novel online experiment design-based procedure which focuses the exploration budget on the ``directions'' most relevant to learning a near-optimal policy, and may be of independent interest.

\end{abstract}


\setcounter{footnote}{0}
\section{Introduction}

In the PAC (Probably Approximately Correct) reinforcement learning (RL) setting, an agent is tasked with exploring an unknown environment in order to learn a policy which maximizes the amount of reward collected. In general, we are interested in learning such a policy using as few interactions with the environment (as small sample complexity) as possible. We might hope that the number of samples needed would scale with the difficulty of identifying a near-optimal policy in our particular environment. For example, in a ``hard'' environment, we would expect that more samples might be required, while in an ``easy'' environment, fewer samples may be needed. 

The RL community has tended to focus on developing algorithms which have near-optimal \emph{worst-case} sample complexity---sample complexities that are only guaranteed to be optimal on ``hard'' instances. Such algorithms typically have complexities which scale, for example, as $\cO(\poly(d,H)/\epsilon^2)$, for $d$ the dimensionality of the environment, $H$ the horizon, and $\epsilon$ the desired level of optimality. While we may be able to show this complexity is optimal on a hard instance, it is unable to distinguish between ``hard'' and ``easy'' problems. The scaling is identical for two environments as long as the dimensionality and horizon of each are the same---no consideration is given to the actual difficulty of the problem---and we therefore have no guarantee that our algorithm is solving the problem with complexity scaling as the actual difficulty. Indeed, as recent work has shown \citep{wagenmaker2021beyond}, this is not simply an analysis issue: worst-case optimal algorithms can be very suboptimal on ``easy'' instances. 

Towards developing algorithms which overcome this, we might instead consider the \emph{instance-dependent} difficulty---the hardness of solving a particular problem instance---and hope to obtain a sample complexity scaling with this instance-dependent difficulty, thereby guaranteeing that we solve ``easy'' problems using only a small number of samples, but still obtain the worst-case optimal rate on ``hard'' problems. While progress has been made in understanding the instance-dependent complexity of learning in RL, the results are largely limited to environments with a finite number of states and actions. In practice, real-world RL problems often involve large (even infinite) state-spaces and, in order to solve such problems, we must generalize across states. To handle such settings, the RL community has turned to function approximation-based methods, which allow for provable learning in large state-space environments. However, while worst-case optimal results have been shown, little is understood on the instance-dependent complexity of learning in these settings.

In this work we aim to bridge this gap. We consider, in particular, the linear MDP setting, and develop an algorithm which provably learns a near-optimal policy with sample complexity scaling as the difficulty of each individual instance. Furthermore, by comparing to our instance-dependent measure of complexity, we show that low-regret algorithms are provably suboptimal for PAC RL in function approximation settings. Our algorithm relies on a novel online experiment design-based procedure---adapting classical techniques from linear experiment design to settings where \emph{navigation} is required to measure a particular covariate---which may be of independent interest.

\subsection{Contributions}
Our contributions are as follows:
\begin{itemize}[leftmargin=*]
\item We propose an algorithm, \algname, which learns an $\epsilon$-optimal policy with instance-dependent sample complexity scaling as (up to $H$ factors):
\begin{align*}
\sum_{h=1}^H  \inf_{\piexp} \max_{\pi \in \Pi} \frac{ \| \bphi_{\pi,h} \|_{\bLambda_{\piexp,h}^{-1}}^2}{(\Vst_0 - V_0^\pi)^2 \vee \epsilon^2} \cdot \Big (d + \log \frac{1}{\delta} \Big )
\end{align*}
for $\bphi_{\pi,h}$ the ``average feature vector'' of policy $\pi$, $\bLambda_{\piexp,h}$ the expected covariance of the policy $\piexp$, and $\Vst_0 - V_0^\pi$ the ``policy gap''. We show that \algname also has worst-case optimal dimension-dependence---its sample complexity never exceeds $\cOtil(d^2 H^6/\epsilon^2)$---but that on ``easy'' instances it achieves complexity much smaller than the worst-case optimal rate. 
\item It is well-known that low-regret algorithms achieve the worst-case optimal rate for PAC RL. We construct an explicit example, however, where the instance-dependent complexity of \algname improves on the complexity of \emph{any} low-regret algorithm by a factor of the dimensionality, providing the first evidence that low-regret algorithms are provably suboptimal on ``easy'' instances for PAC RL in function approximation settings. 
\item We develop a general experiment design-based approach to exploration in MDPs, which allows us to focus our exploration in the directions most relevant to learning near-optimal policies. Our approach is based on the key observation that, while solving an experiment design in an MDP would require knowledge of the MDP dynamics, we can approximately solve one without knowledge of the dynamics by running a regret minimization algorithm on a carefully chosen reward function, inducing the correct exploration. \iftoggle{arxiv}{We apply our experiment design approach to efficiently explore our MDP so as to identify near-optimal policies, but show that it can also be used to collect observations minimizing much more general experiment design objective functions.}{}
\end{itemize}

\section{Related Work}
The sample complexity of RL has been studied for decades \citep{kearns1998finite,brafman2002r,kakade2003sample}. The two primary problems considered are the regret minimization problem (where the goal is to obtain large online reward) and the PAC policy identification problem (where the goal is to find a near-optimal policy using as few samples as possible), which is the focus of this work.
In the tabular RL setting, the question of obtaining worst-case optimal algorithms is nearly closed \citep{dann2015sample,dann2019policy,menard2020fast,zhang2020reinforcement}. As such, in this section we focus primarily on results in the RL with function approximation literature, as well as results on instance-dependent RL.

\paragraph{Sample-Efficient RL with Linear Function Approximation.}
To generalize beyond MDPs with a finite number of states and acions, the RL community has considered \emph{function approximation}, replacing the tabular model with more powerful settings that allow for generalization across states. Such settings have been considered in classical works \citep{baird1995residual,bradtke1996linear,sutton1999policy,melo2007q}, yet these works do not provide polynomial sample complexities. More recently, there has been intense interest in obtaining polynomial complexities for general function classes \citep{jiang2017contextual,du2021bilinear,jin2021bellman,foster2021statistical}, and, in particular, linear function classes \citep{yang2019sample,jin2020provably,wang2019optimism,du2019good,zanette2020frequentist,zanette2020learning,ayoub2020model,jia2020model,weisz2021exponential,zhou2020nearly,zhou2021provably,zhang2021variance,wang2021exponential}. 

In the linear MDP setting, the state-of-the-art in PAC RL is the work of \cite{wagenmaker2022reward}, which proposes a computationally efficient algorithm achieving a complexity of $\cO(d^2 H^5/\epsilon^2)$ for the more general reward-free RL problem, and shows a matching lower bound of $\Omega(d^2 H^2/\epsilon^2)$ for the PAC RL problem. While this result obtains tight dimension-dependence, it is still worst-case, and offers no insight on the instance-dependent complexity. Other works of note in this category are \citep{jin2020provably,zanette2020learning,zhou2020nearly}, which establish regret guarantees in the setting of linear MDPs and the related setting of linear mixture MDPs. \cite{jin2020provably} and \cite{zanette2020learning} obtain regret guarantees of $\cO(\sqrt{d^3 H^4 K})$ and $\cO(\sqrt{d^2 H^4 K})$, respectively, though the approach of \cite{zanette2020learning} is computationally inefficient. Via an online-to-batch conversion \citep{jin2018q}, these algorithms achieve PAC complexities of $\cO(d^3 H^4/\epsilon^2)$ and $\cO(d^2 H^4/\epsilon^2)$. In the setting of linear mixture MDPs, \cite{zhou2020nearly} show a regret bound of $\cO(\sqrt{d^2 H^3 K})$ and a matching lower bound, yielding the first provably tight and computationally efficient algorithms for RL with function approximation.

\paragraph{Instance-Dependent RL.}
Much of the recent work on instance-dependent RL has focused on the tabular setting. \cite{ok2018exploration} provide an algorithm which achieves asymptotically optimal instance-dependent regret, yet it is computationally inefficient. \cite{simchowitz2019non} show that standard optimistic algorithms achieve regret bounded as $\cO(\sum_{s,a,h} \frac{\log K}{\Delta_h(s,a)})$, for $\Delta_h(s,a)$ the \emph{value-function gap}, a result later refined by \citep{xu2021fine,dann2021beyond}. Obtaining instance-dependent guarantees for policy identification has proved more difficult, yet a variety of results do exist \citep{zanette2019generative,jonsson2020planning,marjani2020best,marjani2021navigating}. In the tabular setting, the most comparable work to ours is that of \cite{wagenmaker2021beyond}, which propose a refined instance-dependent measure of complexity, the \emph{gap-visitation complexity}, and show that it is possible to learn an $\epsilon$-optimal policy with complexity scaling as the gap-visitation complexity. While the gap-visitation complexity is shown to be tight in certain settings, no general lower-bounds exist. Towards obtaining sharp guarantees, \cite{tirinzoni2022near} show that in the simpler setting of deterministic MDPs, a quantity similar in spirit to the gap-visitation complexity is tight, providing matching upper and lower bounds.

In the setting of RL with function approximation, to our knowledge, only two existing works obtain guarantees that would be considered ``instance-dependent''. \cite{wagenmaker2021first} show a ``first-order'' regret bound of $\cO(\sqrt{d^3 H^3 \Vst_0 K})$, where $\Vst_0$ is the value of the \emph{optimal} policy on the particular MDP under consideration. \cite{he2020logarithmic} show that standard optimistic algorithms achieve regret guarantees of $\cO(\frac{d^3 H^5 \log K}{\Delta_{\min}})$ and $\cO(\frac{d^2 H^5 \log^3 K}{\Delta_{\min}})$ in the settings of linear MDPs and linear mixture MDPs, respectively, for $\Delta_{\min}$ the minimum value-function gap. While both these works do obtain instance-dependent results, the instance-dependence is rather coarse, depending on only a single parameter ($\Vst_0$ or $\Delta_{\min}$)---our goal will instead be to obtain more refined instance-dependent guarantees.

\nipscrnew{\paragraph{Experiment Design in Sequential Environments.}
Experiment design is a well-developed subfield of statistics, and a full survey is beyond the scope of this work (see \cite{pukelsheim2006optimal} for an overview). We highlight several works on experiment design in sequential environments that are particularly relevant. First, the work of \cite{fiez2019sequential} achieves the instance-optimal rate for best-arm identification in linear bandits and relies on an adaptive experiment design-based.
The approach of \cite{fiez2019sequential}, as well as the related work of \cite{soare2014best},  provides inspiration for our algorithm---in some sense \algname can be seen as a generalization of the \textsc{Rage} algorithm to problems with horizon greater than 1. 
\finalrev{Concurrent to our work, the \textsc{Rage} algorithm was extended to the setting of contextual bandits in \cite{li2022instance}, achieving the instance-optimal rate.}
Second, the work of \cite{wagenmaker2021task} provides an experiment design-based algorithm in the setting of linear dynamical systems, and show that it hits the optimal instance-dependent rate for learning in such systems. While their results are somewhat more general, they specialize to the problem of identifying a near-optimal controller for the LQR problem---thereby solving the PAC RL problem optimally in the special case of quadratic losses and linear dynamical systems. It is not clear, however, if their approach generalizes beyond linear dynamical systems. Finally, while the current work was in preparation, \cite{mutny2022active} proposed an approach to solving experiment design problems in MDPs. To our knowledge, this is the only existing work that directly considers the problem of experiment design in MDPs. However, they make the simplifying assumption that the transition dynamics are \emph{known}, which essentially reduces their problem to a computational one---in contrast, our approach handles the much more difficult setting of unknown dynamics, and shows that efficient experiment design is possible even in this more difficult setting. \finalrev{We remark as well that our online experiment design approach is somewhat related to several existing algorithms \citep{hazan2019provably,zahavy2021reward}.}
}


\section{Preliminaries}
We let $\| \bphi \|_{\bA}^2 = \bphi^\top \bA \bphi$, $\| \cdot \|_\op$ denote the matrix operator norm (matrix 2-norm), and $\| \cdot \|_{\fro}$ denote the Frobenius norm. \iftoggle{arxiv}{Given some norm $\| \cdot \|$, $\| \cdot \|_*$ denotes the dual norm. $\mathbb{S}^d_+$ denotes the set of PSD matrices in $\R^{d \times d}$.}{} $\cOtil(\cdot)$ hides absolute constants and log factors of the arguments. $\lesssim$ denotes inequality up to constants. $\Exp_\pi$ and $\Pr_\pi$ denote the expectation and probability measure induced by playing some policy $\pi$ in our MDP. We let $\bphi_{h,\tau} := \bphi(s_{h,\tau},a_{h,\tau})$ denote the feature vector encountered at step $h$ of episode $\tau$ (and similarly define $r_{h,\tau}$).

\paragraph{Markov Decision Processes.}
In this work, we study episodic, finite-horizon, time inhomogeneous Markov Decision Processes (MDPs), denoted by a tuple, $\cM = (\cS,\cA,H,\{P_h\}_{h=1}^H,\{\nur_h \}_{h=1}^H)$. We let $\cS$ denote the state space, $\cA$ the action space, $H$ the horizon, $\{P_h\}_{h=1}^H$ the transition kernel, and $\{\nur_h \}_{h=1}^H$ the reward distribution, where $P_h( \cdot |s,a) \in \simplex_{\cS}$ denotes the distribution over the next state when playing action $a$ in state $s$ at step $h$, and $\nur_h(s,a) \in \simplex_{[0,1]}$ denotes the corresponding distribution over reward. We overload notation and let $\nur_h(s,a)$ also refer to the expected reward. We assume that every episode starts in state $s_1$, and that $\{P_h\}_{h=1}^H$ and $\{\nur_h\}_{h=1}^H$ are initially unknown. 

Let $\pi = \{ \pi_h \}_{h=1}^H$ denote a policy mapping states to distributions over actions, $\pi_h : \cS \rightarrow \simplex_\cA$. When $\pi$ is deterministic, we let $\pi_h(s)$ denote the action policy $\pi$ takes at $(s,h)$. An episode begins at state $s_1$. The agent takes action $a_1 \sim \pi_1(s_1)$, transitions to state $s_2 \sim P_1(\cdot | s_1,a_1)$, and receives reward $r_1(s_1,a_1) \sim \nur_1(s_1,a_1)$. In $s_2$, the agent chooses a new action $a_2 \sim \pi_2(s_2)$, and the process repeats. After $H$ steps, the episode terminates, and the agent restarts at $s_1$. 

In general, we are interested in learning policies that collect a large amount of reward. We can quantify the performance of a policy in terms of the \emph{value function}. In particular, the $Q$-value function, $\Qpi_h(s,a)$, denotes the expected reward that will be obtained if we are in state $s$ at step $h$, play action $a$, and then play policy $\pi$ for the remainder of the episode. Formally, $\Qpi_h(s,a) := \Exp_\pi  [ \tsum_{h'=h}^H r_{h'}(s_{h'},a_{h'}) | s_h = s, a_h = a  ]$.
The value function is similarly defined as $\Vpi_h(s) := \Exp_\pi  [ \sum_{h'=h}^H r_{h'}(s_{h'},a_{h'}) | s_h = s ]$. For deterministic policies, $\Vpi_h(s) = \Qpi_h(s,\pi_h(s))$. We denote the optimal $Q$-value function by $\Qst_h(s,a) = \sup_\pi \Qpi_h(s,a)$ and the optimal value function by $\Vst_h(s) = \sup_\pi \Vpi_h(s)$, where the suprema is taken over all policies, both deterministic and stochastic. We define the \emph{value of a policy} as $\Vpi_0 = \Vpi_1(s_1)$---the expected reward policy $\pi$ achieves over an entire episode---and say a policy $\pi$ is optimal if $\Vpi_0 = \Vst_0$. For some set of policies $\Pi$ (which may not contain an optimal policy), we let $\Vst_0(\Pi) := \max_{\pi \in \Pi} \Vpi_0$. We also define the minimum policy gap as:
\begin{align*}
\Delmin^\Pi := \begin{cases} \Vst_0(\Pi) - \max_{\pi \in \Pi : \Vpi_0 < \Vst_0(\Pi)} \Vpi_0 & |\{ \pi \in \Pi : \Vpi_0 = \Vst_0(\Pi) \}| = 1 \\
0 & \text{o.w.} 
\end{cases}
\end{align*}
That is, $\Delmin^\Pi$ is the gap between the values of the best policy and next best policy in $\Pi$ if the best policy in $\Pi$ is unique, and otherwise 0.

\paragraph{PAC Reinforcement Learning.}
In PAC RL, the goal is to identify some policy $\pihat$ using as few episodes as possible, such that, with probability at least $1-\delta$,
\begin{align*}
\Vst_0 - V_0^{\pihat} \le \epsilon.
\end{align*}
We say that such a policy is $\epsilon$-optimal, and an algorithm with such a guarantee on every environment and reward function is $(\epsilon,\delta)$-PAC. We will also refer to this problem as ``policy identification''.

\subsection{Linear MDPs}
In this work, we are interested in the setting where the state space could be infinite, and the learner must generalize across states. In particular, we consider the linear MDP model defined as follows.

\begin{defn}[Linear MDPs \citep{jin2020provably}]\label{defn:linear_mdp}
We say that an MDP is a $d$-\emph{dimensional linear MDP}, if there exists some (known) feature map $\bphi(s,a) : \cS \times \cA \rightarrow \R^d$, $H$ (unknown) signed vector-valued measures $\bmu_h \in \R^d$ over $\cS$, and $H$ (unknown) reward vectors $\btheta_h \in \R^d$, such that:
\begin{align*}
P_h(\cdot | s,a) = \innerb{\bphi(s,a)}{\bmu_h(\cdot)}, \quad \Exp[\nur_h(s,a)] = \innerb{\bphi(s,a)}{\btheta_h}. 
\end{align*}
We will assume $\| \bphi(s,a) \|_2 \le 1$ for all $s,a$; and for all $h$, $\| |\bmu_h|(\cS) \|_2 = \| \int_{s \in \cS} | \rmd \bmu_h(s) | \|_2  \le \sqrt{d}$ and $\| \btheta_h \|_2 \le \sqrt{d}$.
\end{defn}

Linear MDPs encompass, for example, tabular MDPs, but can also model more complex settings, such as feature spaces corresponding to the simplex \citep{jin2020provably}, or the linear bandit problem. Critically, linear MDPs allow for infinite state-spaces, as well as generalization across states---rather than learning the behavior in particular states, we can learn in the $d$-dimensional ambient space.
Note that the standard definition of linear MDPs, for example as given in \cite{jin2020provably}, assumes the rewards are deterministic, while we assume the rewards are random but that their means are linear. We still assume, however, that the random rewards, $r_h(s,a)$, are contained in $[0,1]$ almost surely. 

For a given policy $\pi$, we define the \emph{feature-visitation} at step $h$, the expected feature vector policy $\pi$ encounters at step $h$, as $\bphi_{\pi,h} := \Exp_{\pi}[\bphi(s_h,a_h)]$. 
Note that this is a direct generalization of state-visitations in tabular RL---if our MDP is in fact tabular, $[\bphi_{\pi,h}]_{s,a} = \Pr_\pi[s_h = s, a_h = a]$, so the feature visitation vector corresponds directly to the state visitations. 
Note also that we can write the value of a policy as $\Vpi_0 = \sum_{h=1}^H \inner{\bphi_{\pi,h}}{\btheta_h}$.
Denote the average feature vector induced by $\pi$ in a particular state $s$ as $\bphi_{\pi,h}(s) = \Exp_{a \sim \pi_h(\cdot|s)}[ \bphi(s,a)]$.
We also define $\bLambda_{\pi,h} := \Exp_\pi [ \bphi(s_h,a_h) \bphi(s_h,a_h)^\top]$, the \emph{expected covariance} of policy $\pi$ at step $h$, $\lambda_{\min,h}^\star = \sup_\pi \lammin(\bLambda_{\pi,h})$ the largest achievable minimum eigenvalue at step $h$, and $\lamminst = \min_h \lambda_{\min,h}^\star$. We will make the following assumption.

\begin{asm}[Full Rank Covariates]\label{asm:full_rank_cov}
In our MDP, $\lamminst > 0$.
\end{asm}
\noindent We remark that \Cref{asm:full_rank_cov} is analogous to other explorability assumptions found in the RL with function approximation literature \citep{zanette2020provably,hao2021online,agarwal2021online}.

\iftoggle{arxiv}{
To reduce uncertainty in directions of interest, we will be interested in optimizing over the set of all realizable covariance matrices on our particular MDP. To this end, define 
\begin{align}\label{eq:cov_set_defn}
\bOmega_h := \{ \Exp_{\pi \sim \omega}[\bLambda_{\pi,h}] \ : \ \omega \in \bOmega_\pi \}
\end{align} 
for $\bOmega_\pi$ the set of all valid distributions over Markovian policies (both deterministic and stochastic). $\bOmega_h$ is, then, the set of all covariance matrices realizable by distributions over policies at step $h$.
}{
To reduce uncertainty in directions of interest, we will be interested in optimizing over the set of all realizable covariance matrices on our particular MDP. To this end, define $\bOmega_h := \{ \Exp_{\pi \sim \omega}[\bLambda_{\pi,h}] \ : \ \omega \in \bOmega_\pi \}$ for $\bOmega_\pi$ the set of all valid distributions over Markovian policies (both deterministic and stochastic). We can think of $\bOmega_h$, then, as the set of covariance matrices realizable by distributions over policies at step $h$.}

\section{Near-Optimal Policy Identification in Linear MDPs}\label{sec:main_result}
We are now ready to state our algorithm, \algname.

\begin{algorithm}[h]
\begin{algorithmic}[1]
\State \textbf{input:} tolerance $\epsilon$, confidence $\delta$, policy set $\Pi$
\State  $\ell_0 \leftarrow \lceil \log_2 \frac{d^{3/2}}{H} \rceil$, $\Pi_{\ell_0} \leftarrow \Pi$, $\bphihat^{1}_{\pi,1} \leftarrow \Exp_{a \sim \pi_1(\cdot | s_1)} [ \bphi(s_1,a)], \forall \pi \in \Pi$
\For{$\ell = \ell_0, \ell_0 + 1,\ldots, \lceil \log \frac{4}{\epsilon} \rceil$}
	\State $\epsilon_\ell \leftarrow 2^{-\ell}$, $\beta_\ell \leftarrow 64 H^4 \log \frac{4 H^2 | \Pi_\ell| \ell^2}{\delta}$
	\For{$h = 1,2,\ldots,H$}
		\State Solve \eqref{eq:exp_design} by running \Cref{alg:online_fw_body}, collect data\footnotemark $ \{ (\bphi_{h,\tau}, r_{h,\tau}, s_{h+1,\tau}) \}_{\tau=1}^{K_{h,\ell}}$ such that: 
		\begin{align*}
			\max_{\pi \in \Pi_{\ell}} \| \bphihat_{\pi,h}^\ell \|_{\bLambda_{h,\ell}^{-1}}^2 \le \epsilon_\ell^2/\beta_\ell \quad \text{for} \quad \bLambda_{h,\ell} \leftarrow \tsum_{\tau = 1}^{K_{h,\ell}} \bphi_{h,\tau} \bphi_{h,\tau}^\top + 1/d \cdot I
		\end{align*}
		\For{$\pi \in \Pi_\ell$} \hfill {\color{blue} \texttt{// Estimate feature-visitations for active policies}}
			\State $\bphihat_{\pi,h+1}^\ell \leftarrow \Big ( \sum_{\tau = 1}^{K_{h,\ell}} \bphi_{\pi,h+1}(s_{h+1,\tau}) \bphi_{h,\tau}^\top \bLambda_{h,\ell}^{-1} \Big ) \bphihat_{\pi,h}^\ell$ \label{line:est_feature_visit}
		\EndFor
		\State $\bthetahat_h^\ell \leftarrow \bLambda_{h,\ell}^{-1} \sum_{\tau = 1}^{K_{h,\ell}} \bphi_{h,\tau} r_{h,\tau}$ \hfill {\color{blue} \texttt{// Estimate reward vectors}}
	\EndFor
	\State {\color{blue} \texttt{// Remove provably suboptimal policies from active policy set}}
	\begin{align*}
	\Pi_{\ell+1} \leftarrow  \Pi_\ell \backslash \Big \{ \pi \in \Pi_\ell \ : \ \Vhat_0^\pi < \sup_{\pi' \in \Pi_\ell} \Vhat_0^{\pi'} - 2\epsilon_\ell \Big \} \quad \text{for} \quad \Vhat_0^\pi := \tsum_{h=1}^H \inner{\bphihat_{\pi,h}^\ell}{\bthetahat_h^\ell}
	\end{align*}
	\If{$|\Pi_{\ell+1}| = 1$}\label{line:early_term} \textbf{return} $\pi \in \Pi_{\ell+1}$\EndIf
\EndFor
\State \textbf{return} any $\pi \in \Pi_{\ell+1}$
\end{algorithmic}
\caption{ \textbf{P}olicy Learning via \textbf{E}xperiment \textbf{De}sign in \textbf{L}inear MDPs (\algname)}
\label{alg:linear_exp_design}
\end{algorithm}

\paragraph{\algname Description.}
\algname is a \emph{policy-elimination}-style algorithm. It takes as input some set of policies, $\Pi$, and proceeds in epochs, maintaining a set of \emph{active} policies, $\Pi_\ell$, such that all $\pi \in \Pi_\ell$ are guaranteed to satisfy $\Vpi_0 \ge \Vst_0(\Pi) - 4\epsilon_\ell$, for $\epsilon_\ell = 2^{-\ell}$. After running for $\lceil \log \frac{4}{\epsilon} \rceil$ epochs, it returns any of the remaining active policies, which will be guaranteed to have value at least $\Vst_0(\Pi) - \epsilon$. 

In order to ensure $\Pi_\ell$ only contains $4\epsilon_\ell$-optimal policies, sufficient exploration must be performed at every epoch to refine the estimate of each policy's value. While works such as \cite{wagenmaker2022reward} have demonstrated how to efficiently traverse a linear MDP and collect the necessary observations, existing exploration procedures are unable to obtain the instance-dependent complexity we desire. To overcome this, \algname relies on a novel online experiment design procedure to ensure exploration is focused only on the directions necessary to evaluate the current set of active policies.

In particular, one can show that, if we have collected some covariates $\bLambda_{h,\ell}$, the uncertainty in our estimate of the value of policy $\pi$ at step $h$ scales as $\| \bphihat_{\pi,h}^\ell \|_{\bLambda_{h,\ell}^{-1}}$, for $\bphihat_{\pi,h}^\ell$ the estimated feature-visitation for policy $\pi$ at epoch $\ell$. To reduce our uncertainty at each round, we would therefore like to collect covariates such that $\| \bphihat_{\pi,h}^\ell \|_{\bLambda_{h,\ell}^{-1}} \lesssim \epsilon_\ell$. Collecting covariates which satisfy this using the minimum number of episodes of exploration possible involves solving the experiment design:
\begin{align}\label{eq:exp_design}
\inf_{\bLamexp \in \bOmega_h} \max_{\pi \in \Pi_\ell} \| \bphihat_{\pi,h}^\ell \|_{\bLamexp^{-1}}^2.
\end{align}
Note that this design has the form of an $\mathsf{XY}$-experiment design \citep{soare2014best}. Solving \eqref{eq:exp_design} will produce covariance $\bLamexp$ which reduces uncertainty in relevant feature directions. However, to solve this design we require knowledge of which covariance matrices are realizable on our particular MDP.
In general we do not know the MDP's dynamics, and therefore do not have access to this knowledge. To overcome this and solve \eqref{eq:exp_design}, in \Cref{sec:online_exp_design} we provide an algorithm, \Cref{alg:online_fw_body}, that is able to solve \eqref{eq:exp_design} in an online manner without knowledge of the MDP dynamics by running a low-regret algorithm on a carefully chosen reward function.
\footnotetext{Note that $K_{h,\ell}$ is random and is the number of episodes \Cref{alg:online_fw_body} runs for.}

\paragraph{Estimating Feature-Visitations.} 
We remark briefly on the estimation of the feature-visitations on \Cref{line:est_feature_visit}. If we assume that $ \{ \bphi_{h,\tau} \}_{\tau=1}^{K_{h,\ell}}$ is fixed and that all randomness is due to $s_{h+1,\tau}$, then it is easy to see that, using the structure present in linear MDPs as given in \Cref{defn:linear_mdp},
\begin{align*}
\Exp[ \tsum_{\tau = 1}^{K_{h,\ell}} \bphi_{\pi,h+1}(s_{h+1,\tau}) \bphi_{h,\tau}^\top \bLambda_{h,\ell}^{-1} ] & = \tsum_{\tau = 1}^{K_{h,\ell}}  \big ( {\textstyle \int} \bphi_{\pi,h+1}(s) \rmd \bmu_h(s)^\top \bphi_{h,\tau} \big ) \bphi_{h,\tau}^\top \bLambda_{h,\ell}^{-1} \\
& = {\textstyle \int} \bphi_{\pi,h+1}(s) \rmd \bmu_h(s)^\top.
\end{align*}
By \Cref{defn:linear_mdp}, we have $\bphi_{\pi,h+1} = \big ({\textstyle \int} \bphi_{\pi,h+1}(s) \rmd \bmu_h(s)^\top \big ) \bphi_{\pi,h}$.
Comparing these, we see that our estimator of $\bphi_{\pi,h+1}$ on \Cref{line:est_feature_visit} is (conditioned on  $ \{ \bphi_{h,\tau} \}_{\tau=1}^{K_{h,\ell}}$) unbiased, assuming $\bphihat_{\pi,h}^\ell \approx \bphi_{\pi,h}$.

\subsection{Main Results}

We have the following result on the performance of \algname.

\begin{thm}\label{thm:complexity_linear}
Consider running \algname with some set of Markovian policies $\Pi$ on any linear MDP satisfying \Cref{defn:linear_mdp} and \Cref{asm:full_rank_cov}. Then with probability at least $1-\delta$, \algname outputs a policy $\pihat \in \Pi$ such that $V_0^{\pihat} \ge \Vst_0(\Pi) - \epsilon$, and runs for at most 
\begin{align*}
C_0 H^4 &  \cdot \sum_{h=1}^H  \inf_{\bLamexp \in \bOmega_h} \max_{\pi \in \Pi} \frac{ \| \bphi_{\pi,h} \|_{\bLamexp^{-1}}^2}{\max \{ \Vst_0(\Pi) - \Vpi_0, \Delmin^\Pi, \epsilon \}^2} \cdot \Big ( \log | \Pi | + \log \frac{1}{\delta} \Big )  +  C_1
\end{align*}
episodes, with $C_0 = \log \frac{1}{\epsilon} \cdot \poly\log(H,\log \frac{1}{\epsilon})$, $C_1 = \poly \big ( d, H, \frac{1}{\lamminst}, \log \frac{1}{\delta}, \log \tfrac{1}{\epsilon}, \log | \Pi | \big )$, 
and $\bOmega_h$ the set of covariance matrices realizable on our MDP\iftoggle{arxiv}{, as defined in \eqref{eq:cov_set_defn}.}{.}
\end{thm}

The proof of \Cref{thm:complexity_linear} is given in \Cref{sec:policy_elim}. \Cref{thm:complexity_linear} quantifies, in a precise instance-dependent way, the complexity of identifying a policy $\pihat$ with value at most a factor of $\epsilon$ from the value of the optimal policy in $\Pi$. 
\nipscrnew{At a high level, the complexity measure can be thought of as quantifying the difficulty of exploring the MDP so as to eliminate suboptimal policies. $\| \bphi_{\pi,h} \|_{\bLamexp^{-1}}$ quantifies the difficulty of collecting data necessary to eliminate policy $\pi$ (in particular, $\bphi_{\pi,h}$ is the direction we need to explore to learn about the performance of policy $\pi$, and $\bLamexp^{-1}$ quantifies how easily we can reach directions in the MDP to do this), while $\Vst_0(\Pi) - \Vpi_0$ quantifies how suboptimal policy $\pi$ is, and therefore how many samples are needed to distinguish it from the optimal policy in the class. }
Thus, rather than scaling only with factors such as $d$ and $\epsilon$, our complexity scales with instance-dependent quantities---the covariance matrices we can obtain and the feature vectors we expect to observe on our particular MDP, and the policy gaps on our MDP. \nipscrnew{Our complexity measure has a close resemblance to the complexity measure for best-arm identification in linear bandits found in \cite{fiez2019sequential}, but generalizes it to problems with long horizon where navigation is required.}

\Cref{thm:complexity_linear} holds for an arbitrary set of policies, yet, in general, we are interested in learning a policy which has value within a factor of $\epsilon$ of the value of the \emph{optimal} policy on the MDP, $\Vst_0$. Such a guarantee is immediately attainable by applying \Cref{thm:complexity_linear} with a policy set $\Pi$ such that $\sup_{\pi \in \Pi} \Vpi_0 \ge \Vst_0 - \epsilon$. The following result shows that it is possible to construct such a set of policies, and therefore learn a \emph{globally} near-optimal policy.

\begin{cor}\label{thm:complexity_linear2}
There exists a set of policies $\Pieps$ such that $\log |\Pieps| \le \cOtil(dH \cdot \log 1/\epsilon)$ and, for any linear MDP satisfying \Cref{defn:linear_mdp}, $\sup_{\pi \in \Pieps} \Vpi_0 \ge \Vst_0 - \epsilon$. If we run \algname with $\Pi \leftarrow \Pieps$, then with probability at least $1-\delta$, it returns a policy $\pihat$ such that $V_0^{\pihat} \ge \Vst_0 - 2\epsilon$, and runs for at most 
\begin{align*}
C_0 H^4 &  \cdot \sum_{h=1}^H  \inf_{\bLamexp \in \bOmega_h} \max_{\pi \in \Pieps} \frac{ \| \bphi_{\pi,h} \|_{\bLamexp^{-1}}^2}{\max \{ \Vst_0 - \Vpi_0, \epsilon \}^2} \cdot  \Big ( d H + \log \frac{1}{\delta} \Big )  +  C_1
\end{align*}
episodes, for $C_0 = \poly \log (d, H, \frac{1}{\epsilon})$. 
\end{cor}

\noindent \nipscrnew{Note that the policy set constructed in \Cref{thm:complexity_linear2}, $\Pieps$, is guaranteed to contain an $\epsilon$-optimal policy for \emph{any} linear MDP. Thus, this result states that without any prior knowledge of our MDP, \algname can be applied to find an $\epsilon$-optimal policy.}
While \Cref{thm:complexity_linear} and \Cref{thm:complexity_linear2} quantify the instance-dependent complexity of learning, it is natural to ask what the \emph{worst-case} complexity of \algname is. The following result provides such a bound. 

\begin{cor}\label{prop:minimax}
For any linear MDP satisfying \Cref{defn:linear_mdp}, $\inf_{\bLamexp \in \bOmega_h} \max_{\pi \in \Pieps}  \| \bphi_{\pi,h} \|_{\bLamexp^{-1}}^2 \le d$, so
the sample complexity of \Cref{alg:linear_exp_design} when run with $\Pi \leftarrow \Pieps$ is no larger than
\begin{align*}
\cOtil \left ( \frac{d H^5 (d H + \log 1/\delta)}{\epsilon^2} + C_1 \right ).
\end{align*}
\end{cor}

\Cref{prop:minimax} shows that \algname has worst-case optimal dimension dependence, matching the lower bound of $\Omega(d^2 H^2/\epsilon^2)$ given in \cite{wagenmaker2022reward}, up to $H$ and $\log$ factors\footnote{We remark that the focus of this work is on instance-dependence and dimension-dependence, not in optimizing $H$ factors, and we leave improving our $H$ dependence for future work.}.

\begin{rem}[Performance on Linear Contextual Bandits]
\Cref{thm:complexity_linear2} applies directly to linear contextual bandits by setting $H = 1$\footnote{We describe the exact mapping to linear contextual bandits in \Cref{sec:interpret_complexity}.}. To our knowledge, this is the first instance-dependent result on PAC policy identification in linear contextual bandits, though we remark that the concurrent work of \cite{li2022instance} also provides instance-dependent guarantees for contextual bandits. Furthermore, \Cref{prop:minimax} shows that we also obtain a worst-case complexity of $\cOtil(d^2/\epsilon^2)$ on linear contextual bandits, which is the optimal rate \citep{wagenmaker2022reward}.
\end{rem}

\subsection{Low-Regret Algorithms are Suboptimal for PAC RL in Large State-Spaces}
We next show that there are problems on which the instance-dependent complexity of \algname improves on the worst-case lower bound shown in \cite{wagenmaker2022reward}, thereby demonstrating that we do indeed obtain favorable complexities on ``easy'' instances. 

\begin{prop}\label{prop:easy_instance_upper}
For any $d > 2$, there exists a $d$-dimensional linear MDP with $H = 2$ such that with probability $1-\delta$, \algname identifies an $\epsilon$-optimal policy on this MDP after running for only $\cOtil \big ( \frac{\log d/\delta}{\epsilon^2} + \poly(d, \log \frac{1}{\delta}, \log \frac{1}{\epsilon} ) \big )$
episodes.
\end{prop}

The complexity given in \Cref{prop:easy_instance_upper} is a factor of $d^2$ better than the worst-case lower bound of $\Omega(d^2/\epsilon^2)$. While this shows that \algname yields a significant improvement over existing worst-case lower bounds on favorable instances, it is natural to ask whether the same complexity is attainable with existing algorithms, perhaps by applying a tighter analysis. Towards answering this, we will consider a class of low-regret algorithms and an online-to-batch learning protocol.

\begin{defn}[Low-Regret Algorithm]\label{def:low_regret}
We say that an algorithm is a \emph{low-regret algorithm} if its expected regret is bounded as, for all $K$:
\begin{align*}
\Exp[\cR_K] = \tsum_{k=1}^K \Exp[\Vst_0 - V_0^{\pi_k}] \le \cC_1 K^\alpha + \cC_2
\end{align*}
for some constants $\cC_1,\cC_2$, and $\alpha \in (0,1)$. 
\end{defn}

\begin{protocol}[Online-to-Batch Learning]\label{prot:online_to_batch}
The online-to-batch protocol proceeds as follows:
\begin{enumerate}[leftmargin=*]
\item The learner plays a low-regret algorithm satisfying \Cref{def:low_regret} for $K$ episodes.
\item The learner stops at a (possibly random) time $K$, and, using the observations it has collected in any way it wishes, outputs a policy $\pihat$ it believes is $\epsilon$-optimal. 
\end{enumerate}
\end{protocol}

In general, by applying online-to-batch learning, one can convert a regret guarantee of $ \cC_1 K^\alpha + \cC_2$ to a PAC complexity of $\cO( (\frac{\cC_1}{\epsilon})^{\frac{1}{1-\alpha}} + \frac{\cC_2}{\epsilon})$ \citep{jin2018q}, allowing low-regret algorithms such as that of \cite{zanette2020learning} to obtain the minimax-optimal PAC complexity of $\cO(d^2 H^4/\epsilon^2)$. The following result shows, however, that this protocol is unable to obtain the instance-optimal rate. 

\begin{prop}\label{prop:easy_instance_lower}
On the instance of \Cref{prop:easy_instance_upper}, for small enough $\epsilon$, any learner that is $(\epsilon,\delta)$-PAC on  the set of linear MDPs satisfying \Cref{defn:linear_mdp} and follows \Cref{prot:online_to_batch} with stopping time $K$ must have $\Exp[K] \ge \Omega \big ( \frac{d \cdot \log 1/\delta}{\epsilon^2} \big )$.
\end{prop}

Together, \Cref{prop:easy_instance_upper} and \Cref{prop:easy_instance_lower} show that running a low-regret algorithm to learn a near-optimal policy in a linear MDP is provably suboptimal---at least a factor of $d$ worse than the instance-dependent rate obtained by \algname. While a similar observation was recently made in the setting of tabular MDPs \citep{wagenmaker2021beyond}, to our knowledge, this is the first such result in the RL with function approximation setting, implying that, in this setting, low-regret algorithms are insufficient for obtaining optimal PAC sample complexity. 
As standard optimistic algorithms are also low-regret, this result implies that all such optimistic algorithms are also suboptimal.

\subsection{Tabular and Deterministic MDPs}
To relate our results to existing results on instance-dependent RL, we next turn to the setting of tabular MDPs, where it is assumed that $S := |\cS| < \infty, A := |\cA| < \infty$. Define:
\begin{align*}
\Delta_h(s,a) = \Vst_h(s) - \Qst_h(s,a), \quad w_h^{\pi}(s,a) = \Pr_{\pi}[s_h = s, a_h = a]. 
\end{align*}
$\Delta_h(s,a)$ denotes the \emph{value-function gap}, and quantifies the suboptimality of playing action $a$ in state $s$ at step $h$ and then playing the optimal policy, as compared to taking the optimal action in $(s,h)$. $w_h^\pi(s,a)$ denotes the state-action visitation distribution for policy $\pi$, and quantifies how likely policy $\pi$ is to reach $(s,a)$ at step $h$. Note that $[\bphi_{\pi,h}]_{(s,a)} = \wpi_h(s,a)$. 
We obtain the following corollary.

\begin{cor}\label{cor:tabular}
In the setting of tabular MDPs, \algname outputs an $\epsilon$-optimal policy with probability at least $1-\delta$, and has sample complexity bounded as
\iftoggle{arxiv}{
\begin{align*}
\cOtil \Bigg (  & \sum_{h=1}^H \inf_{\piexp} \max_{\pi \in \Pi}  \max_{s,a} \frac{H^4}{w_h^{\piexp}(s,a)} \min \left \{  \frac{1}{ w_h^\pi(s,a) \Delta_h(s,a)^2}, \frac{w_h^\pi(s,a)}{ \Delmin(\Pi)^2}, \frac{w_h^\pi(s,a)}{\epsilon^2 } \right \}  \cdot  (SH + \log \tfrac{1}{\delta}  )  +  C_1 \Bigg ),
\end{align*}}{
\begin{align*}
\cOtil \Big (  & \tsum_{h=1}^H \inf_{\piexp} \max_{\pi \in \Pi}  \max_{s,a} \tfrac{H^4}{w_h^{\piexp}(s,a)} \min \left \{  \tfrac{1}{ w_h^\pi(s,a) \Delta_h(s,a)^2}, \tfrac{w_h^\pi(s,a)}{ \Delmin(\Pi)^2}, \tfrac{w_h^\pi(s,a)}{\epsilon^2 } \right \}    (SH + \log \tfrac{1}{\delta}  )  +  C_1 \Big ),
\end{align*}
}
for $C_1 = \poly \big ( S, A, H, \frac{1}{\min_h \min_s \sup_\pi w_h^\pi(s)}, \log \frac{1}{\delta}, \log \frac{1}{\epsilon} \big )$ and $\Pi$ the set of all deterministic policies. 
\end{cor}

For tabular MDPs, the primary comparable result on instance-dependent policy identification is that obtained by \cite{wagenmaker2021beyond}, which introduces a different measure of complexity, the \emph{gap-visitation complexity}, and an algorithm, \moca, with sample complexity scaling as the gap-visitation complexity.
The following result shows that the complexity \algname obtains on tabular MDPs and the gap-visitation complexity do not have a clear ordering.

\begin{prop}\label{prop:gap_vis_vs_pedel}
Fix any $\epsilon \in (0,1/2)$ and $S \ge \log_2(1/\epsilon)$. Then there exist tabular MDPs $\cM_1$ and $\cM_2$, each with $H=2$, $S$ states, and $\cO(S)$ actions, such that:
\begin{itemize}[leftmargin=*]
\item On $\cM_1$, the complexity bound of \algname given in \Cref{thm:complexity_linear} scales as $\poly(S,\log 1/\delta)$, while the gap-visitation complexity scales as $\Omega(1/\epsilon^2)$.
\item On $\cM_2$, the complexity bound of \algname given in \Cref{thm:complexity_linear} scales as $\Omega(1/\epsilon^2)$, while the gap-visitation complexity scales as $\poly(S,\log 1/\delta)$. 
\end{itemize}
\end{prop}

The lack of ordering between the two complexity measures arises because, on some problem instances, it is easier to learn in policy-space (as \algname does), while on other instances, it is easier to learn near-optimal actions on individual states directly, and then synthesize these actions into a near-optimal policy (the approach \moca takes). \iftoggle{arxiv}{This difference arises because, in the former instance, the minimum \emph{policy gap} is large ($\Vst_0 - \Vpi_0 = \Omega(1)$ for every deterministic policy $\pi \neq \pist$), while in the latter instance, the minimum policy gap is small, but all value-function gaps are large, satisfying $\Delta_h(s,a) = \Omega(1)$ for all $a \neq \argmax_{a \in \cA} \Qst_h(s,a)$ and all $s$ and $h$. Thus, on the former instance, it is much easier to learn over the space of policies, while on the latter it is much easier to learn optimal actions in individual states\footnote{We remark that, in \Cref{prop:gap_vis_vs_pedel}, $\cM_1$ is a deterministic MDP, which \algname takes advantage of to obtain a tighter complexity. It is unclear if a similar result can be shown to hold for $\cM_1$ a stochastic MDP.}.}{
{\nipscrnew{This difference arises because, in the former instance, the minimum \emph{policy gap} is large ($\Vst_0 - \Vpi_0 = \Omega(1)$ for every policy $\pi \neq \pist$), while in the latter instance, the minimum policy gap is small, but all value-function gaps are large, satisfying $\Delta_h(s,a) = \Omega(1)$ for all $a \neq \argmax_{a \in \cA} \Qst_h(s,a)$ and all $s$ and $h$. Thus, on the former instance, it is much easier to learn over the space of policies, while on the latter it is much easier to learn optimal actions in individual states.}}
} Resolving this discrepancy with an algorithm able to achieve the ``best-of-both-worlds'' is an interesting direction for future work.

\paragraph{Deterministic MDPs.} Finally, we turn to the simplified setting of tabular, deterministic MDPs. Here, for each $(s,a,h)$, there exists some $s'$ such that $P_h(s'|s,a) = 1$. We still allow the rewards to be random, however, so the agent must still learn in order to find a near-optimal policy.
 Following the same notation as the recent work of \cite{tirinzoni2022near}, let $\Pi_{sah} = \{ \pi \text{ deterministic}  \ : \ s_h^\pi = s, a_h^\pi = a \}$, where $s_h^\pi$ and $a_h^\pi$ are the state and action policy $\pi$ will be in at step $h$ (note that these quantities are well-defined quantities for deterministic policies). Also define the \emph{deterministic return gap} as $\Delbar_h(s,a) := \Vst_0 - \max_{\pi \in \Pi_{sah}} V_0^\pi$, and let $\Delbarmin := \min_{s,a,h : \Delbar_h(s,a) > 0} \Delbar_h(s,a)$ in the case when there exists a unique optimal deterministic policy, and $\Delbarmin := 0$ otherwise. We obtain the following.

\begin{cor}\label{cor:deterministic_mdps}
In the setting of tabular, deterministic MDPs, \algname outputs an $\epsilon$-optimal policy with probability at least $1-\delta$, and has sample complexity bounded as
\begin{align*}
\cOtil \Bigg ( H^4 &  \cdot \sum_{h=1}^H \sum_{s,a} \frac{1}{\max \{ \Delbar_h(s,a), \Delbarmin, \epsilon \}^2}  \cdot  (H + \log \tfrac{1}{\delta} )  +  \poly \left ( S, A, H, \log \tfrac{1}{\delta}, \log \tfrac{1}{\epsilon} \right ) \Bigg ).
\end{align*}
\end{cor}

Up to $H$ and $\log$ factors and lower-order terms, the rate given in \Cref{cor:deterministic_mdps} matches the instance-dependent lower bound given in \cite{tirinzoni2022near}\footnote{The lower bound of \cite{tirinzoni2022near} depends on a slightly different (but nearly equivalent) minimum gap term, $\Delbarmin^h$. Similar to our upper bound, the upper bound of \cite{tirinzoni2022near} scales with $\Delbarmin$ instead of $\Delbarmin^h$. We offer a more in-depth discussion of this point in \Cref{sec:interpret_complexity}. }. Thus, we conclude that, in the setting of tabular, deterministic MDPs, \algname is (nearly) instance-optimal. While \cite{tirinzoni2022near} also obtain instance-optimality in this setting, their algorithm and analysis are specialized to tabular, deterministic MDPs---in contrast, \algname requires no modification from its standard operation.

\section{Online Experiment Design in Linear MDPs}\label{sec:online_exp_design}
As described in \Cref{sec:main_result}, to reduce our uncertainty and explore in a way that only targets the relevant feature directions, we must solve an $\mathsf{XY}$-experiment design problem of the form:
\iftoggle{arxiv}{\begin{align}\label{eq:exp_design2}
 \inf_{\bLamexp \in \bOmega_h} \max_{\bphi \in \Phi} \| \bphi \|_{\bLamexp^{-1}}^2 ,
\end{align}
where here $\Phi$ will be some set of estimated feature-visitations. Recall that $\bOmega_h$ denotes the set of covariance matrices realizable on our MDP, and therefore without knowledge of the MDP dynamics we cannot specify this set. If follows that it is not in general possible to solve \eqref{eq:exp_design2} without knowledge of the MDP dynamics. In this section we describe our approach to solving \eqref{eq:exp_design2} without this knowledge by relying on a low-regret algorithm as an optimization primitive.}{
\begin{align}\label{eq:exp_design2}
 \inf_{\bLamexp \in \bOmega_h} \max_{\bphi \in \Phi} \| \bphi \|_{\bLamexp^{-1}}^2 .
\end{align}
This is not, in general, possible to solve without knowledge of the MDP dynamics. In this section we describe our approach to solving \eqref{eq:exp_design2} without knowledge of the MDP dynamics by relying on a low-regret algorithm as an optimization primitive. }

\paragraph{Approximating Frank-Wolfe via Regret Minimization.}
Given knowledge of the MDP dynamics, we could compute $\bOmega_h$ directly, and apply the celebrated Frank-Wolfe coordinate-descent algorithm \citep{frank1956algorithm} to solve \eqref{eq:exp_design2}. In this setting the Frank-Wolfe update for \eqref{eq:exp_design2} is:
\begin{align}\label{eq:fw_update}
\textstyle \bGamma_t = \argmin_{\bGamma \in \bOmega_h} \  \inner{\nabla_{\bLambda} ( \max_{\bphi \in \Phi} \| \bphi \|_{\bLambda^{-1}}^2) |_{\bLambda = \bLambda_t} }{\bGamma}, \quad \bLambda_{t+1} = (1 - \gamma_t) \bLambda_t + \gamma_t \bGamma_t
\end{align}
for step size $\gamma_t$. Standard Frank-Wolfe analysis shows that this update converges to a near-optimal solution to \eqref{eq:exp_design2} at a polynomial rate. However, without knowledge of $\bOmega_h$, we are unable to solve for $\bGamma_t$ and run the Frank-Wolfe update.

Our critical observation is that the minimization over $\bOmega_h$ in \eqref{eq:fw_update} can be approximated without knowledge of $\bOmega_h$ by running a low-regret algorithm on a particular objective. Some calculation shows that (except on a measure-zero set, assuming $\Phi$ is finite) $\nabla_{\bLambda} ( \max_{\bphi \in \Phi} \| \bphi \|_{\bLambda^{-1}}^2) |_{\bLambda = \bLambda_t} = -\bLambda_t^{-1} \bphitil_t \bphitil_t^\top \bLambda_t^{-1}$ for $\bphitil_t = \argmax_{\bphi \in \Phi} \| \bphi \|_{\bLambda_t^{-1}}^2$. If $\bGamma = \bLambda_{\pi,h} = \Exp_\pi [ \bphi_h \bphi_h^\top]$ for some $\pi$, we have
\begin{align*}
\textstyle \inner{\nabla_{\bLambda} ( \max_{\bphi \in \Phi} \| \bphi \|_{\bLambda^{-1}}^2) |_{\bLambda = \bLambda_t} }{\bGamma} = -\tr( \bLambda_t^{-1} \bphitil_t \bphitil_t^\top \bLambda_t^{-1} \bLambda_{\pi,h}) = -\Exp_\pi[ (\bphi_h^\top \bLambda_t^{-1} \bphitil_t)^2].
\end{align*}
Now, if we run a low-regret algorithm on the (deterministic) reward $\nur_h^t(s,a) = (\bphi(s,a)^\top \bLambda_t^{-1} \bphitil_t)^2$ for a sufficiently large number of episodes $K$, we will be guaranteed to collect reward at a rate close to that of the optimal policy, which implies we will collect some data $\{ \bphi_{h,\tau} \}_{\tau = 1}^K$ such that
\iftoggle{arxiv}{
\begin{align}\label{eq:fw_approx_update}
K^{-1} \cdot \bphitil_t^\top \bGamhat_K \bphitil_t := K^{-1} \sum_{\tau=1}^K (\bphi_{h,\tau}^\top \bLambda_t^{-1} \bphitil_t)^2 \approx \sup_\pi \Exp_\pi[(\bphi_h^\top \bLambda_t^{-1} \bphitil_t)^2].
\end{align}
}{
\begin{align}\label{eq:fw_approx_update}
\textstyle K^{-1} \cdot \bphitil_t^\top \bGamhat_K \bphitil_t := K^{-1} \tsum_{\tau=1}^K (\bphi_{h,\tau}^\top \bLambda_t^{-1} \bphitil_t)^2 \approx \sup_\pi \Exp_\pi[(\bphi_h^\top \bLambda_t^{-1} \bphitil_t)^2].
\end{align}
}
However, this implies the covariates we have collected, $ \bGamhat_K$, approximately minimize \eqref{eq:fw_update}. 
In other words, running a low-regret algorithm on $\nur_h^t$ allows us to obtain covariates which approximate the Frank-Wolfe update---without knowledge of $\bOmega_h$, we can solve the Frank-Wolfe update by running a low-regret algorithm, and therefore solve \eqref{eq:exp_design2}. This motivates \Cref{alg:online_fw_body}.
\begin{algorithm}[h]
\begin{algorithmic}[1]
\State \textbf{input:} uncertain feature directions $\Phi$, step $h$, regularization $\bLambda_0 \succ 0$
\State $K_0 \leftarrow$ sufficiently large number of episodes to guarantee \eqref{eq:fw_approx_update} holds
\State Run any policy for $K_0$ episodes, collect data $\{ \bphi_{h,\tau} \}_{\tau=1}^{K_0}$, set $\bLambda_1 \leftarrow K_0^{-1} \sum_{\tau=1}^{K_0} \bphi_{h,\tau} \bphi_{h,\tau}^\top$
\For{$t = 1,\ldots,T-1$}
	\State $\bphitil_t \leftarrow \argmax_{\bphi \in \Phi} \| \bphi \|_{(\bLambda_t + \bLambda_0)^{-1}}^2$, $\nur_h^t(s,a) \leftarrow (\bphi(s,a)^\top (\bLambda_t + \bLambda_0)^{-1} \bphitil_t)^2$
	\State Run low-regret algorithm on $\nur_h^t$ for $K_0$ episodes, collect covariates $ \bGamhat_{K_0}^t$
	\State Set $\bLambda_{t+1} \leftarrow (1 - \gamma_t) \bLambda_t + \gamma_t K_0^{-1} \bGamhat_{K_0}^t$ for $\gamma_t = \frac{1}{t+1}$
\EndFor
\State \textbf{return:} covariates $T K_0 \bLambda_{T} = \sum_{t=1}^{T-1} \bGamhat_{K_0}^t + \bLambda_1$
\end{algorithmic}
\caption{Online Frank-Wolfe via Regret Minimization (informal)}
\label{alg:online_fw_body}
\end{algorithm}

\begin{thm}[informal]\label{cor:g_opt_exp_design}
Consider running \Cref{alg:online_fw_body} with some $\bLambda_0 \succ 0$. Then with properly chosen settings of $K_0$ and $T$, we can guarantee that, with probability at least $1-\delta$, we will run for at most
\begin{align*}
N \le 20  \cdot  \frac{  \inf_{\bLamexp \in \bOmega_h} \max_{\bphi \in \Phi} \| \bphi \|_{(\bLamexp + \bLambda_0)^{-1}}^2}{\epsexp} + \poly \left ( d, H, \| \bLambda_0^{-1} \|_\op, \log |\Phi|, \log 1/\delta \right )
\end{align*}
episodes, and return covariance $\bLamhat_N$ satisfying $\max_{\bphi \in \Phi} \| \bphi \|_{(\bLamhat_N + N \bLambda_0)^{-1}}^2 \le \epsexp$.
\end{thm}
Note that this rate is essentially optimal, up to constants and lower-order terms. If we let $\omega_{\mathrm{exp}}^\star$ denote the distribution over policies which minimize \eqref{eq:exp_design2}, then to collect covariance $\bLamhat_N$ such that $\max_{\bphi \in \Phi} \| \bphi \|_{(\bLamhat_N + N \bLambda_0)^{-1}}^2 \le \epsexp$,
in expectation, we would need to play $\pi \sim \omega_{\mathrm{exp}}^\star$ for at least
\begin{align*}
\inf_{\bLamexp \in \bOmega_h} \max_{\bphi \in \Phi} \| \bphi \|_{(\bLamexp + \bLambda_0)^{-1}}^2 \cdot \epsexp^{-1} 
\end{align*}
episodes, which is the same scaling as obtained in \Cref{cor:g_opt_exp_design}.

In practice, we instead run \Cref{alg:online_fw_body} on a smoothed version of the objective in \eqref{eq:exp_design2}. We provide a full definition of \Cref{alg:online_fw_body} with exact setting of $T$ and $K_0$ in \Cref{sec:fw}.
\iftoggle{arxiv}{}{\Cref{cor:g_opt_exp_design} is itself a corollary of a more general result, \Cref{thm:regret_data_fw}, given in \Cref{sec:fw}, which shows our Frank-Wolfe procedure can be applied to minimize any smooth
 experiment-design objective---for example, collecting covariates which optimally maximize the minimum eigenvalue, $\mathsf{E}$-optimal design. 
 }

\iftoggle{arxiv}{
\subsection{Experiment Design in MDPs with General Objective Functions}
While the experiment design in \eqref{eq:exp_design2} is the natural design if our goal is to identify a near-optimal policy, in general we may be interested in collecting data to minimize some other objective; that is, solving an experiment design of the form:
\begin{align*}
\inf_{\bLamexp \in \bOmega_h} f(\bLamexp)
\end{align*}
for some function $f$ defined over the space of PSD matrices. For example, we could take $f(\bLamexp) = \| \bLamexp^{-1} \|_\op = \frac{1}{\lammin(\bLamexp)}$, and the above experiment design would correspond to maximizing the minimum eigenvalue of the collected covariates, or $\mathsf{E}$-optimal design \citep{pukelsheim2006optimal}.

Motivated by this, in \Cref{sec:fw} we generalize \Cref{cor:g_opt_exp_design} and \Cref{alg:online_fw_body} to handle a much broader class of experiment design problems. In particular, we consider all \emph{smooth experiment design objectives}, which we define as follows.

\begin{defn}[Smooth Experiment Design Objectives]\label{def:smooth_exp_des_fun}
We say that $f(\bLambda) : \psdmat^d \rightarrow \R$ is a \emph{smooth experiment design objective} if it satisfies the following conditions:
\begin{itemize}
\item $f$ is convex, differentiable, and $\beta$ smooth in the norm $\| \cdot \|$: $\| \nabla f(\bLambda) - \nabla f(\bLambda') \|_* \le \beta \| \bLambda - \bLambda' \|$. 
\item $f$ is $L$-lipschitz in the operator norm: $| f(\bLambda) - f(\bLambda') | \le L \| \bLambda - \bLambda' \|_\op$.
\item Let $\Xi_{\bLambda_0} := -\nabla_{\bLambda} f(\bLambda) |_{\bLambda = \bLambda_0}$. Then $\Xi_{\bLambda_0} \succeq 0$ and $\tr( \Xi_{\bLambda_0}) \le M$ for all $\bLambda_0 \succeq 0$ satisfying $\| \bLambda_0 \|_\op \le 1$.
\end{itemize}
\end{defn}

Our generalization of \Cref{alg:online_fw_body} to handle all smooth experiment design objectives---\optcov, defined in \Cref{sec:optcov}---enjoys the following guarantee.

\begin{thm}\label{cor:Nst_opt_force_body}
Fix $h \in [H]$, consider some $f$ satisfying \Cref{def:smooth_exp_des_fun}, and let $\fmin$ be some value such that $\inf_{\bLamexp \in \bOmega_h} f(\bLamexp) \ge \fmin$. Then with probability at least $1-\delta$, given any $\epsilon > 0$, \optcov runs for at most
\begin{align*}
N \le 5 \cdot \frac{\inf_{\bLamexp \in \bOmega_h} f(\bLamexp)}{\epsilon} + \poly \left ( d,H, M, \beta, L, \fmin^{-1}, \log 1/\delta \right )
\end{align*}
episodes, and collects covariates $\bSighat_N = \sum_{\tau = 1}^N \bphi_{h,\tau} \bphi_{h,\tau}^\top$ such that
\begin{align*}
f(N^{-1} \bSighat_N) \le N \epsilon.
\end{align*}
\end{thm}
We will often be interested in objectives $f$ that satisfy $f(a \bLambda) = a^{-1} f(\bLambda)$ for a scalar $a$, in which case the guarantee $f(N^{-1} \bSighat_N) \le N \epsilon$ reduces to $f( \bSighat_N) \le \epsilon$. We note also that many typical experiment design objectives are non-smooth. As we show in \Cref{sec:xy_design}, however, it is often possible to derive smoothed versions of such objectives with negligible approximation error. 
}{}

\section{Conclusion}
In this work, we have shown that it is possible to obtain instance-dependent guarantees in RL with function approximation, and that our algorithm, \algname, yields provable gains over low-regret algorithms. As the first result of its kind in this setting, it opens several directions for future work.

The computational complexity of \algname scales as $\poly(d,H,\tfrac{1}{\epsilon}, |\Pi|, |\cA|, \log \tfrac{1}{\delta})$. In general, to ensure $\Pi$ contains an $\epsilon$-optimal policy, $|\Pi|$ must be exponential in problem parameters, rendering \algname computationally inefficient. Furthermore, the sample complexity of \algname scales with $\lamminst$, the ``hardest-to-reach'' direction. While this is not uncommon in the literature, we might hope that if a direction is very difficult to reach, learning in that direction should not be necessary, as we are unlikely to ever encounter it. Obtaining an algorithm with a similar instance-dependence but that is computationally efficient and does not depend on $\lamminst$ is an interesting direction for future work.

\iftoggle{arxiv}{
Extending our results to the setting of general function approximation is also an exciting direction. While our results do rely on the linear structure of the MDP, we believe the online experiment-design approach we propose could be generally applicable in more complex settings. As a first step, it could be interesting to extend our approach to the setting of Bilinear classes \citep{du2021bilinear}, which also exhibits a certain linear structure.}{
\nipscrnew{Extending our results to the setting of general function approximation is also an exciting direction. While our results do rely on the linear structure of the MDP, we believe the online experiment-design approach we propose could be generally applicable in more complex settings. As a first step, it could be interesting to extend our approach to the setting of Bilinear classes \citep{du2021bilinear}, which also exhibits a certain linear structure.}}

\subsection*{Acknowledgements}
The authors would like to thank Qiwen Cui and Tiancheng Jin for helpful comments on an earlier version of this paper. 
The work of AW was supported by an NSF GFRP Fellowship DGE-1762114. The work of KJ was funded in part by the AFRL and NSF TRIPODS 2023166.

\newpage
\bibliographystyle{icml2021}
\bibliography{bib_generals.bib}


\newpage
\appendix
\tableofcontents

\newpage

\section{Technical Results}
\begin{lem}[\cite{vershynin2010introduction}]\label{lem:euc_ball_cover}
For any $\epsilon > 0$, the $\epsilon$-covering number of the Euclidean ball $\mathcal{B}^d(R) := \{\bx \in \R^d: \|\bx\|_2 \le R\}$ with radius $R > 0$ in the Euclidean metric is upper bounded by $(1 + 2R/\epsilon)^d$. 
\end{lem}

\begin{lem}[Lemma A.4 of \cite{wagenmaker2022reward}]\label{claim:log_lin_burnin}
If $x \ge C (2n)^n \log^n(2n CB)$ for $n, C,B \ge 1$, then $x \ge C \log^n(B x)$.
\end{lem}

\begin{lem}[\cite{mcsherry2007mechanism,epasto2020optimal}]\label{lem:softmax_approx}
Consider some $(x_i)_{i=1}^n$. Then if $\eta \ge \log(n)/\delta$, we have
\begin{align*}
\frac{\sum_{i = 1}^n e^{\eta x_i} x_i}{\sum_{i=1}^n e^{\eta x_i}} \ge \max_{i \in [n]} x_i - \delta .
\end{align*}
\end{lem}

\begin{lem}[Azuma-Hoeffding]\label{lem:ah}
et $\cF_0 \subset \cF_1 \subset \ldots \subset \cF_T$ be a filtration and let $X_1,X_2,\ldots,X_T$ be real random variables such that $X_t$ is $\cF_t$-measurable, $\Exp[X_t | \cF_{t-1}] = 0$, and $|X_t| \le b$ almost surely. Then for any $\delta \in (0,1)$, we have with probability at least $1-\delta$, 
\begin{align*}
\left | \sum_{t=1}^T X_t \right | \le \sqrt{8 b^2 \log 2/\delta}.
\end{align*}
\end{lem}

\begin{lem}[Freedman's Inequality \citep{freedman1975tail}]\label{lem:freedmans}
Let $\cF_0 \subset \cF_1 \subset \ldots \subset \cF_T$ be a filtration and let $X_1,X_2,\ldots,X_T$ be real random variables such that $X_t$ is $\cF_t$-measurable, $\Exp[X_t | \cF_{t-1}] = 0$, $|X_t| \le b$ almost surely, and $\sum_{t=1}^T \Exp[X_t^2 | \cF_{t-1}] \le V$ for some fixed $V > 0$ and $b > 0$. Then for any $\delta \in (0,1)$, we have with probability at least $1-\delta$, 
\begin{align*}
\sum_{t=1}^T X_t \le 2 \sqrt{V \log 1/\delta} + b \log 1/\delta.
\end{align*}
\end{lem}

\subsection{Properties of Linear MDPs}

\begin{lem}\label{lem:feature_norm_lb}
For any linear MDP satisfying \Cref{defn:linear_mdp}, we must have that $\| \bphi(s,a) \|_2 \ge 1/\sqrt{d}$ for all $s$ and $a$, and $\| \bphi_{\pi,h} \|_2 \ge 1/\sqrt{d}$ for all $\pi$ and $h$.
\end{lem}
\begin{proof}
By \Cref{defn:linear_mdp}, we know that $P_h(\cdot | s,a) = \inner{\bphi(s,a)}{\bmu_h(\cdot)}$ forms a valid probability distribution, and that $\| \int_\cS |\rmd \bmu_h(s)| \|_2 \le \sqrt{d}$. It follows that
\begin{align*}
1 = \int_\cS \inner{\bphi(s,a)}{\rmd \bmu_h(s)} \le \| \bphi(s,a) \|_2 \| \int_\cS |\rmd \bmu_h(s)| \|_2 \le \sqrt{d}  \| \bphi(s,a) \|_2
\end{align*}
from which the first result follows.

For the second result, using that $1 = \int_\cS \inner{\bphi(s,a)}{\rmd \bmu_h(s)} $, we get
\begin{align*}
\int_\cS \inner{\bphi_{\pi,h}}{\rmd \bmu_h(s)} & = \int_\cS \inner{\Exp_\pi[\bphi_h]}{\rmd \bmu_h(s)} \\
& = \Exp_\pi \left [ \int_{\cS} \inner{\bphi_h}{\rmd \bmu_h(s)} \right ] \\
& = \Exp_\pi[1] \\
& = 1
\end{align*}
where we can exchange the order of integration by Fubini's Theorem since the integrand is absolutely integrable, by \Cref{defn:linear_mdp}. As above, we then have
\begin{align*}
1 = \int_\cS \inner{\bphi_{\pi,h}}{\rmd \bmu_h(s)}  \le \sqrt{d} \| \bphi_{\pi,h} \|_2
\end{align*}
so the second result follows. 
\end{proof}

\subsection{Feature-Visitations in Linear MDPs}
Define
\begin{align*}
\bphi_{\pi,h} = \Exp_\pi[\bphi(s_h,a_h)], \quad \bphi_{\pi,h}(s) = \sum_{a \in \cA} \bphi(s,a) \pi_h(a|s)
\end{align*}
and
\begin{align*}
\cT_{\pi,h} := \int \bphi_{\pi,h}(s) \rmd \bmu_{h-1}(s)^\top .
\end{align*}

\begin{lem}
$\bphi_{\pi,h} = \cT_{\pi,h} \bphi_{\pi,h-1} = \ldots = \cT_{\pi,h} \ldots \cT_{\pi,1} \bphi_{\pi,0}$.
\end{lem}
\begin{proof}
By the linear MDP assumption, we have:
\begin{align*}
\bphi_{\pi,h} & = \Exp_\pi[\bphi(s_h,a_h)] \\
& = \Exp_\pi[ \Exp[\bphi(s_h,a_h) | \cF_{h-1}]] \\
& = \Exp_\pi[ \int \int \bphi(s,a) \rmd \pi_h(a|s) \rmd \bmu_{h-1}(s)^\top \bphi(s_{h-1},a_{h-1})] \\
& = \Exp_\pi[ \int \bphi_{\pi,h}(s) \rmd \bmu_{h-1}(s)^\top \bphi(s_{h-1},a_{h-1})] \\
& =  \int \bphi_{\pi,h}(s) \rmd \bmu_{h-1}(s)^\top \Exp_\pi[\bphi(s_{h-1},a_{h-1})] \\
& = \cT_{\pi,h} \bphi_{\pi,h-1}.
\end{align*}
This yields the first equality. Repeating this calculation $h-1$ more times yields the final equality. 
\end{proof}

\begin{lem}\label{lem:state_vis_norm_bound}
Fix some $h$ and $i < h$, and consider the vector
\begin{align*}
\bv := \cT_{\pi,i+1}^\top \cT_{\pi,i+2}^\top \ldots \cT_{\pi,h-1}^\top \cT_{\pi,h}^\top \bu . 
\end{align*}
Assume that either $\bu = \btheta_h$ for some $\btheta_h$ which is a valid reward vector as defined in \Cref{defn:linear_mdp}, or $\bu \in \cS^{d-1}$. In either case, we have that, for any $s,a$, $|\bv^\top \bphi(s,a)| \le 1$, and $\| \bv \|_2 \le \sqrt{d}$.
\end{lem}
\begin{proof}
By the linear MDP structure (see Proposition 2.3 of \cite{jin2020provably}), for any $j$, 
\begin{align*}
\Qpi_j(s,a) & = \inner{\bphi(s,a)}{\bw_j^\pi} \\
& = \inner{\bphi(s,a)}{\btheta_j} + \int \Vpi_{j+1}(s') \rmd \bmu_j(s')^\top \bphi(s,a) \\
& = \inner{\bphi(s,a)}{\btheta_j} + \int \inner{\bw_{j+1}^\pi}{\bphi_{j+1,\pi}(s') } \rmd \bmu_j(s')^\top \bphi(s,a) \\
& = \inner{\bphi(s,a)}{\btheta_j + \cT_{\pi,j+1}^\top \bw_{j+1}^\pi} 
\end{align*}
so in general,
\begin{align*}
\bw_i^\pi = \sum_{h' = i}^H ( \prod_{j = i+1}^{h'} \cT_{\pi,j}^\top ) \btheta_{h'}
\end{align*}
where we order the product $\prod_{j = i+1}^{h'} \cT_{\pi,j}^\top = \cT_{\pi,i+1}^\top \cT_{\pi,i+1}^\top \ldots \cT_{\pi,h'}^\top$. 

\paragraph{Case 1: $\bu = \btheta_h$.} 
We first consider the case where $\bu = \btheta_h$ for some $\btheta_h$ which is a valid reward satisfying \Cref{defn:linear_mdp}. Assume that the reward in our MDP is set such that for $h' \neq h$, $\btheta_{h'} = 0$. In this case, we then have that
\begin{align*}
\bw_i^\pi = \cT_{\pi,i+1}^\top \cT_{\pi,i+2}^\top \ldots \cT_{\pi,h}^\top \btheta_h = \bv.
\end{align*}
In this case, we know that the trajectory rewards are always bounded by 1, so it follows that $\Qpi_i(s,a) \le 1$. Thus,
\begin{align*}
1 \ge \Qpi_i(s,a) = \inner{\bphi(s,a)}{\bw_i^\pi} = \inner{\bphi(s,a)}{\bv}
\end{align*}
and this holds for any $s,a$. Since $Q$-values are always positive, it also holds that $\inner{\bphi(s,a)}{\bv} \ge 0$. 

To bound the norm of $\bv$, we note that by the Bellman equation and the calculation above,
\begin{align*}
\| \bv \|_2 =  \| \bw^\pi_i \|_2 &  = \| \btheta_i + \int V_{i+1}^\pi(s') \rmd \bmu_i(s') \|_2 \\
& \le \| \btheta_i \|_2 + \| \int | V_{i+1}^\pi(s') | \rmd \bmu_i(s') \|_2 \\
& \le \| \int | \rmd \bmu_i(s') | \|_2 \\
& \le \sqrt{d}
\end{align*}
where we have used that $| V_{i+1}^\pi(s') | \le 1$ since the total episode return is at most 1 on our augmented reward function, and the linear MDP assumption.

\paragraph{Case 2: $\bu \in \cS^{d-1}$.}
We can repeat the argument above in the case where we only assume $\bu \in \cS^{d-1}$. Since $\| \bphi(s,a) \|_2 \le 1$, it follows that with the reward vector at level $h$ set to $\bu$, the reward will still be bounded in $[-1,1]$. Thus, essentially the same argument can be used, with the slight modification to handle $Q$-values that are negative.
\end{proof}

\begin{lem}\label{lem:policy_cov_convex}
The set $\bOmega_h$ is convex and compact.
\end{lem}
\begin{proof}
Take $\bLambda_1, \bLambda_2 \in \bOmega_h$. By definition, $\bLambda_1 = \Exp_{\pi \sim \omega_1}[\bLambda_{\pi,h}], \bLambda_2 = \Exp_{\pi \sim \omega_2}[\bLambda_{\pi,h}]$. It follows that, for any $t \in [0,1]$, $t \bLambda_1 + (1-t) \bLambda_2 = \Exp_{\pi \sim t \omega_1 + (1-t) \omega_2}[\bLambda_{\pi,h}]$. For $t \omega_1 + (1-t) \omega_2$ the mixture of $\omega_1$ and $\omega_2$. As $t \omega_1 + (1-t) \omega_2$ is a valid mixture over policies, it follows that $t \bLambda_1 + (1-t) \bLambda_2 \in \bOmega_h$, which proves convexity. 

Compactness follows since $\| \bphi(s,a) \|_2 \le 1$ for all $s,a$, so $\| \bLambda_{\pi,h} \|_\op \le 1$, which implies $\| \bLambda \|_\op \le 1$ for any $\bLambda \in \bOmega_h$. Furthermore, the set $\bOmega_h$ is clearly closed, which proves compactness.

\end{proof}

\subsection{Constructing the Policy Class}

\begin{lem}[Lemma B.1 of \cite{jin2020provably}]\label{lem:wpi_norm_bound}
Let $\bw_h^\pi$ denote the set of weights such that $\Qpi_h(s,a) = \inner{\bphi(s,a)}{\bw_h^\pi}$. Then $\| \bw_h^\pi \|_2 \le 2 H \sqrt{d}$.
\end{lem}

\begin{lem}\label{lem:action_covering}
For any $\delta > 0$ there exists sets of actions $(\cAtil_s)_{s \in \cS}$, $\cAtil_s \subseteq \cA$, such that $|\cAtil_s| \le (1 + 8H\sqrt{d}/\delta)^d$ for all $s$ and, for all $a \in \cA$, $s$, $h$, and any $\pi$, there exists some $\atil \in \cAtil_s$ such that
\begin{align*}
| \Qpi_h(s,a) - \Qpi_h(s,\atil) | \le \delta, \quad |r_h(s,a) - r_h(s,\atil)| \le \delta.
\end{align*}
\end{lem}
\begin{proof}
Let $\cN$ be a $\delta/(4H\sqrt{d})$ cover of the unit ball. By \Cref{lem:euc_ball_cover} we can bound $|\cN| \le (1 + 8H\sqrt{d}/\delta)^d$. Take any $s$ and let $\cAtil_s = \emptyset$. Then for each $\bphi \in \cN$, choose any $a$ at random from the set $\{ a \in \cA \ : \ \| \bphi(s,a) - \bphi \|_2 \le \delta/2 \}$ and set $\cAtil_s \leftarrow \cAtil_s \cup \{ a \}$. With this construction, we claim that for all $a \in \cA$, there exists some $\atil \in \cAtil_s$ such that $\| \bphi(s,a) - \bphi(s,\atil) \|_2 \le \delta/(2H\sqrt{d})$. To see why this is, note that by construction of $\cN$, there always exists some $\bphi \in \cN$ such that $\| \bphi(s,a) - \bphi \|_2 \le \delta/(4H\sqrt{d})$. Since $\cAtil_s$ will contain some $\atil$ such that  $\| \bphi(s,\atil) - \bphi \|_2 \le \delta/(4H\sqrt{d})$, the claim follows by the triangle inequality.

By \Cref{lem:wpi_norm_bound}, we have that for any $\pi$, $\| \bw_h^\pi \|_2 \le 2 H \sqrt{d}$. Take $a \in \cA$ and let $\atil \in \cAtil_s$ be the action such that $\| \bphi(s,a) - \bphi(s,\atil) \|_2 \le \delta/(2H\sqrt{d})$. Then
\begin{align*}
| \Qpi_h(s,a) - \Qpi_h(s,\atil) | & = | \inner{\bphi(s,a) - \bphi(s,\atil)}{\wpi_h} |  \le 2 H \sqrt{d} \| \bphi(s,a) - \bphi(s,\atil) \|_2 \le \delta .
\end{align*}
The bound on $|r_h(s,a) - r_h(s,\atil)|$ follows analogously, since we assume our rewards are linear, and that $\| \btheta_h \|_2 \le \sqrt{d}$. 
\end{proof}

\begin{defn}[Linear Softmax Policy]
We say a policy is a \emph{linear softmax policy} with parameters $\eta$ and $\{ \bw_h \}_{h=1}^H$ if it can be written as
\begin{align*}
\pi_h(a|s) = \frac{e^{\eta \inner{\bphi(s,a)}{\bw_h}}}{\sum_{a' \in \cA} e^{\eta \inner{\bphi(s,a')}{\bw_h}}}
\end{align*}
for some $\bw = \{ \bw_h \}_{h=1}^H$. We will denote such a policy as $\piw$. 
\end{defn}

\begin{defn}[Restricted-Action Linear Softmax Policy]
We say a policy is a \emph{restricted-action linear softmax policy} with parameters $\eta$, $\{ \bw_h \}_{h=1}^H$, and $(\cAtil_s)_{s \in \cS}$ if it can be written as
\begin{align*}
\pitil_h(a|s) = \frac{e^{\eta \inner{\bphi(s,a)}{\bw_h}} \cdot \I \{ a \in \cAtil_s \}}{\sum_{a' \in \cAtil_s} e^{\eta \inner{\bphi(s,a')}{\bw_h}}}
\end{align*}
for some $\bw = \{ \bw_h \}_{h=1}^H$. We will denote such a policy as $\pitilw$. 
\end{defn}

\begin{lem}\label{lem:softmax_val_approx}
For any restricted-action linear softmax policies $\piw$ and $\piu$ with identical restricted sets $(\cAtil_s)_{s \in \cS}$, we can bound
\begin{align*}
| V^{\piw}_0(s_1) - V^{\piu}_0(s_1) | \le 2  d H \eta \sum_{h=1}^H \| \bw_{h} - \bu_{h} \|_2.
\end{align*}
\end{lem}
\begin{proof}
Note that for any policy $\pi$, the value of the policy can be expressed as
\begin{align*}
\Vpi_0(s_1) = \sum_{h=1}^H \inner{\btheta_h}{\bphi_{\pi,h}}.
\end{align*}
Thus,
\begin{align*}
| V^{\piw}_0(s_1) - V^{\piu}_0(s_1) |  \le \sum_{h=1}^H | \inner{\btheta_h}{\bphi_{\piw,h} - \bphi_{\piu,h}}|.
\end{align*}
So it suffices to bound $| \inner{\btheta_h}{\bphi_{\piw,h} - \bphi_{\piu,h}}|$. Using the same decomposition as in the proof of \Cref{lem:phihat_est_fixed}, we have
\begin{align*}
\bphi_{\piw,h} - \bphi_{\piu,h} & = \sum_{i=0}^{h-1} \left ( \prod_{j = h - i + 1}^{h} \cT_{\piw,j} \right ) ( \cT_{\piw,h-i} - \cT_{\piu,h-i}) \bphi_{\piu,h-i-1}.
\end{align*}
By definition,
\begin{align*}
 \cT_{\piw,h-i} - \cT_{\piu,h-i} = \int ( \bphi_{\piw,h-i}(s) - \bphi_{\piu,h-i}(s) ) \rmd \bmu_{h-i-1}(s)^\top
\end{align*}
where
\begin{align*}
\bphi_{\piw,h-i}(s) = \sum_{a \in \cAtil_s} \bphi(s,a) \piw_{h-i}(a|s).
\end{align*}
Our goal is to bound $\| \bphi_{\piw,h-i}(s) - \bphi_{\piu,h-i}(s) \|_2 $. Let $\bw^t = t \bw + (1-t) \bu$, and $\bphi_{t,h-i}(s) := \bphi_{\pi^{\bw^t},h-i}(s)$. Then $\| \bphi_{\piw,h-i}(s) - \bphi_{\piu,h-i}(s) \|_2 = \| \bphi_{1,h-i}(s) - \bphi_{0,h-i}(s) \|_2$. By the Mean Value Theorem, we can bound
\begin{align*}
\| \bphi_{1,h-i}(s) - \bphi_{0,h-i}(s) \|_2 \le \max_{t' \in [0,1]} \| \frac{\rmd}{\rmd t} \bphi_{t,h-i}(s) |_{t = t'} \|_2 \cdot | 1 - 0 | .
\end{align*}
Some calculation shows that
\begin{align*}
\frac{\rmd}{\rmd t} \bphi_{t,h-i}(s) & = \frac{1}{(\sum_{a' \in \cAtil_s} e^{\eta \inner{\bphi(s,a')}{\bw_h^t}})^2} \bigg [ \sum_{a \in \cAtil_s} \eta \bphi(s,a) \Big [ e^{\eta \inner{\bphi(s,a)}{\bw_h^t}} \inner{\bphi(s,a)}{\bw_{h-i} - \bu_{h-i}} \cdot \sum_{a' \in \cAtil_s} e^{\eta \inner{\bphi(s,a')}{\bw_h^t}} \\
& \qquad - e^{\eta \inner{\bphi(s,a)}{\bw_h^t}} \cdot \sum_{a' \in \cAtil_s} e^{\eta \inner{\bphi(s,a')}{\bw_h^t}}\inner{\bphi(s,a')}{\bw_{h-i} - \bu_{h-i}} \Big ] \bigg ] .
\end{align*}
Since $\| \bphi(s,a) \|_2 \le 1$, we can then bound
\begin{align*}
\| \frac{\rmd}{\rmd t} \bphi_{t,h-i}(s)  \|_2 & \le \frac{2\eta}{(\sum_{a' \in \cAtil_s} e^{\eta \inner{\bphi(s,a')}{\bw_h^t}})^2} \bigg [ \sum_{a \in \cAtil_s}  e^{\eta \inner{\bphi(s,a)}{\bw_h^t}} \cdot \sum_{a' \in \cAtil_s} e^{\eta \inner{\bphi(s,a')}{\bw_h^t}} \bigg ] \cdot \| \bw_{h-i} - \bu_{h-i} \|_2 \\
& = 2 \eta \cdot \| \bw_{h-i} - \bu_{h-i} \|_2.
\end{align*}
It follows that
\begin{align*}
\| \bphi_{\piw,h-i}(s) - \bphi_{\piu,h-i}(s) \|_2 \le  2 \eta \cdot \| \bw_{h-i} - \bu_{h-i} \|_2.
\end{align*}
With \Cref{defn:linear_mdp}, this implies that 
\begin{align*}
\|  \cT_{\piw,h-i} - \cT_{\piu,h-i} \|_\op & \le \int  \| \bphi_{\piw,h-i}(s) - \bphi_{\piu,h-i}(s) \|_2  \| \rmd \bmu_{h-i-1}(s) \|_2 \le 2  \sqrt{d} \eta \| \bw_{h-i} - \bu_{h-i} \|_2. 
\end{align*}
By \Cref{lem:state_vis_norm_bound}, we can bound $\| \btheta_h^\top  \left ( \prod_{j = h - i + 1}^{h} \cT_{\piw,j} \right ) \|_2$. Thus, returning to the error decomposition given above, we have
\begin{align*}
| V^{\piw}_0(s_1) - V^{\piu}_0(s_1) | &  \le \sum_{h=1}^H   \sum_{i=0}^{h-1} \left | \btheta_h^\top \left ( \prod_{j = h - i + 1}^{h} \cT_{\piw,j} \right ) ( \cT_{\piw,h-i} - \cT_{\piu,h-i}) \bphi_{\piu,h-i-1} \right | \\
& \le \sqrt{d} \sum_{h=1}^H   \sum_{i=0}^{h-1} \|  \cT_{\piw,h-i} - \cT_{\piu,h-i} \|_\op \| \bphi_{\piu,h-i-1} \|_2 \\
& \le  2  d \eta \sum_{h=1}^H   \sum_{i=0}^{h-1} \| \bw_{h-i} - \bu_{h-i} \|_2 \\
& \le 2  d H \eta \sum_{h=1}^H \| \bw_{h} - \bu_{h} \|_2 .
\end{align*}
\end{proof}

\begin{lem}\label{lem:softmax_qst_approx}
Let $\wst$ denote the weights such that $\Qst_h(s,a) = \inner{\bphi(s,a)}{\wst_h}$, and $\piwst$ the restricted-action linear softmax policy with action sets $(\cAtil_s)_{s \in \cS}$ as defined in \Cref{lem:action_covering} with $\delta = \frac{\epsilon}{2H \cdot e^2}$. Then
\begin{align*}
| V_0^{\piwst}(s_1) - \Vst_0(s_1) | \le \epsilon
\end{align*}
as long as $\eta \ge 30 d H \log (1 + 120 H d/\epsilon) \cdot \frac{1}{\epsilon}$. 
\end{lem}
\begin{proof}
We prove this by induction. Assume that at step $h$, for all $s$, we have $|\Vst_h(s) - V_h^{\pi^{\wst}}(s)| \le \delta_h$ for some $\delta_h$. Then,
\begin{align*}
|Q_{h-1}^{\piwst}(s,a) - \Qst_{h-1}(s,a)| & = \left | \int (V_{h}^{\piwst}(s') - \Vst_h(s')) \rmd P_{h-1}(s' \mid s,a) \right | \\
& \le \int |V_{h}^{\piwst}(s') - \Vst_h(s')| \rmd P_{h-1}(s' \mid s,a) \\
& \le \delta_h \cdot \int \rmd P_{h-1}(s' \mid s,a) \\
& \le \delta_h.
\end{align*} 
Thus,
\begin{align*}
V_{h-1}^{\piwst}(s) & = \frac{\sum_{a \in \cAtil_s} e^{\eta \inner{\bphi(s,a)}{\wst_{h-1}}} Q_{h-1}^{\pi^{\wst}}(s,a)}{\sum_{a \in \cAtil_s} e^{\eta \inner{\bphi(s,a)}{\wst_{h-1}}}} \\
& =  \frac{\sum_{a \in \cAtil_s} e^{\eta \Qst_{h-1}(s,a)} Q_{h-1}^{\pi^{\wst}}(s,a)}{\sum_{a \in \cAtil_s} e^{\eta \Qst_{h-1}(s,a)}}\\
& \ge \frac{\sum_{a \in \cAtil_s} e^{\eta \Qst_{h-1}(s,a)} \Qst_{h-1}(s,a)}{\sum_{a \in \cAtil_s} e^{\eta \Qst_{h-1}(s,a)}} - \delta_h.
\end{align*}
By \Cref{lem:softmax_approx}, as long as $\eta \ge H \log |\cAtil_s| / \delta_h$, we can lower bound
\begin{align*}
\frac{\sum_{a \in \cAtil_s} e^{\eta \Qst_{h-1}(s,a)} \Qst_{h-1}(s,a)}{\sum_{a \in \cAtil_s} e^{\eta \Qst_{h-1}(s,a)}} -  \delta_h \ge \max_{a \in \cAtil_s} \Qst_{h-1}(s,a) - (1+1/H) \delta_h .
\end{align*}
Furthermore, by \Cref{lem:action_covering} and our choice of $\cAtil_s$, we have
\begin{align*}
 \max_{a \in \cAtil_s} \Qst_{h-1}(s,a) - (1+1/H) \delta_h & \ge  \max_{a \in \cA} \Qst_{h-1}(s,a) - (1+1/H) \delta_h - \frac{\epsilon}{2H \cdot e^2} \\
 &= \Vst_{h-1}(s) - (1+1/H) \delta_h - \frac{\epsilon}{2H \cdot e^2}.
\end{align*}
Define recursively $\delta_{h-1} = (1+2/H) \delta_h$ and $\delta_H = \frac{\epsilon}{e^2}$. 
Then $\delta_{h} = (1+2/H)^{H-h} \frac{\epsilon}{e^2} \ge \frac{\epsilon}{e^2}$, so 
\begin{align*}
\Vst_{h-1}(s) - (1+1/H) \delta_h - \frac{\epsilon}{2H \cdot e^2} \ge \Vst_{h-1}(s) - (1+1/H) \delta_h  - \delta_{h}/H =  \Vst_{h-1}(s) - \delta_{h-1}.
\end{align*}
So $|\Vst_{h}(s) - V_{h}^{\piwst}(s)| \le \delta_{h-1}$ for all $s$, which proves the inductive step.

For the base case, we have
\begin{align*}
 V_H^{\piwst}(s) - \Vst_H(s)& =  \frac{\sum_{a \in \cAtil_s} e^{\eta \Qst_H(s,a)} \nu_H(s,a)}{\sum_{a \in \cAtil_s} e^{\eta \Qst_H(s,a)}} - \max_a \nu_H(s,a) \\
& \ge   \max_{a \in \cAtil_s} \nu_H(s,a) - \max_a \nu_H(s,a) - \delta_H/2 \\
& \ge -\delta_H
\end{align*}
where the first inequality holds by \Cref{lem:softmax_approx} as long as $\eta \ge 2\log | \cAtil_s|/\delta_H$, and the second inequality holds by \Cref{lem:action_covering} and our choice of $\cAtil_s$ and $\delta_H$. This proves the base case, since $ V_H^{\piwst}(s) \le \Vst_H(s)$. 

Recursing this all the way back, we conclude that
\begin{align*}
V_0^{\piwst}(s_1) \ge \Vst_0(s_1) - \delta_0
\end{align*}
for $\delta_0 = (1+2/H)^H \delta_H \le \epsilon$ since $(1+2/x)^x \le e^2$ for all $x$.

For this argument to hold, we must choose $\eta \ge 2 \log | \cAtil_s | /\delta_H$ and $\eta \ge H \log | \cAtil_s|/ \delta_h$ for all $s$ and $h$. By \Cref{lem:action_covering} and our choice of $\cAtil_s$, we can bound 
\begin{align*}
|\cAtil_s| \le (1 + 16 e^2 H^2 \sqrt{d}/\epsilon)^d \le (1 + 120 H d/\epsilon)^{2d}
\end{align*}
so, since $\delta_H = \epsilon/e^2 \le \delta_h$, it suffices that we take $\eta \ge 30 d H \log (1 + 120 H d/\epsilon) \cdot \frac{1}{\epsilon}$.

\end{proof}

\begin{lem}\label{lem:policy_class_suff}
Let $\eta = 30 d H \log (1 + 120 H d/\epsilon) \cdot \frac{1}{\epsilon}$ and $\cW$ an $\frac{\epsilon}{4 d H^2 \eta}$-net of $\Ball^{d}(2H\sqrt{d})$. Let $\Pi$ denote the set of restricted-action linear softmax policy with vectors $\bw \in \cW^H$, parameter $\eta$, and action sets $(\cAtil_s)_{s \in \cS}$ as defined in \Cref{lem:action_covering} with $\delta = \frac{\epsilon}{2H \cdot e^2}$.
Then for any MDP and reward function, there exists some $\pi \in \Pi$ such that $| V_0^\pi - \Vst_0 | \le \epsilon$, and 
\begin{align*}
| \Pi |  \le \Big (1 + \frac{480 H^4 d^{5/2}  \log(1+120 Hd/\epsilon)}{\epsilon^2} \Big )^{dH}.
\end{align*}
\end{lem}
\begin{proof}
Consider some MDP and reward function, and let $\{ \wst_h \}_{h=1}^H$ denote the optimal $Q$-function linear representation: $\Qst_h(s,a) = \inner{\bphi(s,a)}{\wst_h}$. Let $\bwtil$ denote the vector in $\cW^H$ such that $\sum_{h=1}^H \| \wst_h - \bwtil_h \|_2$ is minimized. Then by \Cref{lem:softmax_val_approx} and \Cref{lem:softmax_qst_approx}, as long as $\eta \ge 30 d H \log (1 + 120 H d/\epsilon) \cdot \frac{1}{\epsilon}$, we have
\begin{align*}
| V_0^{\pi^{\bwtil}}(s_1) - \Vst_0(s_1) | & \le | V_0^{\pi^{\bwtil}}(s_1) - V_0^{\piwst}(s_1)| + | V_0^{\piwst}(s_1) - \Vst_0(s_1) | \\
& \le 2  d H \eta \sum_{h=1}^H \| \wst_{h} - \bwtil_{h} \|_2 + \epsilon/2.
\end{align*}
The first conclusion then follows as long as we can find some $\bwtil$ such that
\begin{align*}
2  d H \eta \sum_{h=1}^H \| \wst_{h} - \bwtil_{h} \|_2 \le \epsilon/2.
\end{align*}
However, by \Cref{lem:wpi_norm_bound}, we can bound $\| \wst_h \|_2 \le 2 H \sqrt{d}$. Therefore, since $\cW$ is a $\frac{\epsilon}{4 d H^2 \eta}$-net of $\Ball^d(2H\sqrt{d})$, for each $h$ there will exist some $\bwtil_h \in \cW$ such that $\| \wst_h - \bwtil_h \|_2 \le \frac{\epsilon}{4 d H^2 \eta}$, which implies that we can find $\bwtil \in \cW^d$ such that
\begin{align*}
2  d H \eta \sum_{h=1}^H \| \wst_{h} - \bwtil_{h} \|_2 \le \epsilon/2,
\end{align*}
which gives the first conclusion.

To bound the size of $\Pi$, we apply \Cref{lem:euc_ball_cover} and our choice of $\eta$ to bound
\begin{align*}
|\cW| \le (1 + \frac{16 H^3 d^{3/2} \eta}{\epsilon})^d \le (1 + \frac{480 H^4 d^{5/2}  \log(1+120Hd/\epsilon)}{\epsilon^2} )^{d}.
\end{align*}
The bound on $|\Pi|$ follows since $|\Pi| =  | \cW|^H$.

\end{proof}

\section{Policy Elimination}\label{sec:policy_elim}

Throughout this section, assuming we have run for some number of episodes $K$, we let $(\cF_{\tau})_{\tau=1}^K$ the filtration on this, with $\cF_\tau$ the filtration up to and including episode $\tau$. We also let $\cF_{\tau,h}$ denote the filtration on all episodes $\tau' < \tau$, and on steps $h' = 1,\ldots,h$ of episode $\tau$.

\subsection{Estimating Feature-Visitations and Rewards}

\begin{lem}\label{lem:Tpi_est_fixed}
Assume that we have collected some data $\{ (s_{h-1,\tau},a_{h-1,\tau},s_{h,\tau}) \}_{\tau = 1}^K$, where, for each $\tau'$, $s_{h,\tau'} | \cF_{h-1,\tau'}$ is independent of $\{ (s_{h-1,\tau},a_{h-1,\tau},s_{h,\tau}) \}_{\tau \neq \tau'}$. Denote $\bphi_{h-1,\tau} = \bphi(s_{h-1,\tau},a_{h-1,\tau})$ and $\bLambda_{h-1} = \sum_{\tau=1}^K \bphi_{h-1,\tau} \bphi_{h-1,\tau}^\top + \lambda I$. Fix $\pi$ and let
\begin{align*}
\cThat_{\pi,h} = \left (  \sum_{\tau=1}^K \bphi_{\pi,h}(s_{h,\tau}) \bphi_{h-1,\tau}^\top \right ) \bLambda_{h-1}^{-1}.
\end{align*}
Fix $\bv \in \R^d$ satisfying $|\bv^\top \bphi_{\pi,h}(s)| \le 1$ for all $s$ and $\bu \in \R^d$.
Then with probability at least $1-\delta$, we can bound
\begin{align*}
| \bv^\top (\cT_{\pi,h} - \cThat_{\pi,h}) \bu | \le \left ( 2 \sqrt{\log 2/\delta} + \frac{\log 2/\delta}{\sqrt{\lammin(\bLambda_{h-1})}} + \sqrt{\lambda} \| \cT_{\pi,h}^\top \bv \|_2 \right ) \cdot \| \bu \|_{\bLambda_{h-1}^{-1}}.
\end{align*}
\end{lem}
\begin{proof}
Let $\frakD = \{ (s_{h-1,\tau},a_{h-1,\tau}) \}_{\tau = 1}^K$, our data collected at step $h-1$. Then by our assumption on the independence of $s_{h,\tau}$, we have that $s_{h,\tau} | \cF_{h-1,\tau}$ has the same distribution as $s_{h,\tau} | (\cF_{h-1,\tau},\frakD)$. Conditioning on $\frakD$, the $\bphi_{h-1,\tau}$ vectors are fixed, so $\bLambda_{h-1}$ is also fixed. Note that
\begin{align*}
\cT_{\pi,h} & = \int \bphi_{\pi,h}(s) \rmd \bmu_{h-1}(s)^\top \\
& = \int \bphi_{\pi,h}(s) \rmd \bmu_{h-1}(s)^\top \left (  \sum_{\tau=1}^K \bphi_{h-1,\tau} \bphi_{h-1,\tau}^\top \right ) \bLambda_{h-1}^{-1} + \lambda \int \bphi_{\pi,h}(s) \rmd \bmu_{h-1}(s)^\top \bLambda_{h-1}^{-1} \\
& =  \sum_{\tau=1}^K \left ( \int \bphi_{\pi,h}(s) \rmd \bmu_{h-1}(s)^\top   \bphi_{h-1,\tau} \right ) \bphi_{h-1,\tau}^\top  \bLambda_{h-1}^{-1} + \lambda \int \bphi_{\pi,h}(s) \rmd \bmu_{h-1}(s)^\top \bLambda_{h-1}^{-1} \\
& = \sum_{\tau=1}^K \Exp[\bphi_{\pi,h}(s_{h,\tau}) | \cF_{h-1,\tau}] \bphi_{h-1,\tau}^\top  \bLambda_{h-1}^{-1} + \lambda \int \bphi_{\pi,h}(s) \rmd \bmu_{h-1}(s)^\top \bLambda_{h-1}^{-1} \\
& = \sum_{\tau=1}^K \Exp[\bphi_{\pi,h}(s_{h,\tau}) | \cF_{h-1,\tau}] \bphi_{h-1,\tau}^\top  \bLambda_{h-1}^{-1} + \lambda \cT_{\pi,h} \bLambda_{h-1}^{-1}
\end{align*}
so
\begin{align*}
| \bv^\top (\cT_{\pi,h} -  \cThat_{\pi,h}) \bu | & \le \underbrace{\Big | \sum_{\tau=1}^K \bv^\top \left (  \Exp[\bphi_{\pi,h}(s_{h,\tau}) | \cF_{h-1,\tau}]  - \bphi_{\pi,h}(s_{h,\tau}) \right ) \bphi_{h-1,\tau}^\top  \bLambda_{h-1}^{-1} \bu \Big |}_{(a)} + \underbrace{\Big | \lambda \bv^\top \cT_{\pi,h} \bLambda_{h-1}^{-1} \bu \Big |}_{(b)} .
\end{align*}
Conditioned on $\frakD$, $(a)$ is simply the sum of mean 0 random variables, where the $\tau$th random variable has magnitude bounded as
\begin{align*}
 | \bv^\top \left (  \Exp[\bphi_{\pi,h}(s_{h,\tau}) | \cF_{h-1,\tau}]  - \bphi_{\pi,h}(s_{h,\tau}) \right ) \bphi_{h-1,\tau}^\top  \bLambda_{h-1}^{-1} \bu | & \le 2 | \bphi_{h-1,\tau}^\top \bLambda_{h-1}^{-1} \bu | \\
& \le 2 \| \bphi_{h-1,\tau} \|_{\bLambda_{h-1}^{-1}} \| \bu \|_{\bLambda_{h-1}^{-1}} \\
& \le 2 \| \bu \|_{\bLambda_{h-1}^{-1}} / \sqrt{\lammin(\bLambda_{h-1})}
\end{align*}
Furthermore, the variance of each term in $(a)$ is bounded as
\begin{align*}
& \Var \left [ \bv^\top \left (  \Exp[\bphi_{\pi,h}(s_{h,\tau}) | \cF_{h-1,\tau}]  - \bphi_{\pi,h}(s_{h,\tau}) \right ) \bphi_{h-1,\tau}^\top  \bLambda_{h-1}^{-1} \bu | \cF_{h-1} \right ] \\
& = \Exp \left [ \left ( \bv^\top \left (  \Exp[\bphi_{\pi,h}(s_{h,\tau}) | \cF_{h-1,\tau}]  - \bphi_{\pi,h}(s_{h,\tau}) \right ) \bphi_{h-1,\tau}^\top  \bLambda_{h-1}^{-1} \bu \right )^2 | \cF_{h-1} \right ] \\
& \le \bu^\top \bLambda_{h-1}^{-1} \bphi_{h-1,\tau} \bphi_{h-1,\tau}^\top \bLambda_{h-1}^{-1} \bu.
\end{align*}
It follows that, by Bernstein's Inequality, we can bound, with probability at least $1-\delta$ conditioned on $\frakD$:
\begin{align*}
(a) & \le 2\sqrt{\sum_{\tau=1}^K \bu^\top \bLambda_{h-1}^{-1} \bphi_{h-1,\tau} \bphi_{h-1,\tau}^\top \bLambda_{h-1}^{-1} \bu \cdot \log \frac{2}{\delta}} + \frac{2 \| \bu \|_{\bLambda_{h-1}^{-1}}}{\sqrt{\lammin(\bLambda_{h-1})}} \cdot \log \frac{2}{\delta} \\
& \le 2( \sqrt{\log 2/\delta} + \frac{\log 2/\delta}{\sqrt{\lammin(\bLambda_{h-1})}}) \cdot \| \bu \|_{\bLambda_{h-1}^{-1}}.
\end{align*}
In other words,
\begin{align*}
\Pr \left [(a) \ge 2( \sqrt{\log 2/\delta} + \frac{\log 2/\delta}{\sqrt{\lammin(\bLambda_{h-1})}}) \cdot \| \bu \|_{\bLambda_{h-1}^{-1}} | \frakD \right ] \le  \delta
\end{align*}
so, by the law of total probability, for any distribution $F$ over $\frakD$,
\begin{align*}
& \Pr \left [(a) \ge 2( \sqrt{\log 2/\delta} + \min \{ 1, \lambda^{-1} \} \log 2/\delta) \cdot \| \bu \|_{\bLambda_{h-1}^{-1}} \right ] \\
& = \int \Pr \left [(a) \ge 2( \sqrt{\log 2/\delta} +\min \{ 1, \lambda^{-1} \} \log 2/\delta) \cdot \| \bu \|_{\bLambda_{h-1}^{-1}} | \frakD \right ] \rmd F(\frakD) \\
& \le \delta \int \rmd F(\frakD) \\
& = \delta. 
\end{align*}

We can also bound
\begin{align*}
(b) \le \sqrt{\lambda} \| \bu \|_{\bLambda_{h-1}^{-1}} \| \cT_{\pi,h}^\top \bv \|_2.
\end{align*}
Combining these gives the result.
\end{proof}

\begin{lem}\label{lem:phihat_est_fixed}
Fix $\pi$ and let
\begin{align*}
\bphihat_{\pi,h} = \cThat_{\pi,h} \cThat_{\pi,h-1} \ldots \cThat_{\pi,2} \cThat_{\pi,1} \bphi_{\pi,0}.
\end{align*}
Fix $\bu \in \cS^{d-1}$ or $\bu$ a valid reward vector as defined by \Cref{defn:linear_mdp}. Then with probability at least $1-\delta$:
\begin{align*}
|\inner{\bu}{\bphi_{\pi,h} - \bphihat_{\pi,h}}| \le \sum_{i=1}^{h-1} \left ( 2 \sqrt{\log \frac{2H}{\delta}} + \frac{\log \frac{2H}{\delta}}{\sqrt{\lammin(\bLambda_{i})}} + \sqrt{d\lambda} \right ) \cdot \| \bphihat_{\pi,i} \|_{\bLambda_{i}^{-1}}.
\end{align*}
\end{lem}
\begin{proof}
Note that
\begin{align*}
\bphi_{\pi,h} - \bphihat_{\pi,h} & = \cT_{\pi,h} \bphi_{\pi,h-1} - \cThat_{\pi,h} \bphihat_{\pi,h-1} \\
& = \cT_{\pi,h} ( \bphi_{\pi,h-1} - \bphihat_{\pi,h-1}) + (\cT_{\pi,h} - \cThat_{\pi,h}) \bphihat_{\pi,h-1} .
\end{align*}
Thus, unrolling this all the way back, we get
\begin{align*}
\bphi_{\pi,h} - \bphihat_{\pi,h} & = \sum_{i=0}^{h-2} \left ( \prod_{j = h - i + 1}^{h} \cT_{\pi,j} \right ) ( \cT_{\pi,h-i} - \cThat_{\pi,h-i}) \bphihat_{\pi,h-i-1}
\end{align*}
where we order the product $\prod_{j = h - i + 1}^{h} \cT_{\pi,j} = \cT_{\pi,h} \cT_{\pi,h-1} \ldots \cT_{\pi,h-i+1}$. It follows that
\begin{align*}
|\inner{\bu}{\bphi_{\pi,h} - \bphihat_{\pi,h}}| & \le \sum_{i=0}^{h-2} \left | \bu^\top \left ( \prod_{j = h - i + 1}^{h} \cT_{\pi,j} \right ) ( \cT_{\pi,h-i} - \cThat_{\pi,h-i}) \bphihat_{\pi,h-i-1} \right |.
\end{align*}
Denote $\bv_i := \bu^\top \left ( \prod_{j = h - i + 1}^{h} \cT_{\pi,j} \right )$. By \Cref{lem:state_vis_norm_bound} and our assumption on $\bu$, we can bound $\| \bv_i \|_2 \le \sqrt{d}$ and also have that for all $s,a$, $|\bv_i^\top \bphi(s,a) | \le 1$, which implies
\begin{align*}
| \bv_i^\top \bphi_{\pi,j}(s) | = \left | \sum_{a \in \cA} \bv_i^\top \bphi(s,a) \pi_h(a|s) \right | \le \sum_{a \in \cA} \pi_h(a|s) = 1. 
\end{align*}
We can therefore apply \Cref{lem:Tpi_est_fixed} to get that, with probability at least $1-\delta$, for all $i$,
\begin{align*}
 \left | \bv_i^\top ( \cT_{\pi,h-i} - \cThat_{\pi,h-i}) \bphihat_{\pi,h-i-1} \right | \le \left ( 2 \sqrt{\log \frac{2H}{\delta}} + \frac{\log \frac{2H}{\delta}}{\sqrt{\lammin(\bLambda_{h-i-1})}} + \sqrt{\lambda} \| \cT_{\pi,h}^\top \bv_i \|_2 \right ) \cdot \| \bphihat_{\pi,h-i-1} \|_{\bLambda_{h-i-1}^{-1}}
\end{align*}
By \Cref{lem:state_vis_norm_bound}, the definition of $\bv_i$, and our assumption on $\bu$, we can bound $ \| \cT_{\pi,h}^\top \bv_i \|_2 \le \sqrt{d}$.
Summing over $i$ proves the result.
\end{proof}

\begin{lem}\label{lem:phihat_est_2norm_fixed}
With probability at least $1-\delta$:
\begin{align*}
\| \bphihat_{\pi,h} - \bphi_{\pi,h} \|_2 \le d \sum_{h'=1}^{h-1} \left ( 2 \sqrt{\log \frac{2H d}{\delta}} + \frac{\log \frac{2H d}{\delta}}{\sqrt{\lammin(\bLambda_{h'})}} + \sqrt{d\lambda} \right ) \cdot \| \bphihat_{\pi,h'} \|_{\bLambda_{h'}^{-1}}.
\end{align*}
\end{lem}
\begin{proof}
We have:
\begin{align*}
\| \bphihat_{\pi,h} - \bphi_{\pi,h} \|_2 \le \| \bphihat_{\pi,h} - \bphi_{\pi,h} \|_1 = \sum_{i=1}^d | [\bphihat_{\pi,h}]_i - [\bphi_{\pi,h}]_i | = \sum_{i=1}^d | \inner{\be_i}{\bphihat_{\pi,h} - \bphi_{\pi,h}} | .
\end{align*}
Since $\be_i \in \cS^{d-1}$, we can apply \Cref{lem:phihat_est_fixed} to bound, with probability $1-\delta/d$,
\begin{align*}
 | \inner{\be_i}{\bphihat_{\pi,h} - \bphi_{\pi,h}} | \le \sum_{h'=1}^{h-1} \left ( 2 \sqrt{\log \frac{2H d}{\delta}} + \frac{\log \frac{2H d}{\delta}}{\sqrt{\lammin(\bLambda_{h'})}} + \sqrt{d\lambda} \right ) \cdot \| \bphihat_{\pi,h'} \|_{\bLambda_{h'}^{-1}}.
\end{align*}
Summing over $i$ gives the result. 
\end{proof}

\begin{lem}\label{lem:reward_est_fixed}
Assume we have collected data $\{ \bphi(s_{h,\tau},a_{h,\tau}),r_h(s_{h,\tau},a_{h,\tau})\}_{\tau = 1}^K$ and that for each $\tau'$, $r_h(s_{h,\tau'},a_{h,\tau'}) | (s_{h,\tau'},a_{h,\tau'})$ is independent of $\{ (s_{h,\tau},a_{h,\tau}) \}_{\tau \neq \tau'}$. Let
\begin{align*}
\bthetahat_h = \argmin_{\btheta} \sum_{\tau = 1}^K (r_{h,\tau} - \inner{\bphi_{h,\tau}}{\btheta})^2 + \lambda \| \btheta \|_2^2
\end{align*}
and fix $\bu \in \R^d$ that is independent of $\{ \bphi(s_{h,\tau},a_{h,\tau}),r_h(s_{h,\tau},a_{h,\tau})\}_{\tau = 1}^K$.
Then with probability at least $1-\delta$:
\begin{align*}
| \inner{\bu}{\bthetahat_h - \btheta_h}| \le \left  ( \sqrt{\log 2/\delta} +  \frac{ \log 2/\delta}{\sqrt{\lammin(\bLambda_h)}} + \sqrt{d\lambda} \right ) \cdot \| \bu \|_{\bLambda_h^{-1}}.
\end{align*}
\end{lem}
\begin{proof}
Let $\frakD= \{ (s_{h,\tau},a_{h,\tau}) \}_{\tau = 1}^K$. Then by our assumption on the independence of $r_{h,\tau}$, we have that $r_{h,\tau} | (s_{h,\tau},a_{h,\tau})$ has the same distribution as $r_{h,\tau} | \frakD$. Conditioning on $\frakD$, the $\bphi_{h,\tau}$ vectors are fixed, so $\bLambda_{h}$ is also fixed. 

By construction we have
\begin{align*}
\bthetahat_h = \bLambda_h^{-1} \sum_{\tau = 1}^K \bphi_{h,\tau} r_{h,\tau}.
\end{align*}
Furthermore: 
\begin{align*}
\btheta_h & = \bLambda_h^{-1} \bLambda_h \btheta_h  = \bLambda_h^{-1} \sum_{\tau = 1}^K \bphi_{h,\tau} \Exp[r_{h,\tau}| \cF_{h-1,\tau}] + \lambda \bLambda_h^{-1} \btheta_h.
\end{align*}
Thus,
\begin{align*}
| \inner{\bu}{\bthetahat_h - \btheta_h}| & \le \underbrace{\left |  \sum_{\tau = 1}^K \bu^\top  \bLambda_h^{-1} \bphi_{h,\tau} (r_{h,\tau} - \Exp[r_{h,\tau}| \cF_{h-1,\tau}]) \right |}_{(a)} + \underbrace{\left | \lambda \bu^\top \bLambda_h^{-1} \btheta_h \right  |}_{(b)} .
\end{align*}
Since $R_{h,\tau} \in [0,1]$ almost surely, we can bound
\begin{align*}
|\bu^\top  \bLambda_h^{-1} \bphi_{h,\tau} (r_{h,\tau} - \Exp[r_{h,\tau}| \cF_{h-1,\tau}])| \le \| \bu \|_{\bLambda_h^{-1}} \| \bphi_{h,\tau} \|_{\bLambda_h^{-1}} \le \| \bu \|_{\bLambda_h^{-1}} / \sqrt{\lammin(\bLambda_h)}.
\end{align*}
Furthermore, we can bound
\begin{align*}
 \Var & \left [ \bu^\top  \bLambda_h^{-1} \bphi_{h,\tau} (r_{h,\tau} - \Exp[r_{h,\tau}| \cF_{h-1,\tau}]) | \frakD \right ] \\
& = \Exp  \left [ ( \bu^\top  \bLambda_h^{-1} \bphi_{h,\tau} (r_{h,\tau} - \Exp[r_{h,\tau}| \cF_{h-1,\tau}]) )^2 | \frakD \right ] \\
& \le \bu^\top \bLambda_h^{-1} \bphi_{h,\tau} \bphi_{h,\tau}^\top \bLambda_h^{-1} \bu.
\end{align*}
By Bernstein's inequality, we then have, with probability at least $1-\delta$ conditioned on $\frakD$:
\begin{align*}
(a) & \le \sqrt{ \sum_{\tau = 1}^K \bu^\top \bLambda_h^{-1} \bphi_{h,\tau} \bphi_{h,\tau}^\top \bLambda_h^{-1} \bu \cdot \log 2/\delta} + \frac{ \| \bu \|_{\bLambda_h^{-1}} \cdot \log 2/\delta}{\sqrt{\lammin(\bLambda_h)}} \\
& \le ( \sqrt{\log 2/\delta} +  \frac{ \log 2/\delta}{\sqrt{\lammin(\bLambda_h)}} ) \cdot \| \bu \|_{\bLambda_h^{-1}}.
\end{align*}
Applying the Law of Total Probability as in \Cref{lem:Tpi_est_fixed}, we obtain
\begin{align*}
\Pr \left [ (a) \ge ( \sqrt{\log 2/\delta} +  \frac{ \log 2/\delta}{\sqrt{\lammin(\bLambda_h)}} ) \cdot \| \bu \|_{\bLambda_h^{-1}} \right ] \le \delta.
\end{align*}
By \Cref{defn:linear_mdp}, we can also bound
\begin{align*}
(b) \le \sqrt{\lambda} \| \bu \|_{\bLambda_h^{-1}} \| \btheta_h \|_2 \le \sqrt{d \lambda} \| \bu \|_{\bLambda_h^{-1}}.
\end{align*}
Combining these proves the result. 
\end{proof}

\subsection{Correctness and Sample Complexity of \algname}

\begin{algorithm}[h]
\begin{algorithmic}[1]
\State \textbf{input:} tolerance $\epsilon$, confidence $\delta$, policy set $\Pi$
\State  $\ell_0 \leftarrow \lceil \log_2 \frac{d^{3/2}}{H} \rceil$, $\Pi_{\ell_0} \leftarrow \Pi$, $\bphihat^{1}_{\pi,1} \leftarrow \Exp_{a \sim \pi_1(\cdot | s_1)} [ \bphi(s_1,a)], \forall \pi \in \Pi$
\For{$\ell = \ell_0, \ell_0 + 1,\ldots, \lceil \log \frac{4}{\epsilon} \rceil$}
	\State $\epsilon_\ell \leftarrow 2^{-\ell}$, $\beta_\ell \leftarrow 64 H^4 \log \frac{4 H^2 | \Pi_\ell| \ell^2}{\delta}$
	\For{$h = 1,2,\ldots,H$}
		\State Run procedure described in \Cref{cor:gopt_exp_design} with parameters  \label{line:explore}
		\begin{align*}
			\epsilon_{\mathrm{exp}} \leftarrow \frac{\epsilon_\ell^2}{\beta_\ell}, \quad \delta \leftarrow \frac{\delta}{2 H \ell^2}, \quad \lamun \leftarrow \log \frac{4H^2|\Pi_\ell| \ell^2}{\delta}, \quad \Phi \leftarrow \Phi_{h,\ell} := \{ \bphihat_{\pi,h}^\ell : \pi \in \Pi_\ell \}
		\end{align*}
		\hspace{2.75em} and denote returned data as $\{ (s_{h,\tau}, a_{h,\tau}, r_{h,\tau}, s_{h+1,\tau}) \}_{\tau=1}^{K_{h,\ell}}$, for $K_{h,\ell}$ total number of episodes run \awarxiv{I think just make this an alg?}, and covariates
		\begin{align*}
		\bLambda_{h,\ell} \leftarrow \sum_{\tau = 1}^{K_{h,\ell}} \bphi(s_{h,\tau}, a_{h,\tau}) \bphi(s_{h,\tau}, a_{h,\tau})^\top + 1/d \cdot I
		\end{align*}
		\For{$\pi \in \Pi_\ell$} \hfill {\color{blue} \texttt{// Estimate feature-visitations for active policies}}
			\State $\bphihat_{\pi,h+1}^\ell \leftarrow \Big ( \sum_{\tau = 1}^{K_{h,\ell}} \bphi_{\pi,h+1}(s_{h+1,\tau}) \bphi_{h,\tau}^\top \bLambda_{h,\ell}^{-1} \Big ) \bphihat_{\pi,h}^\ell$ \label{line:est_feature_visit}
		\EndFor
		\State $\bthetahat_h^\ell \leftarrow \bLambda_{h,\ell}^{-1} \sum_{\tau = 1}^{K_{h,\ell}} \bphi_{h,\tau} r_{h,\tau}$ \hfill {\color{blue} \texttt{// Estimate reward vectors}}
	\EndFor
	\State {\color{blue} \texttt{// Remove provably suboptimal policies from active policy set}}
	\begin{align*}
	\Pi_{\ell+1} \leftarrow  \Pi_\ell \backslash \Big \{ \pi \in \Pi_\ell \ : \ \Vhat_0^\pi < \sup_{\pi' \in \Pi_\ell} \Vhat_0^{\pi'} - 2\epsilon_\ell \Big \} \quad \text{for} \quad \Vhat_0^\pi := \tsum_{h=1}^H \inner{\bphihat_{\pi,h}^\ell}{\bthetahat_h^\ell}
	\end{align*}
	\If{$|\Pi_{\ell+1}| = 1$}\label{line:early_term} \textbf{return} $\pi \in \Pi_{\ell+1}$\EndIf
\EndFor
\State \textbf{return} any $\pi \in \Pi_{\ell+1}$
\end{algorithmic}
\caption{ \textbf{P}olicy Learning via \textbf{E}xperiment \textbf{De}sign in \textbf{L}inear MDPs (\algname, full version)}
\label{alg:linear_exp_design2}
\end{algorithm}

\begin{lem}\label{lem:est_good_event}
Let $\cEest^{\ell,h}$ denote the event on which, for all $\pi \in \Pi_\ell$:
\begin{align*}
& |\inner{\btheta_{h+1}}{\bphihat_{\pi,h+1}^\ell - \bphi_{\pi,h+1}}| \le \sum_{i=1}^{h} \left ( 3 \sqrt{\log \frac{4 H^2 | \Pi_\ell| \ell^2}{\delta}} + \frac{\log \frac{4 H^2 | \Pi_\ell| \ell^2}{\delta}}{\sqrt{\lammin(\bLambda_{i,\ell})}}  \right ) \cdot \| \bphihat_{\pi,i}^\ell \|_{\bLambda_{i,\ell}^{-1}}, \\
& \| \bphihat_{\pi,h+1}^\ell - \bphi_{\pi,h+1} \|_2 \le d \sum_{i=1}^{h} \left ( 3 \sqrt{\log \frac{4 H^2 d | \Pi_\ell| \ell^2}{\delta}} + \frac{\log \frac{4 H^2 d | \Pi_\ell| \ell^2}{\delta}}{\sqrt{\lammin(\bLambda_{i,\ell})}}  \right ) \cdot \| \bphihat_{\pi,i}^\ell \|_{\bLambda_{i,\ell}^{-1}} , \\
& | \inner{\bphihat_{\pi,h}^\ell}{\bthetahat_h - \btheta_h}| \le \left  ( 2\sqrt{\log \frac{4 H^2 | \Pi_\ell| \ell^2}{\delta} } +  \frac{ \log \frac{4 H^2 | \Pi_\ell| \ell^2}{\delta}}{\sqrt{\lammin(\bLambda_{h,\ell})}}  \right ) \cdot \| \bphihat_{\pi,h}^\ell \|_{\bLambda_{h,\ell}^{-1}}.
\end{align*}
Then $\Pr[(\cEest^{\ell,h})^c] \le  \frac{\delta}{2H\ell^2}$.
\end{lem}
\begin{proof}
Note that the data collection procedure outlined in \Cref{cor:gopt_exp_design} collects data that satisfies the independence requirement of \Cref{lem:Tpi_est_fixed} and \Cref{lem:reward_est_fixed}, since \Cref{cor:gopt_exp_design} operates on the $h$-truncated-horizon MDP defined with respect to our original MDP (see \Cref{def:trun_horizon_mdp} and following discussion), so by construction the data obtained at step $h$ is independent of $s_{h+1}$ and $r_h(s_h,a_h)$. Note also that $\bphihat_{\pi,h}^\ell$ is independent of $\{ r_{h,\tau}^\ell \}_{\tau = 1}^{K_{h,\ell}} | \{ (s_{h,\tau},a_{h,\tau}) \}_{\tau=1}^{K_{h,\ell}}$, since we construct $\bphihat_{\pi,h}^\ell$ using only observations taken at step $h-1$. 

The result follows by \Cref{lem:phihat_est_fixed}, \Cref{lem:phihat_est_2norm_fixed}, and \Cref{lem:reward_est_fixed}, and setting $\lambda = 1/d$. 
\end{proof}

\begin{lem}\label{lem:exp_good_event}
Let $\cEexp^{\ell,h}$ denote the event on which:
\begin{itemize}
\item The exploration procedure on \Cref{line:explore} terminates after running for at most
\begin{align*}
C \cdot \frac{\inf_{\bLambda \in \bOmega_h} \max_{\bphi \in \Phi_{\ell,h}} \| \bphi \|_{\bA(\bLambda)^{-1}}^2}{\epsilon_\ell^2/\beta_\ell} + \poly \left ( d, H, \log \frac{\ell^2}{\delta}, \frac{1}{\lamminst}, \log | \Pi_\ell| \right )
\end{align*}
episodes.
\item The covariates returned by \Cref{line:explore} for any $(h,\ell)$, $\bLambda_{h,\ell}$, satisfy
\begin{align*}
& \max_{\bphi \in \Phi_{\ell,h}} \| \bphi \|_{\bLambda_{h,\ell}^{-1}}^2 \le \frac{\epsilon_\ell^2}{\beta_\ell}, \qquad \lammin(\bLambda_{h,\ell}) \ge \log \frac{4H^2|\Pi_\ell| \ell^2}{\delta}.
\end{align*}
\end{itemize}
Then $\Pr[(\cEexp^{\ell,h})^c \cap \cEest^{\ell,h-1} \cap (\cap_{i=1}^{h-1} \cEexp^{\ell,i} )] \le  \frac{\delta}{2H\ell^2}$. 
\end{lem}
\begin{proof}
By \Cref{lem:good_event_ell}, on the event $\cEest^{\ell,h-1} \cap (\cap_{i=1}^{h-1} \cEexp^{\ell,i} )$ we can bound $\| \bphihat_{\pi,h}^\ell - \bphi_{\pi,h} \|_2 \le d \epsilon_\ell / 2 H$. By \Cref{lem:feature_norm_lb}, we can lower bound $\| \bphi_{\pi,h} \|_2 \ge 1/\sqrt{d}$. By the reverse triangle inequality,
\begin{align*}
\| \bphihat_{\pi,h}^\ell \|_2 \ge \| \bphi_{\pi,h} \|_2 - \| \bphihat_{\pi,h}^\ell - \bphi_{\pi,h} \|_2 \ge 1/\sqrt{d} - d \epsilon_\ell / 2 H.
\end{align*}
It follows that as long as $\epsilon_\ell \le H/d^{3/2}$, that we can lower bound $\| \bphihat_{\pi,h}^\ell \|_2 \ge 1/(2 \sqrt{d})$. Since we start $\ell$ at $\ell = \lceil \log_2 \frac{d^{3/2}}{H} \rceil$, we will have that $\epsilon_\ell = 2^{-\ell} \le H/d^{3/2}$.

The result then follows by applying \Cref{cor:gopt_exp_design} with our chosen parameters and $\gamphi \leftarrow 1/(2 \sqrt{d})$.

\end{proof}

\begin{lem}\label{lem:good_event_ell}
On the event $\cEest^{\ell,h} \cap (\cap_{i=1}^h \cEexp^{\ell,i})$, for all $\pi \in \Pi_\ell$:
\begin{align*}
& |\inner{\btheta_{h+1}}{\bphihat_{\pi,h+1}^\ell - \bphi_{\pi,h+1}}| \le \epsilon_\ell/2H, \\
& \| \bphihat_{\pi,h+1}^\ell - \bphi_{\pi,h+1} \|_2 \le d \epsilon_\ell/2H, \\
& | \inner{\bphihat_{\pi,h}^\ell}{\bthetahat_h - \btheta_h}| \le \epsilon_\ell/2H.
\end{align*}
\end{lem}
\begin{proof}
On $\cEexp^{\ell,i}$, we can lower bound
\begin{align*}
\lammin(\bLambda_{i,\ell}) \ge \log \frac{4H^2|\Pi_\ell| \ell^2}{\delta}
\end{align*}
which implies
\begin{align*}
3 \sqrt{\log \frac{4 H^2 | \Pi_\ell| \ell^2}{\delta}} + \frac{\log \frac{4 H^2 | \Pi_\ell| \ell^2}{\delta}}{\sqrt{\lammin(\bLambda_{i,\ell})}}  \le 4 \sqrt{\log \frac{4 H^2 | \Pi_\ell| \ell^2}{\delta}}. 
\end{align*}
Furthermore, on $\cEexp^{\ell,i}$, $\| \bphihat_{\pi,i}^\ell \|_{\bLambda_{i,\ell}^{-1}} \le \frac{\epsilon_\ell}{\sqrt{\beta_\ell}}$. Since $\beta_\ell = 64 H^4 \log \frac{4 H^2 | \Pi_\ell| \ell^2}{\delta}$, on $\cEest^{\ell,h}$, we can then upper bound 
\begin{align*}
|\inner{\btheta_{h+1}}{\bphihat_{\pi,h+1}^\ell - \bphi_{\pi,h+1}}| & \le \sum_{i=1}^{h} \left ( 3 \sqrt{\log \frac{4 H^2 | \Pi_\ell| \ell^2}{\delta}} + \frac{\log \frac{4 H^2 | \Pi_\ell| \ell^2}{\delta}}{\sqrt{\lammin(\bLambda_{i,\ell})}}  \right ) \cdot \| \bphihat_{\pi,i}^\ell \|_{\bLambda_{i,\ell}^{-1}} \\
& \le H 4 \sqrt{\log \frac{4 H^2 | \Pi_\ell| \ell^2}{\delta}} \frac{\epsilon_\ell}{\sqrt{\beta_\ell}} \\
& \le \epsilon_\ell/2H .
\end{align*}
The same calculation gives the bounds on $\| \bphihat_{\pi,h}^\ell - \bphi_{\pi,h} \|_2$ and $| \inner{\bphihat_{\pi,h}^\ell}{\bthetahat_h - \btheta_h}|$. 
\end{proof}

\begin{lem}\label{lem:good_event}
Define $\cEexp = \cap_{\ell} \cap_h \cEexp^{\ell,h}$ and $\cEest = \cap_{\ell} \cap_h \cEest^{\ell,h}$. Then $\Pr[\cEest \cap \cEexp] \ge 1-2\delta$ and on $\cEest \cap \cEexp$, for all $h,\ell$, and $\pi \in \Pi_\ell$, 
\begin{align*}
& |\inner{\btheta_{h+1}}{\bphihat_{\pi,h+1}^\ell - \bphi_{\pi,h+1}}| \le \epsilon_\ell/2H, \\
& \| \bphihat_{\pi,h+1}^\ell - \bphi_{\pi,h+1} \|_2 \le d \epsilon_\ell/2H, \\
& | \inner{\bphihat_{\pi,h}^\ell}{\bthetahat_h - \btheta_h}| \le \epsilon_\ell/2H.
\end{align*}
\end{lem}
\begin{proof}
Clearly,
\begin{align*}
\cEest^c \cup \cEexp^c & = \bigcup_{\ell=\ell_0}^{\lceil \log 4/\epsilon \rceil} \bigcup_{h=1}^H  ( (\cEest^{\ell,h})^c \cup (\cEexp^{\ell,h})^c) \\
& = \bigcup_{\ell=\ell_0}^{\lceil \log 4/\epsilon \rceil} \bigcup_{h=1}^H   (\cEest^{\ell,h})^c \backslash \Big ( (\cEest^{\ell,h-1})^c \cup (\cup_{i=1}^{h-1} (\cEexp^{\ell,i})^c) \Big ) \cup \bigcup_{\ell=\ell_0}^{\lceil \log 4/\epsilon \rceil} \bigcup_{h=1}^H   (\cEexp^{\ell,h})^c  \\
& = \bigcup_{\ell=\ell_0}^{\lceil \log 4/\epsilon \rceil} \bigcup_{h=1}^H   (\cEest^{\ell,h})^c \cap \Big ( \cEest^{\ell,h-1} \cap (\cup_{i=1}^{h-1} \cEexp^{\ell,i} ) \Big ) \cup \bigcup_{\ell=\ell_0}^{\lceil \log 4/\epsilon \rceil} \bigcup_{h=1}^H   (\cEexp^{\ell,h})^c .
\end{align*}
The first conclusion follows by \Cref{lem:est_good_event}, \Cref{lem:est_good_event}, and since we can bound
\begin{align*}
\sum_{\ell} \sum_{h=1}^H 2 \cdot \frac{\delta}{2 H \ell^2} \le \frac{\pi^2}{6} \delta \le 2 \delta.
\end{align*}
The second conclusion follows by \Cref{lem:good_event_ell}.
\end{proof}

\begin{lem}\label{lem:correctness}
On the event $\cEest \cap \cEexp$, for all $\ell > \ell_0$, every policy $\pi \in \Pi_\ell$ satisfies $\Vst_0(\Pi) - \Vpi_0 \le 4\epsilon_\ell $ and $\pitilst \in \Pi_\ell$, for $\pitilst = \argmax_{\pi \in \Pi} \Vpi_0$. 
\end{lem}
\begin{proof}
The value of a policy $\pi$ is given by
\begin{align*}
\sum_{h=1}^H \inner{\btheta_h}{\bphi_{\pi,h}}.
\end{align*}
By \Cref{lem:good_event}, for all $\pi \in \Pi_\ell$ we can bound
\begin{align*}
| \inner{\bthetahat_h}{\bphihat_{\pi,h}^\ell} - \inner{\btheta_h}{\bphi_{\pi,h}}| \le | \inner{\bthetahat_h - \btheta_h}{\bphihat_{\pi,h}^\ell}| + |\inner{\btheta_h}{\bphihat_{\pi,h}^\ell - \bphi_{\pi,h}}| \le \epsilon_\ell/2H + \epsilon_\ell/2H = \epsilon_\ell/H.
\end{align*}
Thus,
\begin{align*}
\left | \sum_{h=1}^H \inner{\bthetahat_h^\ell}{\bphihat_{\pi,h}^\ell} - \sum_{h=1}^H \inner{\btheta_h}{\bphi_{\pi,h}}  \right | & \le \epsilon_\ell .
\end{align*}
We will only include $\pi \in \Pi_{\ell+1}$ if $\pi \in \Pi_\ell$ and
\begin{align*}
\sum_{h=1}^H \inner{\bphihat_{\pi,h}^\ell}{\bthetahat_h^\ell} \ge \sup_{\pi' \in \Pi_\ell} \sum_{h=1}^H \inner{\bphihat_{\pi',h}^\ell}{\bthetahat_h^\ell} - 2\epsilon_\ell.
\end{align*}
Using the estimation error given above, this implies that for any $\pi \in \Pi_\ell$,
\begin{align*}
V_0^\pi =  \sum_{h=1}^H \inner{\btheta_h}{\bphi_{\pi,h}} \ge \sup_{\pi' \in \Pi_\ell} \sum_{h=1}^H \inner{\btheta_h}{\bphi_{\pi',h}} - 4 \epsilon_\ell =  \sup_{\pi' \in \Pi_\ell} V_0^{\pi'} - 4 \epsilon_\ell.
\end{align*}

Both claims then follow if we can show $\pitilst$ is always contained in the active set. Assume that $\pitilst \in \Pi_\ell$. Then
\begin{align*}
\sum_{h=1}^H \inner{\bphihat_{\pitilst,h}^\ell}{\bthetahat_h^\ell} \ge V_0^{\pitilst} - \epsilon_\ell, \quad \sup_{\pi' \in \Pi_\ell} \sum_{h=1}^H \inner{\bphihat_{\pi',h}^\ell}{\bthetahat_h^\ell} \le \sup_{\pi' \in \Pi_\ell} \sum_{h=1}^H \inner{\bphi_{\pi',h}}{\btheta_h} + \epsilon_\ell = V_0^{\pitilst} + \epsilon_\ell.
\end{align*}
Rearranging this gives
\begin{align*}
\sum_{h=1}^H \inner{\bphihat_{\pitilst,h}^\ell}{\bthetahat_h^\ell} \ge \sup_{\pi' \in \Pi_\ell} \sum_{h=1}^H \inner{\bphihat_{\pi',h}^\ell}{\bthetahat_h^\ell} - 2\epsilon_\ell
\end{align*}
so $\pitilst \in \Pi_{\ell+1}$.
\end{proof}

\begin{thm}\label{thm:complexity}
With probability at least $1-2\delta$, \Cref{alg:linear_exp_design} will terminate after collecting at most
\begin{align*}
& CH^4 \sum_{h=1}^H \sum_{\ell=\ell_0 + 1}^{\iotaalg}  \frac{\inf_{\bLambda \in \bOmega_h} \max_{\pi \in \Pi(4\epsilon_\ell)} \| \bphi_{\pi,h} \|_{\bLambda^{-1}}^2}{\epsilon_\ell^2} \cdot  \log \frac{H | \Pi(4\epsilon_\ell) | \log \frac{1}{\epsilon}}{\delta}  + \poly \left ( d, H, \frac{1}{\lamminst}, \log \frac{1}{\delta}, \log | \Pi|, \log \frac{1}{\epsilon} \right ) \\
& \qquad + C H^4 \sum_{h=1}^H \frac{\inf_{\bLambda \in \bOmega_h} \max_{\pi \in \Pi} \| \bphi_{\pi,h} \|_{\bLambda^{-1}}^2}{\epsilon_{\ell_0}^2} \cdot \log \frac{H | \Pi | \log(1/\epsilon)}{\delta} 
\end{align*}
episodes for $\iotaalg := \min \{ \lceil \log \frac{4}{\epsilon} \rceil, \log \frac{4}{\Delmin(\Pi)} \}$, and will output a policy $\pihat$ such that
\begin{align*}
V_0^{\pihat} \ge \max_{\pi \in \Pi} V_0^\pi - \epsilon,
\end{align*}
where here $\Pi(4\epsilon_\ell) =  \{ \pi \in \Pi \ : \ V_0^\pi \ge \max_{\pi \in \Pi} V_0^{\pi} - 4 \epsilon_\ell \}$. 
\end{thm}
\begin{proof}
By \Cref{lem:good_event} the event $\cEest \cap \cEexp$ occurs with probability at least $1-2\delta$. Henceforth we assume we are on this event.

Correctness follows by \Cref{lem:correctness}, since upon termination, $\Pi_\ell$ will only contain policies $\pi$ satisfying $V_0^{\pi} \ge \max_{\pi \in \Pi} V_0^\pi - \epsilon$ (and will contain at least 1 policy since $\pitilst \in \Pi_\ell$ for all $\ell$). Furthermore, by \Cref{lem:correctness}, if $4 \epsilon_\ell < \Delmin(\Pi)$, we must have that $\Pi_\ell = \{ \pitilst \}$, and will therefore terminate on \Cref{line:early_term} since $|\Pi_\ell| = 1$. Thus, we can bound the number of number of epochs by 
\begin{align*}
\iotaalg := \min \{ \lceil \log \frac{4}{\epsilon} \rceil, \log \frac{4}{\Delmin(\Pi)} \}.
\end{align*} 

By \Cref{lem:exp_good_event}, the total number of episodes collected is bounded by
\begin{align*}
& \sum_{h=1}^H \sum_{\ell=1}^{\iotaalg} C \cdot \frac{\inf_{\bLambda \in \bOmega_h} \max_{\bphi \in \Phi_{\ell,h}} \| \bphi \|_{\bA(\bLambda)^{-1}}^2}{\epsilon_\ell^2/\beta_\ell} + \poly \left ( d, H, \log \frac{1}{\delta}, \frac{1}{\lamminst}, \log | \Pi|, \log \frac{1}{\epsilon} \right ) \\
& \le \sum_{h=1}^H \sum_{\ell=1}^{\iotaalg} C \cdot \frac{\inf_{\bLambda \in \bOmega_h} \max_{\bphi \in \Phi_{\ell,h}} \| \bphi \|_{\bA(\bLambda)^{-1}}^2}{\epsilon_\ell^2} \cdot H^4 \log \frac{H | \Pi_\ell | \log(1/\epsilon)}{\delta}  + \poly \left ( d, H, \log \frac{1}{\delta}, \frac{1}{\lamminst}, \log | \Pi|, \log \frac{1}{\epsilon} \right ) .
\end{align*}

On $\cEest \cap \cEexp$, by \Cref{lem:good_event}, for each $\pi \in \Pi_\ell$, we have $\| \bphihat_{\pi,h}^\ell - \bphi_{\pi,h} \|_2 \le d \epsilon_\ell / 2H$. As $\Phi_{\ell,h} = \{ \bphihat_{\pi,h}^\ell : \pi \in \Pi_\ell \}$, it follows that we can upper bound
\begin{align*}
\inf_{\bLambda \in \bOmega_h} \max_{\bphi \in \Phi_{\ell,h}} \| \bphi \|_{\bA(\bLambda)^{-1}}^2 & = \inf_{\bLambda \in \bOmega_h} \max_{\pi \in \Pi_\ell} \| \bphihat_{\pi,h}^\ell \|_{\bA(\bLambda)^{-1}}^2 \\
& \le \inf_{\bLambda \in \bOmega_h} \max_{\pi \in \Pi_\ell} ( 2 \| \bphi_{\pi,h} \|_{\bA(\bLambda)^{-1}}^2 + 2 \| \bphihat_{\pi,h}^\ell - \bphi_{\pi,h} \|_{\bA(\bLambda)^{-1}}^2) \\
& \le \inf_{\bLambda \in \bOmega_h} \max_{\pi \in \Pi_\ell} ( 2 \| \bphi_{\pi,h} \|_{\bA(\bLambda)^{-1}}^2 + \frac{d^2 \epsilon_\ell^2}{2 H^2 \lammin(\bA(\bLambda))}) \\
& \le \inf_{\bLambda \in \bOmega_h} \max_{\pi \in \Pi_\ell} 4 \| \bphi_{\pi,h} \|_{\bA(\bLambda)^{-1}}^2 + \inf_{\pi} \frac{d^2 \epsilon_\ell^2}{H^2 \lammin(\bA(\bLambda))} \\
& \le \inf_{\bLambda \in \bOmega_h} \max_{\pi \in \Pi_\ell} 4 \| \bphi_{\pi,h} \|_{\bLambda^{-1}}^2 +  \frac{d^2 \epsilon_\ell^2}{H^2 \lamminst}
\end{align*}
so
\begin{align*}
 \frac{\inf_{\bLambda \in \bOmega_h} \max_{\bphi \in \Phi_{\ell,h}} \| \bphi \|_{\bA(\bLambda)^{-1}}^2}{\epsilon_\ell^2} \le \frac{ \inf_{\bLambda \in \bOmega_h} \max_{\pi \in \Pi_\ell} 4 \| \bphi_{\pi,h} \|_{\bLambda^{-1}}^2 }{\epsilon_\ell^2} + \frac{d^2 }{H^2 \lamminst}.
\end{align*}

Note also that, by \Cref{lem:correctness}, for $\ell > \ell_0$, every policy $\pi \in \Pi_\ell$ will be $4\epsilon_\ell$ optimal, so we therefore have
\begin{align*}
\Pi_\ell \subseteq \{ \pi \in \Pi \ : \ V_0^\pi \ge V_0^{\pitilst} - 4 \epsilon_\ell \} =: \Pi(4\epsilon_\ell).
\end{align*}

Putting this together, we can upper bound the complexity by 
\begin{align*}
& \sum_{h=1}^H \sum_{\ell=\ell_0+1}^{\iotaalg} C \cdot \frac{\inf_{\bLambda \in \bOmega_h} \max_{\pi \in \Pi(4\epsilon_\ell)} \| \bphi_{\pi,h} \|_{\bLambda^{-1}}^2}{\epsilon_\ell^2} \cdot H^4 \log \frac{H | \Pi(4\epsilon_\ell) | \log(1/\epsilon)}{\delta} + \poly \left ( d, H, \log \frac{1}{\delta}, \frac{1}{\lamminst}, \log | \Pi|, \log \frac{1}{\epsilon} \right ) \\
& \qquad + C \cdot \frac{\inf_{\bLambda \in \bOmega_h} \max_{\pi \in \Pi} \| \bphi_{\pi,h} \|_{\bLambda^{-1}}^2}{\epsilon_{\ell_0}^2} \cdot H^4 \log \frac{H | \Pi | \log(1/\epsilon)}{\delta} .
\end{align*}
\end{proof}

\begin{cor}[Full Statement of \Cref{thm:complexity_linear}]\label{cor:complexity1}
With probability at least $1-\delta$, the complexity of \Cref{alg:linear_exp_design} can be bounded as
\begin{align*}
& CH^4 \log \frac{1}{\epsilon} \cdot \sum_{h=1}^H  \inf_{\bLambda \in \bOmega_h} \max_{\pi \in \Pi} \frac{ \| \bphi_{\pi,h} \|_{\bLambda^{-1}}^2}{(\Vst_0(\Pi) - V_0^\pi)^2 \vee \epsilon^2 \vee \Delmin(\Pi)^2} \cdot  \log \frac{H | \Pi | \log \frac{1}{\epsilon}}{\delta} \\ 
& \qquad + \poly \left ( d, H, \frac{1}{\lamminst}, \log \frac{1}{\delta}, \log | \Pi|, \log \frac{1}{\epsilon} \right )
\end{align*}
episodes, and \Cref{alg:linear_exp_design} will output a policy $\pihat$ such that
\begin{align*}
V_0^{\pihat} \ge \max_{\pi \in \Pi} V_0^\pi - \epsilon.
\end{align*}
\end{cor}
\begin{proof}
By the definition of $\Pi(4\epsilon_\ell)$, for each $\pi \in \Pi(4\epsilon_\ell)$ we have 
\begin{align*}
\epsilon_\ell^2 = \frac{1}{16} \left ( (\Vst_0(\Pi) - V_0^\pi)^2 \vee (4\epsilon_\ell)^2 \right ).
\end{align*}
We can therefore upper bound
\begin{align*}
& \sum_{\ell=\ell_0 + 1}^{\iotaalg}  \frac{\inf_{\bLambda \in \bOmega_h} \max_{\pi \in \Pi(4\epsilon_\ell)} \| \bphi_{\pi,h} \|_{\bLambda^{-1}}^2}{\epsilon_\ell^2} \cdot  \log \frac{H | \Pi(4\epsilon_\ell) | \log \frac{1}{\epsilon}}{\delta} \\
& \le  C \sum_{\ell=\ell_0+1}^{\iotaalg}  \inf_{\bLambda \in \bOmega_h} \max_{\pi \in \Pi(4\epsilon_\ell)} \frac{ \| \bphi_{\pi,h} \|_{\bLambda^{-1}}^2}{(\Vst_0(\Pi) - V_0^\pi)^2 \vee \epsilon_\ell^2} \cdot  \log \frac{H | \Pi(4\epsilon_\ell) | \log \frac{1}{\epsilon}}{\delta} \\
& \le  C \log \frac{1}{\epsilon} \cdot \inf_{\bLambda \in \bOmega_h} \max_{\pi \in \Pi} \frac{ \| \bphi_{\pi,h} \|_{\bLambda^{-1}}^2}{(\Vst_0(\Pi) - V_0^\pi)^2 \vee \epsilon^2 \vee \Delmin(\Pi)^2} \cdot  \log \frac{H | \Pi | \log \frac{1}{\epsilon}}{\delta}.
\end{align*}
Furthermore, since $\ell_0 = \lceil \log_2 d^{3/2}/H \rceil$, using \Cref{lem:opt_design_bound} we can also bound
\begin{align*}
C \cdot \frac{\inf_{\bLambda \in \bOmega_h} \max_{\pi \in \Pi} \| \bphi_{\pi,h} \|_{\bLambda^{-1}}^2}{\epsilon_{\ell_0}^2} \cdot H^4 \log \frac{H | \Pi | \log(1/\epsilon)}{\delta}  \le \poly \left ( d, H, \log 1/\delta, \log |\Pi|, \log 1/\epsilon \right ). 
\end{align*}
The result then follows by \Cref{thm:complexity}. 
\end{proof}

\begin{proof}[Proof of \Cref{thm:complexity_linear2}]
By \Cref{lem:policy_class_suff}, we can choose $\Pi_\epsilon$ to be the restricted-action linear softmax policy set constructed in \Cref{lem:policy_class_suff}. \Cref{lem:policy_class_suff} shows that $\Pi_\epsilon$ will contain an $\epsilon$-optimal policy for any MDP and reward function, and that
\begin{align*}
| \Pi_\epsilon |  \le \Big (1 + \frac{480 H^4 d^{5/2}  \log(1+120 Hd/\epsilon)}{\epsilon^2} \Big )^{dH}.
\end{align*}
Combining this with the guarantee of \Cref{cor:complexity1} shows that $V_0^{\pihat} \ge \Vst_0 - 2 \epsilon$ and that $\Vst_0(\Pi) - V_0^\pi$ is within a factor of $\epsilon$ of $\Vst_0 - V_0^\pi$. To bound the complexity of this procedure, we apply the bound given in \Cref{cor:complexity1} with the bound on the cardinality of $\Pi_\epsilon$ given above. 
\end{proof}

\subsection{Interpreting the Complexity}\label{sec:interpret_complexity}

\begin{lem}\label{lem:opt_design_bound}
For any set of policies $\Pi$, we can bound 
\begin{align*}
\inf_{\bLambda \in \bOmega_h} \sup_{\pi \in \Pi}  \| \bphi_{\pi,h} \|_{\bLambda^{-1}}^2 \le d .
\end{align*}
\end{lem}
\begin{proof}
By Jensen's inequality, for any $\bv \in \R^d$, we have
\begin{align*}
\bv^\top \bLambda_{\pi,h} \bv = \Exp_{\pi}[(\bv^\top \bphi_h)^2] \ge (\Exp_{\pi}[\bv^\top \bphi_h])^2 = (\bv^\top \bphi_{\pi,h})^2.
\end{align*}
It follows that, for any $\pi$, 
\begin{align*}
\bLambda_{\pi,h} \succeq \bphi_{\pi,h} \bphi_{\pi,h}^\top. 
\end{align*}

Take $\bLambda \in \bOmega_h$. Then,
\begin{align*}
\bLambda = \Exp_{\pi \sim \omega}[ \bLambda_{\pi,h}] \succeq \Exp_{\pi \sim \omega}[ \bphi_{\pi,h} \bphi_{\pi,h}^\top].
\end{align*}
It follows that we can upper bound
\begin{align*}
\inf_{\bLambda \in \bOmega_h} \sup_{\pi \in \Pi}  \| \bphi_{\pi,h} \|_{\bLambda^{-1}}^2 \le \inf_{\lambda \in \simplex_\Pi} \sup_{\pi \in \Pi}  \| \bphi_{\pi,h} \|_{A(\lambda)^{-1}}^2
\end{align*}
where $A(\lambda) = \sum_\pi \lambda_\pi \bphi_{\pi,h} \bphi_{\pi,h}^\top$. By Kiefer-Wolfowitz \citep{lattimore2020bandit}, this is upper bounded by $d$. 
\end{proof}

\begin{proof}[Proof of \Cref{prop:minimax}]
This follows directly from \Cref{lem:opt_design_bound} and \Cref{thm:complexity_linear2}, by upper bounding:
\begin{align*}
 \inf_{\bLambda \in \bOmega_h} \max_{\pi \in \Pieps} \frac{ \| \bphi_{\pi,h} \|_{\bLambda^{-1}}^2}{\max \{ \Vst_0 - \Vpi_0, \epsilon \}^2} \le  \inf_{\bLambda \in \bOmega_h} \max_{\pi \in \Pieps} \frac{ \| \bphi_{\pi,h} \|_{\bLambda^{-1}}^2}{\epsilon^2} \le \frac{d}{\epsilon^2}.
\end{align*}
\end{proof}

\subsubsection{Linear Contextual Bandits}

Since we always assume the MDP starts in some state $s_1$, to encode a linear contextual bandit, the direct mapping of our linear MDP in \Cref{defn:linear_mdp} would require considering an $H=2$ MDP, where we encode the ``context'' in the transition to state $s$ at step $h = 2$. While we could run our algorithm directly on this, in the standard contextual bandit setting, the learner has no control over the context, and so their action before receiving that context has no effect. Thus, there is no need for the learner to explore at stage $h=1$. To account for this, we can simply run our algorithm but ignore the exploration at stage $h=1$, which will reduce the $h=1$ term in the sample complexity.

\subsubsection{Tabular MDPs}

\begin{lem}\label{lem:tab_complexity1}
In the tabular MDP setting, assuming that $\Pi$ contains an optimal policy,
\begin{align*}
\inf_{\bLambda \in \bOmega_h} \max_{\pi \in \Pi} & \frac{ \| \bphi_{\pi,h} \|_{\bLambda^{-1}}^2}{(\Vst_0 - \Vpi_0)^2 \vee \epsilon^2 \vee \Delmin(\Pi)^2} \\
& \le \inf_{\piexp} \max_{\pi \in \Pi}  \max_{s,a} \frac{1}{w_h^{\piexp}(s,a)} \min \left \{  \frac{1}{ w_h^\pi(s,a) \Delta_h(s,a)^2}, \frac{w_h^\pi(s,a)}{\epsilon^2 \vee \Delmin(\Pi)^2} \right \} \\
& \le  \inf_{\piexp}  \max_{s,a} \frac{1}{w_h^{\piexp}(s,a)} \cdot \frac{1}{ \epsilon \max \{ \Delta_h(s,a), \epsilon, \Delmin(\Pi) \}} .
\end{align*}
\end{lem}
\begin{proof}
We have that $[\bphi_{\pi,h}]_{s,a} = \wpi_h(s,a)$. Furthermore, $\bphi(s,a) = \be_{s,a}$, so for any $\bLambda \in \bOmega_h$,  $\bLambda$ is diagonal with $[\bLambda]_{sa,sa} = \Exp_{\pi \sim \omega}[w_h^{\pi}(s,a)]$. Furthermore, by the Performance-Difference Lemma, $\Vst_0 - \Vpi_0 = \sum_{s,a,h} \wpi_h(s,a) \Delta_h(s,a)$. Thus, 
\begin{align}\label{eq:tab_complex1}
\inf_{\bLambda \in \bOmega_h} \max_{\pi \in \Pi} \frac{ \| \bphi_{\pi,h} \|_{\bLambda^{-1}}^2}{(\Vst_0 - \Vpi_0)^2 \vee \epsilon^2 \vee \Delmin(\Pi)^2} & \le \inf_{\piexp} \max_{\pi \in \Pi}  \frac{ \sum_{s,a} \frac{\wpi_h(s,a)^2}{w_h^{\piexp}(s,a)}}{(\sum_{s',a',h'} \wpi_{h'}(s',a') \Delta_{h'}(s',a'))^2 \vee \epsilon^2 \vee \Delmin(\Pi)^2} .
\end{align}
We have
\begin{align*}
\sum_{s,a} \frac{\wpi_h(s,a)^2}{w_h^{\piexp}(s,a)} \le \left ( \sum_{s,a} \wpi_h(s,a) \right ) \cdot \max_{s,a} \frac{\wpi_h(s,a)}{w_h^{\piexp}(s,a)} = \max_{s,a} \frac{\wpi_h(s,a)}{w_h^{\piexp}(s,a)}.
\end{align*}
Thus,
\begin{align}
\eqref{eq:tab_complex1} & \le \inf_{\piexp} \max_{\pi \in \Pi}  \max_{s,a} \frac{ \wpi_h(s,a)/w_h^{\piexp}(s,a) }{ (\sum_{s',a',h'} \wpi_{h'}(s',a') \Delta_{h'}(s',a'))^2 \vee \epsilon^2 \vee \Delmin(\Pi)^2} \nonumber \\
& \le \inf_{\piexp} \max_{\pi \in \Pi}  \max_{s,a} \frac{ \wpi_h(s,a)/w_h^{\piexp}(s,a) }{ (\wpi_h(s,a) \Delta_h(s,a))^2 \vee \epsilon^2 \vee \Delmin(\Pi)^2} \nonumber \\
& = \inf_{\piexp} \max_{\pi \in \Pi}  \max_{s,a} \frac{1}{w_h^{\piexp}(s,a)} \min \left \{  \frac{1}{ w_h^\pi(s,a) \Delta_h(s,a)^2}, \frac{w_h^\pi(s,a)}{\epsilon^2 \vee \Delmin(\Pi)^2} \right \}. \label{eq:tab_complex2}
\end{align}
We can further upper bound
\begin{align*}
 \min \left \{  \frac{1}{ w_h^\pi(s,a) \Delta_h(s,a)^2}, \frac{w_h^\pi(s,a)}{\epsilon^2 \vee \Delmin(\Pi)^2} \right \} \le \frac{1}{\Delta_h(s,a) (\epsilon \vee \Delmin(\Pi))}
\end{align*}
so
\begin{align*}
\eqref{eq:tab_complex2} & \le \inf_{\piexp} \max_{\pi \in \Pi}  \max_{s,a} \frac{1}{w_h^{\piexp}(s,a)} \min \left \{  \frac{1}{\Delta_h(s,a) \epsilon}, \frac{w_h^\pi(s,a)}{\epsilon^2 \vee \Delmin(\Pi)^2} \right \} \\
& \le \inf_{\piexp}  \max_{s,a} \frac{1}{w_h^{\piexp}(s,a)} \frac{1}{ \epsilon \max \{ \Delta_h(s,a), \epsilon, \Delmin(\Pi) \}}.
\end{align*}

\end{proof}

\begin{lem}\label{lem:tab_complexity2}
If \algname is run with a set $\Pi$ that contains an optimal policy, the complexity of \algname is upper bounded as
\begin{align*}
\cOtil \left ( H^4 \sum_{h=1}^H \sup_{\epsilon' \ge \max \{ \epsilon, \Delmin(\Pi)/4 \}} \inf_{\piexp} \max_{\pi \in \Pi(\epsilon')} \max_{s,a} \frac{1}{w_h^{\piexp}(s,a)} \min \left \{ \frac{1}{\wpi_h(s,a) \Delta_h(s,a)^2}, \frac{\wpi_h(s,a)}{(\epsilon')^2} \right \} \cdot \log \frac{| \Pi(\epsilon') |}{\delta} + C_0 \right )
\end{align*}
for $\Pi(\epsilon') = \{ \pi \in \Pi \ : \ \Vpi_0 \ge \Vst_0(\Pi) - \epsilon \}$.
\end{lem}
\begin{proof}
Using an argument identical to that in \Cref{lem:tab_complexity1}, we can upper bound
\begin{align*}
\frac{\inf_{\bLambda \in \bOmega_h} \max_{\pi \in \Pi(4\epsilon_\ell)} \| \bphi_{\pi,h} \|_{\bLambda^{-1}}^2}{\epsilon_\ell^2} & \le \inf_{\bLambda \in \bOmega_h} \max_{\pi \in \Pi(4\epsilon_\ell)} \frac{c \| \bphi_{\pi,h} \|_{\bLambda^{-1}}^2}{(\Vst_0 - \Vpi_0)^2 \vee \epsilon_\ell^2} \\
& \le \inf_{\piexp} \max_{\pi \in \Pi(4\epsilon_\ell)} \max_{s,a} \frac{1}{w_h^{\piexp}(s,a)} \min \left \{ \frac{1}{\wpi_h(s,a) \Delta_h(s,a)^2}, \frac{\wpi_h(s,a)}{\epsilon_\ell^2} \right \}.
\end{align*}
The result then follows from \Cref{thm:complexity}, noting that we will never run for $\epsilon_\ell < \Delmin(\Pi)/4$.
\end{proof}

\begin{proof}[Proof of \Cref{cor:tabular}]
Note that in the tabular MDP setting, we can choose $\Pi$ to be the set of all deterministic policies, since this set is guaranteed to contain an optimal policy. We can then bound $|\Pi| \le A^{SH}$. The result then follows directly from \Cref{lem:tab_complexity1} and \Cref{thm:complexity_linear}.
\end{proof}

\begin{proof}[Proof of \Cref{prop:gap_vis_vs_pedel}]
We begin with an example where \algname has complexity smaller than the Gap-Visitation Complexity, and then turn to an example where the reverse is true.

\paragraph{\algname Improves on Gap-Visitation Complexity.}
Consider the tabular MDP with $|\cS| = |\cA| = N$, and where
\begin{align*}
& P_h(s_1 | s_1, a_1) = 1, \quad \nu_h(s_1,a_1) = 1, \forall h \in [H] \\
& P_h(s_1 | s_1, a_j) = 0, \quad \nu_h(s_1,a_j) = 0, \forall h \in [H], j \neq 1 \\
& P_h(s_1 | s_i, a_j) = 0, \forall h \in [H], j \in [N], i \neq 1 \\
& P_h(s_i | s_j, a_i) = 1, \forall h \in [H], j \in [N], i \neq 1 \\
& \nu_h(s_i,a_1) = \epsilon, \forall h \in [H], i \neq 1, \quad \nu_h(s_i,a_j) = 0, \forall h \in [H], j \neq 1, i \neq 1 .
\end{align*}
In this MDP, the optimal policy simply plays action $a_1$ $H$ times and is always in state $s_1$. The total reward it collects is $H$. 

Note that, since the MDP is deterministic, we can choose $\Pi$ as in \Cref{cor:deterministic_mdps}, to be simply the set of policies which play a deterministic sequences of actions. It follows that any policy in $\Pi$ that does not play $a_1$ $H$ consecutive times has optimality gap of at least $1-\epsilon$. In this case, then, we have $\Delmin(\Pi) = 1-\epsilon$. 

By \Cref{thm:complexity_linear}, we can therefore upper bound the complexity of the leading-order term by $\cOtil(\poly(S,A,H,\log 1/\delta))$, so \algname will identify the optimal policy (since $\Pi$ contains an optimal policy). Thus, the total complexity of \algname is $\cO(\poly(S,A,H,\log 1/\delta))$. 

On this example, in every state $s_i$, $i \neq 1$, action $a_1$ still collects a reward of $\epsilon$. Thus, we have that $\Delta_h(s_i,a_j) = \epsilon$ for $j \neq 1$. The Gap-Visitation complexity is given by
\begin{align*}
\sum_{h=1}^H \inf_\pi \max_{s,a} \min \left \{ \frac{1}{\wpi_h(s,a) \Delta_h(s,a)^2}, \frac{W_h(s)^2}{\epsilon^2} \right \}.
\end{align*}
Since $W_h(s) = 1$ for each $s$, we conclude that
\begin{align*}
\sum_{h=1}^H \inf_\pi \max_{s,a} \min \left \{ \frac{1}{\wpi_h(s,a) \Delta_h(s,a)^2}, \frac{W_h(s)^2}{\epsilon^2} \right \} \ge \sum_{h=1}^H \frac{1}{\epsilon^2}.
\end{align*}
Thus, for small $\epsilon$, the Gap-Visitation complexity can be arbitrarily worse than the complexity of \algname.

\paragraph{The Gap-Visitation Complexity Improves on \algname.}
To show that the Gap-Visitation Complexity improves on the complexity of \algname, we consider the example in Instance Class 5.1 of \cite{wagenmaker2021beyond}. As shown by Proposition 6 of \cite{wagenmaker2021beyond}, on this example, for any $\epsilon$, the Gap-Visitation Complexity is $\cOtil(\poly(S))$.

To bound the complexity of \algname on this example, we consider the complexity given in \Cref{thm:complexity} with $\Pi$ the set of all deterministic policies, which is slightly tighter than the complexity of \Cref{cor:tabular}. Take $\epsilon \ge 2^{-S}$. Then, on this example, it follows that $\Delmin(\Pi) \le \cO(\epsilon)$, since we can find a policy $\pi$ which is optimal on every state $s_i$ at step $h = 2$ for $i = \cO(\log 1/\epsilon)$, which will give it a policy gap of $\cO(\epsilon)$. Furthermore, any near-optimal policy will have $[\bphi_{\pi,2}]_{s_1,a_1} = \wpi_2(s_1,a_1) = \cO(1)$, so we always have $\inf_{\bLambda \in \bOmega_2} \max_{\pi \in \Pi(4\epsilon_\ell)} \| \bphi_{\pi,h} \|_{\bLambda^{-1}}^2 \ge \Omega(1)$. It follows that the complexity of \algname is lower bounded by $\Omega(1/\epsilon^2)$.

\end{proof}

\subsubsection{Deterministic, Tabular MDPs}

\begin{lem}\label{lem:deterministic_complexity}
In the deterministic MDP setting,
\begin{align*}
\inf_{\bLambda \in \bOmega_h} \max_{\pi \in \Pi} \frac{ \| \bphi_{\pi,h} \|_{\bLambda^{-1}}^2}{(\Vst_0 - \Vpi_0)^2 \vee \epsilon^2} \le \sum_{s,a} \frac{1}{\Delbar_h(s,a)^2 \vee \epsilon^2}.
\end{align*}
\end{lem}
\begin{proof}
Note that $[\bphi_{\pi,h}]_{s_h^\pi,a_h^\pi} = 1$, and otherwise, for $(s,a) \neq (s_h^\pi,a_h^\pi)$, $[\bphi_{\pi,h}]_{s,a} = 0$. Furthermore, $\bLambda_{\piexp,h}$ will always be diagonal, with diagonal elements $w_h^{\pi}(s,a)$. 
We then have $\| \bphi_{\pi,h} \|_{\bLambda_{\piexp,h}^{-1}}^2 = \frac{1}{w_h^{\piexp}(s_h^\pi,a_h^\pi)}$, so 
\begin{align*}
\inf_{\bLambda \in \bOmega_h} \max_{\pi \in \Pi} \frac{ \| \bphi_{\pi,h} \|_{\bLambda^{-1}}^2}{(\Vst_0 - \Vpi_0)^2 \vee \epsilon^2} & \le \inf_{\piexp} \max_{\pi \in \Pi} \frac{ \| \bphi_{\pi,h} \|_{\bLambda_{\piexp,h}^{-1}}^2}{(\Vst_0 - \Vpi_0)^2 \vee \epsilon^2} \\
& = \inf_{\piexp} \max_{\pi \in \Pi} \frac{  w_h^{\piexp}(s_h^\pi,a_h^\pi)^{-1}}{(\Vst_0 - \Vpi_0)^2 \vee \epsilon^2} \\
& \overset{(a)}{=}  \inf_{\piexp} \max_{s,a} \max_{\pi \in \Pi_{sah}} \frac{  w_h^{\piexp}(s_h^\pi,a_h^\pi)^{-1}}{(\Vst_0 - \Vpi_0)^2 \vee \epsilon^2} \\
& \overset{(b)}{=}  \inf_{\piexp} \max_{s,a} \max_{\pi \in \Pi_{sah}} \frac{  w_h^{\piexp}(s,a)^{-1}}{(\Vst_0 - \Vpi_0)^2 \vee \epsilon^2} \\
& = \inf_{\piexp} \max_{s,a}  \frac{  w_h^{\piexp}(s,a)^{-1}}{(\Vst_0 - \max_{\pi \in \Pi_{sah}} \Vpi_0)^2 \vee \epsilon^2} \\
& \overset{(c)}{=} \inf_{\piexp} \max_{s,a}  \frac{  w_h^{\piexp}(s,a)^{-1}}{\Delbar_h(s,a)^2 \vee \epsilon^2}
\end{align*}
where $(a)$ follows since $\Pi = \cup_{s,a} \Pi_{sah}$, $(b)$ follows since by definition, for any $\pi \in \Pi_{sah}$, $(s_h^\pi,a_h^\pi) = (s,a)$, and $(c)$ follows by the definition of $\Delbar_h(s,a)$. 

Let $\pi^{sa}$ denote any policy such that $(s_h^\pi,a_h^\pi) = (s,a)$. Set
\begin{align*}
    \lambda_{\pi^{sa}} = \frac{\max \{ \bar{\Delta}_h(s,a), \epsilon \}^{-2}}{\sum_{s',a'} \max \{ \bar{\Delta}_h(s',a'), \epsilon \}^{-2} }.
\end{align*}
Note that this is a valid distribution. Let $\piexp = \sum_{s,a} \lambda_{\pi^{sa}} \pi^{sa}$, then $w_h^{\piexp}(s,a) = \lambda_{\pi^{sa}}$, so
\begin{align*}
 \inf_{\piexp} \max_{s,a}  \frac{  w_h^{\piexp}(s,a)^{-1}}{\Delbar_h(s,a)^2 \vee \epsilon^2} & \le   \max_{s,a}  \frac{   \lambda_{\pi^{sa}}^{-1}}{\Delbar_h(s,a)^2 \vee \epsilon^2} \\
 & \le \sum_{s,a} \frac{1}{\Delbar_h(s,a)^2 \vee \epsilon^2}
\end{align*}
which proves the result. 
\end{proof}

\begin{proof}[Proof of \Cref{cor:deterministic_mdps}]
As in tabular MDPs, we can set $\Pi$ to correspond to the set of all deterministic policies. However, since our MDP is also deterministic, at any given $h$, we only need to specify $\pi_h(s)$ for a single $s$---the state we will end up in at step $h$ with probability 1. Thus, we can take $\Pi$ to be a set of cardinality $|\Pi| = A^H$.
The result then follows directly from \Cref{lem:deterministic_complexity} and \Cref{thm:complexity_linear}.
\end{proof}

\paragraph{Comparison to Lower Bound of \cite{tirinzoni2022near}.}
The precise definition for $\Delbarmin^h$ is 
$\Delbarmin^h := \min_{s,a : \Delbar_h(s,a) > 0} \Delbar_h(s,a)$ in the setting when every deterministic $\epsilon$-optimal policy will reach the same $(s,a)$ at step $h$, and $\Delbarmin^h := 0$ otherwise.

The exact lower bound given in \cite{tirinzoni2022near} scales as $\varphi^\star(\underline{c})$ which does not have an explicit form. However, they show that
\begin{align*}
\max_{h \in [H]} \sum_{s \in \cS} \sum_{a \in \cA} \frac{\log(1/4\delta)}{4 \max \{ \Delbar_h(s,a), \Delbarmin^h, \epsilon \}^2} \le \varphi^\star(\underline{c}) \le \sum_{h \in [H]} \sum_{s \in \cS} \sum_{a \in \cA} \frac{\log(1/4\delta)}{4 \max \{ \Delbar_h(s,a), \Delbarmin^h, \epsilon \}^2}.
\end{align*}
Up to $H$ factors, then, this matches the complexity of our upper bound in every term but the $\Delbarmin^h$ term. $\Delbarmin^h \ge \Delbarmin$, so this lower bound is potentially smaller than our upper bound in this dependence. We remark, however, that the algorithm presented in \cite{tirinzoni2022near} obtains the same scaling as we do, depending on $\Delbarmin$ instead of $\Delbarmin^h$. Furthermore, in general we can think of these quantities as scaling in a similar manner, since they each quantify the minimum policy gap.


\section{Experiment Design via Online Frank-Wolfe}\label{sec:fw}

\iftoggle{arxiv}{}{
\subsection{Experiment Design in MDPs with General Objective Functions}
While the experiment design in \eqref{eq:exp_design2} is the natural design if our goal is to identify a near-optimal policy, in general we may be interested in collecting data to minimize some other objective; that is, solving an experiment design of the form:
\begin{align*}
\inf_{\bLamexp \in \bOmega_h} f(\bLamexp)
\end{align*}
for some function $f$ defined over the space of PSD matrices. For example, we could take $f(\bLamexp) = \| \bLamexp^{-1} \|_\op = \frac{1}{\lammin(\bLamexp)}$, and the above experiment design would correspond to maximizing the minimum eigenvalue of the collected covariates, or $\mathsf{E}$-optimal design \citep{pukelsheim2006optimal}.

Motivated by this, in this section we generalize \Cref{cor:g_opt_exp_design} and \Cref{alg:online_fw_body} to handle a much broader class of experiment design problems. In particular, we consider all \emph{smooth experiment design objectives}, which we define as follows.

\begin{defn}[Smooth Experiment Design Objectives]\label{def:smooth_exp_des_fun}
We say that $f(\bLambda) : \psdmat^d \rightarrow \R$ is a \emph{smooth experiment design objective} if it satisfies the following conditions:
\begin{itemize}
\item $f$ is convex, differentiable, and $\beta$ smooth in the norm $\| \cdot \|$: $\| \nabla f(\bLambda) - \nabla f(\bLambda') \|_* \le \beta \| \bLambda - \bLambda' \|$. 
\item $f$ is $L$-lipschitz in the operator norm: $| f(\bLambda) - f(\bLambda') | \le L \| \bLambda - \bLambda' \|_\op$.
\item Let $\Xi_{\bLambda_0} := -\nabla_{\bLambda} f(\bLambda) |_{\bLambda = \bLambda_0}$. Then $\Xi_{\bLambda_0} \succeq 0$ and $\tr( \Xi_{\bLambda_0}) \le M$ for all $\bLambda_0 \succeq 0$ satisfying $\| \bLambda_0 \|_\op \le 1$.
\end{itemize}
\end{defn}

We will often be interested in objectives $f$ that satisfy $f(a \bLambda) = a^{-1} f(\bLambda)$ for a scalar $a$, in which case the guarantee $f(N^{-1} \bSighat_N) \le N \epsilon$ reduces to $f( \bSighat_N) \le \epsilon$. We note also that many typical experiment design objectives are non-smooth. As we show in \Cref{sec:xy_design}, however, it is often possible to derive smoothed versions of such objectives with negligible approximation error. }

Through the remainder of \Cref{sec:fw} as well as \Cref{sec:xy_design}, we will be interested in the problem of data collection in linear MDPs. In general, we will seek to collect data for a particular $h \in [H]$. We will therefore consider the following truncation to our MDP.

\begin{defn}[Truncated Horizon MDPs]\label{def:trun_horizon_mdp}
Given some MDP $\cM$ with horizon $H$, we define the $h$-truncated-horizon MDP $\cM_{\mathrm{tr},h}$ to be the MDP that is identical to $\cM$ for $h' \le h$, but that terminates after reaching state $s_h$ and playing action $a_h$. 
\end{defn}

We can simulated a truncated-horizon MDP by playing in our standard MDP $\cM$, and after taking an action at step $h$, $a_h$, taking random actions for $h' > h$ and ignoring all future observations. 

The utility of considering truncated-horizon MDPs is that we can therefore guarantee the data we collect, $\{ \{ (s_{h',\tau}, a_{h',\tau}) \}_{h'=1}^h \}_{\tau = 1}^K$ is uncorrelated with the true next state and reward at step $h$ obtained in $\cM$, $\{ (s_{h+1,\tau}, r_{h,\tau}) \}_{\tau = 1}^K$. While we do not allow our algorithm to use $\{ (s_{h+1,\tau}, r_{h,\tau}) \}_{\tau = 1}^K$ in its operation, it is allowed to store this data and return it. 

For the remainder of \Cref{sec:fw} and \Cref{sec:xy_design}, then, we assume there is some fixed $h$ we are interested in, and that we are running our algorithms in the $h$-truncated-horizon MDP defined with respect to our original MDP. We will also drop the subscript of $h$ from observations, so $\bLambda_\pi =\bLambda_{\pi,h}$, $\bphi_{\tau} = \bphi_{\tau,h}$, and $\bOmega = \bOmega_h$.

Our main experiment design algorithm, \optcov, relies on a regret-minimization algorithm satisfying the following guarantee.

\begin{defn}[Regret Minimization Algorithm]\label{def:regmin}
We say \regmin is a regret minimization algorithm if it has regret scaling as, with probability at least $1-\delta$, 
\begin{align*}
\cR_K := \sum_{k=1}^K ( \Vst_0 - V_0^{\pi_k}) \le \sqrt{\cC_1 K \log^{p_1}(HK/\delta)} + \cC_2 \log^{p_2}(HK/\delta)
\end{align*}
for any deterministic reward function $r_h(s,a) \in [0,1]$.
\end{defn}

Throughout this section, we will let $\bLambda$ refer to covariates normalized by time, and $\bSigma$ unnormalized covariates. So, for example, we might have $\bSigma = \sum_{\tau = 1}^T \bphi_\tau \bphi_\tau^\top$ and $\bLambda = \frac{1}{T} \sum_{\tau = 1}^T \bphi_\tau \bphi_\tau^\top$.

The rest of this section is organized as follows. First, in \Cref{sec:approx_fw} we show that a variant of the Frank-Wolfe algorithm that relies on only approximate updates enjoys a convergence rate similar to the standard Frank-Wolfe rate. Next, in \Cref{sec:onlinefw_regret} we show that for a smooth experiment design objective, we can approximately optimize the objective in a linear MDP by approximating the Frank-Wolfe updates via a regret minimization algorithm. Finally, in \Cref{sec:optcov} we present our main experiment-design algorithm, \optcov, which relies on our online Frank-Wolfe procedure to collect covariates that minimize an online experimental design objective up to an arbitrarily tolerance.

\subsection{Approximate Frank-Wolfe}\label{sec:approx_fw}

We will consider the following approximate variant of the Frank-Wolfe algorithm:

\begin{algorithm}[h]
\begin{algorithmic}[1]
\State \textbf{input}: function to optimize $f$, number of iterations to run $T$, starting iterate $\bx_1$
\For{$t = 1,2,\ldots,T$}
	\State Set $\gamma_t \leftarrow \frac{1}{t+1}$
	\State Choose $\by_t$ to be any point such that
	\begin{align*}
	\nabla f(\bx_t)^\top \by_t \le \min_{\by \in \cX} \nabla f(\bx_t)^\top \by + \epsilon_t
	\end{align*}
	\State $\bx_{t+1} \leftarrow (1- \gamma_t) \bx_t + \gamma_t \by_t$
\EndFor
\State \textbf{return} $\bx_{T+1}$
\end{algorithmic}
\caption{Approximate Frank-Wolfe}
\label{alg:approx_fw}
\end{algorithm}

\begin{lem}\label{lem:approx_fw}
Consider running \Cref{alg:approx_fw} with some convex function $f$ that is $\beta$-smooth with respect to some norm $\| \cdot \|$. Assume that all iterates of \Cref{alg:approx_fw} live in some set $\cY$ and let $R := \sup_{\bx,\by \in \cX \cup \cY} \| \bx - \by \|$. Then for $T \ge 2$, we have
\begin{align*}
f(\bx_{T+1}) - \min_{\bx \in \cX} f(\bx) \le \frac{\beta R^2 (\log T + 1)}{2(T+1)} + \frac{1}{T+1} \sum_{t=1}^{T}  \epsilon_t.
\end{align*}
\end{lem}
\begin{proof}
Let $\xst = \argmin_{\bx \in \cX} f(\bx)$. Using that $f$ is $\beta$-smooth, the definition of $\by_s$, and the convexity of $f$, we have that for any $s$, 
\begin{align*}
f(\bx_{s+1}) - f(\bx_s) & \le \nabla f(\bx_s)^\top (\bx_{s+1} - \bx_s) + \frac{\beta}{2} \| \bx_{s+1} - \bx_s \|^2 \\
& \le \gamma_s \nabla f(\bx_s)^\top (\by_s - \bx_s) + \frac{\beta}{2} \gamma_s^2 R^2 \\
& \le \gamma_s \nabla f(\bx_s)^\top (\xst - \bx_s) + \gamma_s \epsilon_s + \frac{\beta}{2} \gamma_s^2 R^2 \\
& \le \gamma_s (f(\xst) - f(\bx_s)) + \gamma_s \epsilon_s + \frac{\beta}{2} \gamma_s^2 R^2 .
\end{align*}
Letting $\delta_s= f(\bx_s) - f(\xst)$, this implies that
\begin{align*}
\delta_{s+1} & \le (1-\gamma_s) \delta_s + \gamma_s \epsilon_s + \frac{\beta}{2} \gamma_s^2 R^2 .
\end{align*}
Unrolling this backwards gives
\begin{align*}
\delta_{T+1} & \le (1-\gamma_{T}) \delta_{T} + \gamma_{T} \epsilon_{T} + \frac{\beta}{2} \gamma_{T}^2 R^2 \\
& \le (1-\gamma_{T})(1-\gamma_{T-1}) \delta_{T-1} + (1-\gamma_{T}) (\gamma_{T-1} \epsilon_{T-1} + \frac{\beta}{2} \gamma_{T-1}^2 R^2) + \gamma_{T} \epsilon_{T} + \frac{\beta}{2} \gamma_{T}^2 R^2 \\
& \le \sum_{t=1}^{T} \left ( \prod_{s = t+1}^{T} (1 - \gamma_{s}) \right ) (\gamma_t \epsilon_t + \frac{\beta}{2} \gamma_t^2 R^2) .
\end{align*}
We can write
\begin{align*}
\prod_{s = t+1}^{T} (1 - \gamma_{s}) = \prod_{s=t+1}^{T} \frac{s}{s+1} = \frac{t+1}{T+1}
\end{align*}
so
\begin{align*}
\sum_{t=1}^{T} \left ( \prod_{s = t+1}^{T} (1 - \gamma_{s}) \right )   \frac{\beta}{2} \gamma_t^2 R^2 & =   \sum_{t=1}^{T} \frac{t+1}{T+1}  \frac{\beta}{2} \frac{1}{(t+1)^2} R^2 \\
& = \frac{ \beta R^2}{2(T+1)} \sum_{t=1}^{T} \frac{1}{t+1} \\
& \le \frac{\beta R^2 (\log T + 1)}{2(T+1)}
\end{align*}
and
\begin{align*}
\sum_{t=1}^{T} \left ( \prod_{s = t+1}^{T} (1 - \gamma_{s}) \right ) \gamma_t \epsilon_t & = \sum_{t=1}^{T}  \frac{t+1}{T+1}  \frac{1}{t+1} \epsilon_t \\
& = \frac{1}{T+1} \sum_{t=1}^{T}  \epsilon_t
\end{align*}
which proves the result.

\end{proof}

\begin{lem}\label{lem:fw_final_iterate}
When running \Cref{alg:approx_fw}, we have
\begin{align*}
\bx_{T+1} = \frac{1}{T+1} \left ( \sum_{t=1}^T \by_t + \bx_1 \right ).
\end{align*}
\end{lem}
\begin{proof}
We have:
\begin{align*}
\bx_{T+1} & = \sum_{t=1}^T \left ( \prod_{s=t+1}^{T} ( 1 - \gamma_s ) \right ) \gamma_t \by_t + \left ( \prod_{s=1}^{T} ( 1 - \gamma_s ) \right ) \bx_1 \\
& = \sum_{t=1}^T \frac{t+1}{T+1} \frac{1}{t+1} \by_t + \frac{1}{T+1} \bx_1 \\
& = \frac{1}{T+1} \sum_{t=1}^T \by_t + \frac{1}{T+1} \bx_1.
\end{align*}
\end{proof}

\subsection{Online Frank-Wolfe via Regret Minimization}\label{sec:onlinefw_regret}

\begin{algorithm}[h]
\begin{algorithmic}[1]
\State \textbf{input}: function to optimize $f$, number of iterates $T$, episodes per iterate $K$
\State Play any policy for $K$ episodes, denote collected covariates as $\bGamma_0$, collected data as $\frakD_0$
\State $\bLambda_1 \leftarrow K^{-1} \bGamma_0$
\For{$t = 1,2,\ldots,T$}
	\State Set $\gamma_t \leftarrow \frac{1}{t+1}$
	\State Run \regmin on reward $r_h^t(s,a) = \tr(\Xi_{\bLambda_t} \cdot \bphi(s,a) \bphi(s,a)^\top)/M$ for $K$ episodes, denote collected covariates as $\bGamma_t$, collected data as $\frakD_t$
	\State $\bLambda_{t+1} \leftarrow (1- \gamma_t) \bLambda_t + \gamma_t  K^{-1} \bGamma_t$
\EndFor
\State \textbf{return} $\bLambda_{T+1}$, $\cup_{t=0}^T \frakD_t$
\end{algorithmic}
\caption{Online Frank-Wolfe via Regret Minimization (\fwregret)}
\label{alg:regret_fw}
\end{algorithm}

\newcommand{\bOmegahat}{\widehat{\bOmega}}
\begin{lem}\label{lem:regret_fw}
Consider running \Cref{alg:regret_fw} with a function $f$ satisfying \Cref{def:smooth_exp_des_fun} and a regret minimization algorithm satisfying \Cref{def:regmin}. Denote $K_0(T,\beta,M,\delta)$ the minimum integer value of $K$ satisfying
\begin{align*}
K \ge \max \left \{ \frac{72 T^2 M^2 \log(4T/\delta)}{\beta^2 R^4}, \frac{8 T^2 M^2 \cC_1 \log^{p_1}(2HKT/\delta)}{\beta^2 R^4}, \frac{3 TM \cC_2 \log^{p_2}(2HKT/\delta)}{\beta R^2} \right \} .
\end{align*}
Then as long as $K \ge K_0(T,\beta,M,\delta)$, we have that, with probability at least $1-\delta$,
\begin{align*}
f(\bLambda_{T+1}) - \inf_{\bLambda \in \bOmega} f(\bLambda) \le \frac{ \beta R^2 (\log T + 3)}{2(T+1)}  
\end{align*}
for $R = \sup_{\bLambda,\bLambda' \in \bOmegahat } \| \bLambda - \bLambda \|$ and $\bOmegahat = \{ \Exp_{(s,a) \sim \omega}[\bphi(s,a) \bphi(s,a)^\top] \ : \ \omega \in \simplex_{\cS \times \cA} \}$. 
\end{lem}
\begin{proof}
Note that by \Cref{lem:fw_final_iterate} and since $\| \bphi(s,a) \|_2 \le 1$, we can bound $\| \bLambda_t \|_\op \le 1$ and $\| \bphi(s,a) \bphi(s,a)^\top \|_\op \le 1$. \Cref{def:smooth_exp_des_fun} it follows that $r_h^t(s,a) \in [0,1]$ for all $s,a$, since $\tr(\Xi_{\bLambda_t} \cdot \bphi(s,a) \bphi(s,a)^\top) \le \| \bphi(s,a) \bphi(s,a)^\top \|_\op \cdot \tr(\Xi_{\bLambda_t}) \le \tr(\Xi_{\bLambda_t}) \le M$, and $\tr(\Xi_{\bLambda_t} \cdot \bphi(s,a) \bphi(s,a)^\top) \ge 0$ since $\Xi_{\bLambda_t} \succeq 0$. 
If we run \regmin for $K$ episodes on reward function $r_h^t$, by \Cref{def:smooth_exp_des_fun} and \Cref{def:regmin} we then have that, with probability at least $1-\delta/2T$,
\begin{align*}
\sqrt{\cC_1 K \log^{p_1}(2H K T/\delta)} + \cC_2 \log^{p_2}(2H K T/\delta) & \ge K \sup_\pi \Exp_\pi[\tr(\Xi_{\bLambda_t} \cdot \bphi \bphi^\top)/M] - \sum_{k=1}^{K} \Exp_{\pi_k}[\tr(\Xi_{\bLambda_t} \cdot \bphi \bphi^\top)/M] \\
& = K \sup_\pi \tr(\Xi_{\bLambda_t} \bLambda_\pi)/M - K \tr(\Xi_{\bLambda_t} \cdot  K^{-1} \sum_{k=1}^{K} \bLambda_{\pi_k})/M 
\end{align*}
which implies
\begin{align*}
\sqrt{\frac{M^2 \cC_1 \log^{p_1}(2HKT/\delta)}{K}} + \frac{M  \cC_2 \log^{p_2}(2HKT/\delta)}{K} \ge  \sup_\pi \tr(\Xi_{\bLambda_t} \bLambda_\pi) -  \tr(\Xi_{\bLambda_t} \cdot  K^{-1} \sum_{k=1}^{K} \bLambda_{\pi_k}) .
\end{align*}
Furthermore, we have that 
\begin{align*}
\left |  \tr(\Xi_{\bLambda_t} \cdot  K^{-1} \sum_{k=1}^{K} \bLambda_{\pi_k}) - \tr(\Xi_{\bLambda_t} \cdot  K^{-1} \bGamma_t) \right | & = \left |  \frac{1}{K} \sum_{k=1}^K \tr(\Xi_{\bLambda_t} \bLambda_{\pi_k}) - \frac{1}{K} \sum_{k=1}^K \tr(\Xi_{\bLambda_t} \bphi_k \bphi_k^\top) \right | .
\end{align*}
Note that $\Exp_{\pi_k} [\tr(\Xi_{\bLambda_t} \bphi_k \bphi_k^\top) ] = \tr(\Xi_{\bLambda_t} \bLambda_{\pi_k}) $, $\tr(\Xi_{\bLambda_t} \bphi_k \bphi_k^\top)  \in [0,M]$, and $\pi_k$ is $\cF_{k-1}$-measurable. We can therefore apply Azuma-Hoeffding (\Cref{lem:ah}) to get that, with probability at least $1-\delta/2T$, 
\begin{align*}
\left |  \tr(\Xi_{\bLambda_t} \cdot  K^{-1} \sum_{k=1}^{K} \bLambda_{\pi_k}) - \tr(\Xi_{\bLambda_t} \cdot  K^{-1} \bGamma_t) \right |  & \le \sqrt{\frac{8 M^2 \log (4T/\delta) }{K}} .
\end{align*}
Therefore,
\begin{align*}
& \sqrt{\frac{8 M^2 \log (4T/\delta) }{K}}  + \sqrt{\frac{M^2 \cC_1 \log^{p_1}(2HKT/\delta)}{K}} + \frac{M  \cC_2 \log^{p_2}(2HKT/\delta)}{K}  \\
& \qquad \ge \sup_\pi \tr(\Xi_{\bLambda_t} \bLambda_\pi) - \tr(\Xi_{\bLambda_t} \cdot K^{-1}\bGamma_t)  .
\end{align*}
Given our condition on $K$, we have
\begin{align*}
\sqrt{\frac{8 M^2 \log (4T/\delta) }{K}}  + \sqrt{\frac{M^2 \cC_1 \log^{p_1}(2HKT/\delta)}{K}} + \frac{M  \cC_2 \log^{p_2}(2HKT/\delta)}{K} \le \frac{\beta R^2}{T}
\end{align*}
which implies
\begin{align}\label{eq:regret_fw_eps}
\sup_\pi \tr(\Xi_{\bLambda_t} \bLambda_\pi) - \tr(\Xi_{\bLambda_t} \cdot K^{-1}\bGamma_t) \le \frac{\beta R^2}{T}  .
\end{align}
Note that, for any $\bLambda \in \bOmega$, we have 
\begin{align*}
\tr(\Xi_{\bLambda_t} \bLambda) = \tr(\Xi_{\bLambda_t} \Exp_{\pi \sim \omega}[\bLambda_\pi]) = \Exp_{\pi \sim \omega} [\tr(\Xi_{\bLambda_t} \bLambda_\pi)]
\end{align*}
so
\begin{align*}
\sup_{\bLambda \in \bOmega} \tr(\Xi_{\bLambda_t} \bLambda) = \sup_{\omega \in \bOmega_\pi} \Exp_{\pi \sim \omega}[\tr(\Xi_{\bLambda_t} \bLambda_\pi) ] = \sup_\pi \tr(\Xi_{\bLambda_t} \bLambda_\pi)
\end{align*}
By definition, $\Xi_{\bLambda_t} = -\nabla_{\bLambda} f(\bLambda) |_{\bLambda = \bLambda_t}$, so it follows that
\begin{align*}
-\sup_{\bLambda' \in \bOmega} \tr(\Xi_{\bLambda_t} \bLambda') = \inf_{\bLambda' \in \bOmega} \tr(\nabla_{\bLambda} f(\bLambda) |_{\bLambda = \bLambda_t} \cdot \bLambda').
\end{align*}

It follows that \eqref{eq:regret_fw_eps} is precisely the guarantee required on $\by_t$ by \Cref{alg:approx_fw} with $\epsilon_t = \frac{\beta R^2}{T} $. Since $f$ is $\beta$-smooth by \Cref{def:smooth_exp_des_fun} and since the set $\bOmega$ is convex and compact by \Cref{lem:policy_cov_convex}, we can apply \Cref{lem:approx_fw} with a union bound over $t$ to get the result. 
\end{proof}

\subsection{Data Collection via Online Frank-Wolfe}\label{sec:optcov}
\awarxiv{define $\frakD$ more precisely}

\begin{algorithm}[h]
\begin{algorithmic}[1]
\State \textbf{input}: functions to optimize $( f_i )_i$, constraint tolerance $\epsilon$, confidence $\delta$
\For{$i = 1,2,3,\ldots$}
	\State $T_i \leftarrow 2^i$, $K_i \leftarrow 2^i T_i^2$
	\If{$K_i \ge \Ktil_0(T_i,\beta_i,M_i,\frac{\delta}{4i^2}) T_i^2 + \Ktil_1(T_i,\beta_i,M_i,\frac{\delta}{4i^2}) T_i $ for $\Ktil_0$ and $\Ktil_1$ as in \Cref{lem:K0_bound}} \label{line:K0_if}
	\State $\bLamhat, \frakD_i \leftarrow$ \textsc{FWRegret}($f_i,T_i-1,K_i$)
	\If{$f_i(\bLamhat) \le K_i T_i \epsilon$ and $f_i(\bLamhat) \ge \frac{ \beta_i R^2 (\log T_i + 3)}{T_i} $}\label{line:data_fw_if_eps}
		\State \textbf{return} $\bLamhat$, $K_i T_i$, $\frakD_i$
	\EndIf
	\EndIf
\EndFor
\end{algorithmic}
\caption{Collect Optimal Covariates (\optcov)}
\label{alg:regret_data_fw}
\end{algorithm}

\begin{thm}\label{thm:regret_data_fw}
Let $(f_i)_i$ denote some sequence of functions which satisfy \Cref{def:smooth_exp_des_fun} with constants $(\beta_i,L_i,M_i)$ and assume $\beta_i \ge 1$. Let $(\beta,L,M)$ be some values such that $\beta_i \le \beta, L_i \le L, M_i \le M$ for al $i$, and let $f$ be some function such that $f_i(\bLambda) \le f(\bLambda)$ for all $i$ and $\bLambda \succeq 0$. Denote $\fmin$ a lower bound on all $f_i$: $\min_i \inf_{\bLambda \in \bOmega} f_i(\bLambda) \ge \fmin$. 

Define 
\begin{align}\label{eq:Nst_opt}
\Nst(\epsilon; f) := \frac{\inf_{\bLambda \in \bOmega} f(\bLambda)}{\epsilon}.
\end{align}
Then, if we run \Cref{alg:regret_data_fw} on $(f_i)_i$ with constraint tolerance $\epsilon$ and confidence $\delta$, we have that with probability at least $1-\delta$, it will run for at most
\begin{align*}
5 \Nst(\epsilon;f) + \poly \left (2^{p_1 + p_2}, \cC_1, \cC_2, M, \beta, R, L, \fmin^{-1}, \log 1/\delta \right )
\end{align*}
episodes, and will return data $\{ \bphi_\tau \}_{\tau =1}^N$ with covariance $\bSighat_N = \sum_{\tau=1}^N \bphi_\tau \bphi_\tau^\top$ such that 
\begin{align*}
f_{\ihat}(N^{-1} \bSighat_N) \le N \epsilon,
\end{align*}
where $\ihat$ is the iteration on which \optcov terminates.
\end{thm}

\iftoggle{arxiv}{\begin{cor}[\Cref{cor:Nst_opt_force_body}]\label{cor:Nst_opt_force}
Instantiating \regmin with the computationally efficient version of the \force algorithm of \cite{wagenmaker2021first}, we obtain a complexity of 
\begin{align*}
5 \Nst(\epsilon;f) + \poly \left ( d,H, M, \beta, R, L, \fmin^{-1}, \log 1/\delta \right ).
\end{align*}
\end{cor}}{\begin{cor}\label{cor:Nst_opt_force}
Instantiating \regmin with the computationally efficient version of the \force algorithm of \cite{wagenmaker2021first}, we obtain a complexity of 
\begin{align*}
5 \Nst(\epsilon;f) + \poly \left ( d,H, M, \beta, R, L, \fmin^{-1}, \log 1/\delta \right ).
\end{align*}
\end{cor}}
\begin{proof}
This result is immediate since \force satisfies \Cref{def:regmin} with
\begin{align*}
\cC_1 = c_1 d^4 H^4, \quad \cC_2 = c_2 d^4 H^3, \quad p_1 = 3, \quad p_2 = 7/2
\end{align*}
for universal numerical constants $c_1$ and $c_2$. 
\end{proof}

\begin{proof}[Proof of \Cref{thm:regret_data_fw}]
We first show that the condition $f_i(\bLamhat) \ge  \frac{ \beta R^2 (\log T_i + 3)}{T_i} $ is sufficient to ensure a $2$-approximate minimum of $f_i$, and then show a sufficient condition on $K_i$ and $T_i$ that will guarantee the condition on \Cref{line:data_fw_if_eps} is met.

\paragraph{Guaranteeing $2$-optimality.}
We first show that for a fixed $i$, the condition $f_i(\bLamhat) \ge  \frac{ \beta_i R^2 (\log T_i + 3)}{T_i} $ will only be met once
\begin{align*}
f_i(\bLamhat) \le 2 \cdot \inf_{\bLambda \in \bOmega} f_i(\bLambda) 
\end{align*}
and that it will take at most
\begin{align*}
T _i\ge \frac{2 \beta R^2 (\log T_i + 3)}{\inf_{\bLambda \in \bOmega} f_i(\bLambda) }
\end{align*}
iterations to do so, as long as
\begin{align*}
T_i K_i \ge \frac{L^2}{2(d \log (1 + 8 \sqrt{T_i K_i}) +  \log (4i^2/\delta)) \cdot (\inf_{\bLambda \in \bOmega} f_i(\bLambda))^2 } .
\end{align*}

The first part follows by applying \Cref{lem:regret_fw}. Note that the if statement on \Cref{line:K0_if} will only be met once
\begin{align*}
K_i \ge K_0(T_i,\beta_i,M_i,\delta/4i^2).
\end{align*}
This follows by \Cref{lem:K0_bound}. Thus, the condition on $K_i$ required by \Cref{lem:regret_fw} will be met, so it follows that with probability at least $1-\delta/(4i^2)$,
\begin{align*}
f_i(\bLamhat) - \inf_{\bLambda \in \bOmega} f_i(\bLambda) \le \frac{ \beta_i R^2 (\log T_i + 3)}{2T_i} .
\end{align*}
Therefore, if $f_i(\bLamhat) \ge  \frac{ \beta_i R^2 (\log T_i + 3)}{T_i}$, we have
\begin{align*}
f_i(\bLamhat) - \inf_{\bLambda \in \bOmega} f_i(\bLambda) \le \frac{1}{2} f_i(\bLamhat) & \implies \frac{1}{2}  f_i(\bLamhat) \le \inf_{\bLambda \in \bOmega} f_i(\bLambda) \\
& \implies f_i(\bLamhat) \le 2 \cdot \inf_{\bLambda \in \bOmega} f_i(\bLambda). 
\end{align*}

We will show a sufficient condition for $f_i(\bLamhat) \ge \frac{ \beta R^2 (\log T_i + 3)}{T_i}$, which implies that $f_i(\bLamhat) \ge \frac{ \beta_i R^2 (\log T_i + 3)}{T_i}$ since $\beta_i \le \beta$. By \Cref{lem:fw_final_iterate} and the procedure run by \Cref{alg:regret_fw}, we have that $\bLamhat = \frac{1}{T_i K_i} \sum_{\tau=1}^{T_i K_i} \bphi_\tau \bphi_\tau^\top$ where at episodes $\tau$ we run some $\cF_{\tau-1}$-measurable policy $\pi_\tau$ to acquire $\bphi_\tau$. 
Now if $\bLamhat = \bLamtil$ for some $\bLamtil \in \bOmega$, then the second part follows trivially since $\inf_{\bLambda \in \bOmega} f_i(\bLambda) \le f_i(\bLamtil)$, so a sufficient condition for $f_i(\bLamhat) \ge \frac{ \beta R^2 (\log T_i + 3)}{T_i}$ is that $\inf_{\bLambda \in \bOmega} f_i(\bLambda) \ge  \frac{ \beta R^2 (\log T_i + 3)}{T_i}$. However, since $\bLamhat$ is stochastic, we may not have that $\bLamhat \in \bOmega$. Let $\bLamtil := \frac{1}{T_i K_i} \sum_{\tau = 1}^{T_i K_i} \bLambda_{\pi_\tau}$ and note that $\bLamtil \in \bOmega$. Applying \Cref{lem:cov_concentration}, we have that with probability at least $1-\delta/(4i^2)$,
\begin{align*}
\left \|  \bLamtil - \bLamhat \right \|_\op \le \sqrt{\frac{8 d \log (1 + 8 \sqrt{T_i K_i}) + 8 \log (4i^2/\delta)}{T_i K_i}}
\end{align*}
for $\pitil$ the uniform mixture of $\{ \pi_\tau \}_{\tau=1}^{T_iK_i}$. By the Lipschitz condition of \Cref{def:smooth_exp_des_fun}, this implies
\begin{align*}
f_i(\bLamhat) & \ge f_i(\bLamtil) - L_i \| \bLamhat - \bLamtil \|_\op \\
& \ge f_i(\bLamtil) - L \| \bLamhat - \bLamtil \|_\op \\
& \ge f_i(\bLamtil) - L \sqrt{\frac{8 d \log (1 + 8 \sqrt{T_i K_i}) + 8 \log (4i^2/\delta)}{T_i K_i}} \\
&  \ge \inf_{\bLambda \in \bOmega} f_i(\bLambda)  - L \sqrt{\frac{8 d \log (1 + 8 \sqrt{T_i K_i}) + 8 \log (4i^2/\delta)}{T_i K_i}}.
\end{align*}
Thus, a sufficient condition for $f_i(\bLamhat) \ge  \frac{ \beta R^2 (\log T_i + 3)}{T_i}$ is that
\begin{align*}
& \inf_{\bLambda \in \bOmega} f_i(\bLambda)  - L \sqrt{\frac{8 d \log (1 + 8 \sqrt{T_i K_i}) + 8 \log (4i^2/\delta)}{T_i K_i}} \ge  \frac{ \beta R^2 (\log T_i + 3)}{T_i} \\
& \iff T_i \ge \frac{ \beta R^2 (\log T_i + 3)}{\inf_{\bLambda \in \bOmega} f_i(\bLambda)  - L \sqrt{\frac{8 d \log (1 + 8 \sqrt{T_i K_i}) + 8 \log (4i^2/\delta)}{T_i K_i}}} .
\end{align*}
If 
\begin{align*}
T_i K_i \ge \frac{L^2}{2(d \log (1 + 8 \sqrt{T_i K_i}) +  \log (4i^2/\delta)) \cdot (\inf_{\bLambda \in \bOmega} f_i(\bLambda))^2 }
\end{align*}
it follows that a sufficient condition is 
\begin{align*}
T_i \ge \frac{2 \beta R^2 (\log T_i + 3)}{\inf_{\bLambda \in \bOmega} f_i(\bLambda) } .
\end{align*}

Union bounding over the events considered above for all $i$, we have that the total probability of failure is bounded as
\begin{align*}
\sum_{i=1}^\infty ( \frac{\delta}{4 i^2} + \frac{\delta}{4 i^2}) = \frac{\pi^2}{12} \delta \le \delta. 
\end{align*}

\paragraph{Termination Guarantee.}
We next show a sufficient condition to ensure that the if statements on \Cref{line:K0_if} and \Cref{line:data_fw_if_eps} are met.

Assume the if statement on \Cref{line:K0_if} has been met and that we are in the regime where 
\begin{align}\label{eq:data_fw_kt_suff}
T_i K_i \ge \frac{L^2}{2(d \log (1 + 8 \sqrt{T_i K_i}) +  \log (4i^2/\delta)) \cdot \fmin^2 }, \quad T_i \ge \frac{2 \beta R^2 (\log T_i + 3)}{\fmin } .
\end{align}
By the argument above and since $\inf_{\bLambda \in \bOmega} f_i(\bLambda) \ge \fmin$, these conditions are sufficient to guarantee a $2$-optimal solutions has been found, that is,
\begin{align*}
f_i(\bLamhat) \le 2 \cdot \inf_{\bLambda \in \bOmega} f_i(\bLambda),
\end{align*}
and that the condition $f_i(\bLamhat) \ge  \frac{ \beta R^2 (\log T_i + 3)}{T_i}$ has been met. 
Thus, if \eqref{eq:data_fw_kt_suff} holds, a sufficient condition for $f_i(\bLamhat) \le T_iK_i\epsilon$ is
\begin{align*}
2 \cdot \inf_{\bLambda \in \bOmega} f_i( \bLambda) \le T_i K_i \epsilon .
\end{align*}
It follows that this condition will be met (assuming \eqref{eq:data_fw_kt_suff} holds) once $T_i K_i \ge \Nst(\frac{\epsilon}{2}; f_i)$. Since $f_i \le f$, $\Nst(\frac{\epsilon}{2}; f_i) \le \Nst(\frac{\epsilon}{2}; f)$, so a sufficient condition is that $T_i K_i \ge \Nst(\frac{\epsilon}{2}; f)$.

To upper bound the total complexity, it suffices then to guarantee that we run for enough epochs so that
\begin{align}
& K_i = 2^{3i} \ge \Ktil_0(T_i,\beta_i,M_i,\frac{\delta}{4i^2}) T_i^2 + \Ktil_1(T_i,\beta_i,M_i,\frac{\delta}{4i^2}) T_i \label{eq:fw_i_constraint0} \\
& T_i K_i = 2^{4i} \ge  \frac{L^2}{2(d \log (1 + 8 \sqrt{T_i K_i}) +  \log (4i^2/\delta)) \cdot \fmin^2 } \label{eq:fw_i_constraint1} \\
& T_i = 2^i \ge \frac{2 \beta R^2 (\log T_i + 3)}{\fmin } \label{eq:fw_i_constraint2} \\
& T_iK_i = 2^{4i} \ge \Nst(\frac{\epsilon}{2}; f) \label{eq:fw_i_constraint3}.
\end{align}
Here \eqref{eq:fw_i_constraint0} guarantees the if statement on \Cref{line:K0_if} is met, and \eqref{eq:fw_i_constraint1}-\eqref{eq:fw_i_constraint3} guarantee the if statement on line \Cref{line:data_fw_if_eps} is met.

By assumption, $M_i \le M$ and $\beta_i \ge 1$, and note that $\Ktil_0(T_i,\beta_i,M_i,\frac{\delta}{4i^2})$ and $\Ktil_1(T_i,\beta_i,M_i,\frac{\delta}{4i^2})$ are both increasing in $M_i$ and decreasing in $\beta_i$. Thus, a sufficient condition to ensure \eqref{eq:fw_i_constraint0} is met is
\begin{align}\label{eq:fw_i_constraint0_1}
2^{3i} \ge  \Ktil_0(2^i,1,M,\frac{\delta}{4i^2}) 2^{2i} + \Ktil_1(2^i,1,M,\frac{\delta}{4i^2})  2^i.
\end{align}
Some calculation shows that
\begin{align*}
\Ktil_0(2^i,1,M,\frac{\delta}{4i^2}) \le (5i)^{p_1} \Ktil_0(2,1,M,\frac{\delta}{4}), \quad \Ktil_1(2^i,1,M,\frac{\delta}{4i^2}) \le (4i)^{p_2} \Ktil_1(2,1,M,\frac{\delta}{4})
\end{align*}
so a sufficient condition to meet \eqref{eq:fw_i_constraint0_1} is
\begin{align*}
2^i \ge 2 (5i)^{p_1} \Ktil_0(2,1,M,\frac{\delta}{4}), \quad 2^{2i} \ge 2 (4i)^{p_2} \Ktil_1(2,1,M,\frac{\delta}{4}).
\end{align*}
By \Cref{claim:log_lin_burnin} and some calculation, this will be met once
\begin{align*}
i \ge \max \left \{ 4 p_1 \log_2(2 p_1) + 2\log_2 ( 2 (5)^{p_1} \Ktil_0(2,1,M,\frac{\delta}{4}) ), 2 p_2 \log_2(p_2) + 2 \log_2 (2 (4)^{p_2} \Ktil_1(2,1,M,\frac{\delta}{4})) \right \} =: i_0.
\end{align*}

To meet \eqref{eq:fw_i_constraint1} it suffices to take
\begin{align*}
i \ge \frac{1}{4} \log_2  \frac{L^2}{d \fmin^2} =: i_1
\end{align*}
By \Cref{claim:log_lin_burnin}, a sufficient condition to meet \eqref{eq:fw_i_constraint2} is that 
\begin{align*}
T_i \ge \max \left \{ \frac{6 \beta R^2}{\fmin}, \frac{4 \beta R^2}{\fmin} \log \frac{4 \beta R^2}{\fmin} \right \}
\end{align*}
so it suffices that
\begin{align*}
i \ge \log_2 \left ( \frac{6 \beta R^2}{\fmin} \log \frac{4 \beta R^2}{\fmin} \right ) =: i_2.
\end{align*}
Finally, to meet \eqref{eq:fw_i_constraint3}, it suffices that
\begin{align*}
i \ge \frac{1}{4} \log_2 \Nst(\epsilon/2;f) =: i_3.
\end{align*}

If we terminate at epoch $\ihat$, the total sample complexity will be bounded by
\begin{align*}
\sum_{i = 1}^{\ihat} T_i K_i = \sum_{i = 1}^{\ihat} 2^{4i} \le \frac{16}{15} \cdot 2^{4 \ihat}.
\end{align*}
By the above argument, we can bound $\ihat \le \lceil \max \{ i_0,i_1,i_2,i_3 \} \rceil$. Furthermore, we see that
\begin{align*}
2^{4 \lceil i_0 \rceil} & = \poly \left ( 2^{p_1}, 2^{p_2}, M, \cC_1, \cC_2, \log 1/\delta \right ) \\
2^{4 \lceil i_1 \rceil} & = \poly  ( L, \fmin^{-1}) \\
2^{4 \lceil i_2 \rceil} & = \poly ( \beta, R, \fmin^{-1} ) \\
2^{4 \lceil i_3 \rceil} & \le 2 \Nst(\epsilon/2;f)
\end{align*}
so we can bound the total sample complexity by
\begin{align*}
\frac{16}{15} \cdot 2^{4  \lceil \max \{ i_0,i_1,i_2,i_3 \} \rceil} \le \frac{32}{15}  \Nst(\epsilon/2;f) + \poly \left ( 2^{p_1}, 2^{p_2}, \beta, R, L, \fmin^{-1}, M, \cC_1, \cC_2, \log 1/\delta \right ).
\end{align*}
This completes the proof since $\Nst(\frac{\epsilon}{2};f) = 2 \Nst(\epsilon;f)$ and since, by \Cref{lem:fw_final_iterate}, $\bLamhat$ is simply the average of the observed feature vectors: $\bLamhat = \frac{1}{T_i K_i} \sum_{\tau=1}^{T_i K_i} \bphi_\tau \bphi_\tau^\top$.

\end{proof}

\begin{lem}\label{lem:cov_concentration}
Let $\bLambda_K$ denote the time-normalized covariates obtained by playing policies $\{\pi_k\}_{k=1}^K$, where $\pi_k$ is $\cF_{k-1}$-measurable. Then, with probability at least $1-\delta$,
\begin{align*}
\left \| \frac{1}{K} \sum_{k=1}^K \bLambda_{\pi_k} -  \bLambda_K \right \|_\op \le \sqrt{\frac{8 d \log (1 + 8 \sqrt{K}) + 8 \log 1/\delta}{K}}.
\end{align*}
\end{lem}
\begin{proof}
Let $\cV$ denote an $\epsilon$-net of $\cS^{d-1}$, for some $\epsilon$ to be chosen. Then,
\begin{align*}
& \left \| \frac{1}{K} \sum_{k=1}^K \bLambda_{\pi_k} -  \bLambda_K \right \|_\op  = \sup_{\bv \in \cS^{d-1}} \left | \bv^\top \left ( \frac{1}{K} \sum_{k=1}^K \bLambda_{\pi_k} -  \bLambda_K \right ) \bv \right | \\
& \qquad \le \underbrace{ \sup_{\bvtil \in \cV} \left | \bvtil^\top \left ( \frac{1}{K} \sum_{k=1}^K \bLambda_{\pi_k} - \bLambda_K \right ) \bvtil \right |}_{(a)} \\
& \qquad \qquad + \underbrace{\sup_{\bv \in \cS^{d-1}} \inf_{\bvtil \in \cV} \left | \bv^\top \left ( \frac{1}{K} \sum_{k=1}^K \bLambda_{\pi_k} -  \bLambda_K \right ) \bv - \bvtil^\top \left ( \frac{1}{K} \sum_{k=1}^K \bLambda_{\pi_k} -  \bLambda_K \right ) \bvtil \right |}_{(b)}.
\end{align*}
Via a union bound over $\cV$ and application of Azuma-Hoeffding, we can bound, with probability at least $1-\delta$,
\begin{align*}
(a) \le \sqrt{\frac{2\log | \cV |/\delta}{K}}.
\end{align*}
We can bound $(b)$ as
\begin{align*}
(b) & \le \sup_{\bv \in \cS^{d-1}} \inf_{\bvtil \in \cV} 2 \left | \bv^\top \left ( \frac{1}{K} \sum_{k=1}^K \bLambda_{\pi_k} -  \bLambda_K \right ) (\bv - \bvtil) \right | \\
& \le  \sup_{\bv \in \cS^{d-1}} \inf_{\bvtil \in \cV} 2 \| \bv - \bvtil \|_2 \left \| \frac{1}{K} \sum_{k=1}^K \bLambda_{\pi_k} -  \bLambda_K \right \|_\op \\
& \le 4 \epsilon
\end{align*}
where the last inequality follows since $\| \frac{1}{K} \sum_{k=1}^K \bLambda_{\pi_k} \|_\op \le 1$, and $\| \frac{1}{K} \bLambda_K \|_\op \le 1$, and since $\cV$ is an $\epsilon$-net. Setting $\epsilon = 1/(4\sqrt{K})$, \Cref{lem:euc_ball_cover} gives that $|\cV| \le (1 + 8 \sqrt{K})^d$, and we conclude that with probability at least $1-\delta$:
\begin{align*}
 \left \| \frac{1}{K} \sum_{k=1}^K \bLambda_{\pi_k} - \bLambda_K \right \|_\op & \le \sqrt{\frac{2\log | \cV |/\delta}{K}} + 4 \epsilon \\
 & \le \sqrt{\frac{2 d \log (1 + 8 \sqrt{K}) + 2 \log 1/\delta}{K}} + \frac{1}{\sqrt{K}} \\
 & \le 2 \sqrt{\frac{2 d \log (1 + 8 \sqrt{K}) + 2 \log 1/\delta}{K}}.
\end{align*}
\end{proof}

\begin{lem}\label{lem:K0_bound}
We can bound
\begin{align*}
K_0(T,\beta,M,\delta) \le \Ktil_0(T,\beta,M,\delta) T^2 + \Ktil_1(T,\beta,M,\delta) T
\end{align*}
for
\begin{align*}
\Ktil_0(T,\beta,M,\delta) & :=  \max \bigg \{   \frac{72  M^2 \log(4T/\delta)}{\beta^2 R^4}, \frac{8  M^2 \cC_1}{\beta^2 R^4} \cdot  (2p_1)^{p_1} \log^{p_1} \left (  \frac{32 p_1 H T^3 M^2 \cC_1}{\beta^2 R^4\delta} \right ) \bigg \} \\
\Ktil_1(T,\beta,M,\delta) & := \frac{3M\cC_2}{\beta R^2} \cdot (2p_2)^{p_2} \log^{p_2} \left ( \frac{12 p_2 HT^2M\cC_2}{\beta R^2 \delta}  \right ) ,
\end{align*}
\end{lem}
\begin{proof}
By definition $K_0(T,\beta,M,\delta)$ is the smallest integer value of $K$ that satisfies:
\begin{align}\label{eq:K0_condition}
K \ge \max \left \{ \frac{72 T^2 M^2 \log(4T/\delta)}{\beta^2 R^4}, \frac{8 T^2 M^2 \cC_1 \log^{p_1}(2HKT/\delta)}{\beta^2 R^4}, \frac{3 TM \cC_2 \log^{p_2}(2HKT/\delta)}{\beta R^2} \right \} .
\end{align}
By \Cref{claim:log_lin_burnin}, we have that if
\begin{align*}
K \ge \frac{8 T^2 M^2 \cC_1}{\beta^2 R^4} \cdot  (2p_1)^{p_1} \log^{p_1} \left (  \frac{8 T^2 M^2 \cC_1}{\beta^2 R^4} \cdot \frac{4p_1 HT}{\delta} \right ) , \quad K \ge \frac{3TM\cC_2}{\beta R^2} \cdot (2p_2)^{p_2} \log^{p_2} \left ( \frac{3TM\cC_2}{\beta R^2} \cdot \frac{4p_2 HT}{\delta} \right )
\end{align*}
and
\begin{align*}
K \ge \frac{72 T^2 M^2 \log(4T/\delta)}{\beta^2 R^4}
\end{align*}
then \Cref{eq:K0_condition} will be satisfied. Some algebra gives the result. 
\end{proof}


\section{$\mathsf{XY}$-Optimal Design}\label{sec:xy_design}

We are interested in optimizing the function
\begin{align*}
\Gopt(\bLambda) = \max_{\bphi \in \Phi} \| \bphi \|_{\bA(\bLambda)^{-1}}^2 \quad \text{for} \quad  \bA(\bLambda) = \bLambda + \bLambda_0
\end{align*}
with $\bLambda_0 \succ 0$ some fixed regularizer. This objective, however, is not smooth, so we relax it to the following:
\begin{align}\label{eq:gopt_smooth}
\Gopts(\bLambda) := \LSE \left ( \{ e^{\eta \| \bphi \|_{\bA(\bLambda)^{-1}}^2} \}_{\bphi \in \Phi}; \eta \right ) = \frac{1}{\eta} \log \left ( \sum_{\bphi \in \Phi} e^{\eta \| \bphi \|_{\bA(\bLambda)^{-1}}^2} \right ).
\end{align}

We first offer some properties on how well $\Gopts(\bLambda)$ approximates $\Gopt(\bLambda)$, and then show that we can bound the smoothness constant of $\Gopts(\bLambda)$. Throughout this section, we will denote $\gamphi := \max_{\bphi \in \Phi} \| \bphi \|_2$ and let $f(\bLambda) := \Gopts(\bLambda) $. 

\subsection{Approximating Non-Smooth Optimal Design with Smooth Optimal Design}
\begin{lem}\label{lem:xy_approx_error}
\begin{align*}
& | \Gopt(\bLambda) -  \Gopts(\bLambda) | \le \frac{\log | \Phi |}{\eta}, \qquad \Gopt(\bLambda) \le  \Gopts(\bLambda) .
\end{align*}
\end{lem}
\begin{proof}
This result is standard but we include the proof for completeness. We prove it for some generic sequence $(a_i)_{i=1}^n$. Take $\eta > 0$. Clearly,
\begin{align*}
\exp ( \max_i \eta a_i ) \le \sum_{i=1}^n \exp(\eta a_i) \le n \exp(\max_i \eta a_i)
\end{align*}
so
\begin{align*}
\max_i \eta a_i \le \log \left ( \sum_{i=1}^n \exp(\eta a_i) \right ) \le \log n + \max_i \eta a_i.
\end{align*}
The result follows by rearranging and dividing by $\eta$. 
\end{proof}

\begin{lem}\label{lem:eta_ordering}
If $\eta \ge \etatil \ge 0$, then $\Gopts(\bLambda;\eta) \le \Gopts(\bLambda;\etatil)$.
\end{lem}
\begin{proof}
We will prove this for some generic sequence $(a_i)_{i=1}^n$, $a_i \ge 0$. Note that,
\begin{align*}
\frac{\rmd}{\rmd \eta} \frac{1}{\eta} \log \left ( \sum_i e^{\eta a_i} \right ) & = -\frac{1}{\eta^2} \log \left ( \sum_i e^{\eta a_i} \right ) + \frac{1}{\eta} \frac{1}{ \sum_i e^{\eta a_i}} \cdot  \sum_i a_i e^{\eta a_i}.
\end{align*}
We are done if we can show this is non-positive. 
Note that, 
\begin{align*}
\log \left ( \sum_i e^{\eta a_i} \right ) \ge \max_i \log  \left ( e^{\eta a_i} \right ) = \max_i \eta a_i
\end{align*}
so 
\begin{align*}
-\frac{1}{\eta^2} \log \left ( \sum_i e^{\eta a_i} \right ) + \frac{1}{\eta} \frac{1}{ \sum_i e^{\eta a_i}} \cdot  \sum_i a_i e^{\eta a_i} & \le - \frac{1}{\eta} \max_i a_i + \frac{1}{\eta} \frac{1}{ \sum_i e^{\eta a_i}} \cdot  \sum_i a_i e^{\eta a_i} \\
& \le - \frac{1}{\eta} \max_i a_i + \frac{1}{\eta} \max_i a_i \\
& = 0.
\end{align*}
The result follows since $\Gopts$ has this form. 
\end{proof}

\begin{lem}\label{lem:xy_minimum_val}
We have,
\begin{align*}
\inf_{\bLambda \succeq 0, \| \bLambda \|_\op \le 1} \Gopt(\bLambda) \ge \frac{\gamphi}{1 + \| \bLambda_0 \|_\op }.
\end{align*}
\end{lem}
\begin{proof}
Note that $\| \bA(\bLambda) \|_\op \le 1 + \| \bLambda_0 \|_\op$, so 
\begin{align*}
\inf_{\bLambda \succeq 0, \| \bLambda \|_\op \le 1} \max_{\bphi \in \Phi} \| \bphi \|_{\bA(\bLambda)^{-1}}^2 \ge \inf_{\bLambda \succeq 0, \| \bLambda \|_\op \le 1 + \| \bLambda_0 \|_\op }  \| \bphi \|_{\bA(\bLambda)^{-1}}^2 \ge \frac{\max_{\bphi \in \Phi} \| \bphi \|_2}{1 + \| \bLambda_0 \|_\op } .
\end{align*}
\end{proof}

\begin{lem}\label{lem:nst_smooth_to_orig}
Assume that we set $\eta \ge \frac{2}{\gamphi} (1 + \| \bLambda_0 \|_\op) \cdot \log |\Phi| $. Then
\begin{align*}
\Nst  ( \epsilon;  \Gopts(\bLambda)  ) \le 2 \Nst(\epsilon;  \Gopt(\bLambda) ).
\end{align*}
\end{lem}
\begin{proof}
Denote $f(\bLambda) \leftarrow \LSE \left ( \{ e^{\eta \| \bphi \|_{\bA(\bLambda)^{-1}}^2} \}_{\bphi \in \Phi}; \eta \right )$. By \Cref{lem:xy_approx_error} and \Cref{lem:xy_minimum_val}, we have
\begin{align*}
| \max_{\bphi \in \Phi} \| \bphi \|_{\bA(\bLambda)^{-1}}^2 & - f(\bLambda) | \le \frac{\log | \Phi |}{\eta} \le \frac{\gamphi}{2(1 + \| \bLambda_0 \|_\op)} \le \min_{\bLambda \succeq 0, \| \bLambda \|_\op \le 1} \frac{1}{2} f(\bLambda) \\
& \implies f(\bLambda) \le 2 \max_{\bphi \in \Phi} \| \bphi \|_{\bA(\bLambda)^{-1}}^2 .
\end{align*} 
Let $\bLamst$ denote the matrix that minimizes $\max_{\bphi \in \Phi} \| \bphi \|_{\bA(\bLambda)^{-1}}^2$ over the constraint set: $\max_{\bphi \in \Phi} \| \bphi \|_{\bA(\bLamst)^{-1}}^2 = \inf_{\bLambda \in \bOmega} \max_{\bphi \in \Phi} \| \bphi \|_{\bA(\bLambda)^{-1}}^2$. Then it follows that, by definition of $\Nst(\epsilon;  \max_{\bphi \in \Phi} \| \bphi \|_{\bA(\bLambda)^{-1}}^2)$:
\begin{align*}
\max_{\bphi \in \Phi} \| \bphi \|_{\bA(\bLamst)^{-1}}^2 \le \epsilon \cdot  \Nst(\epsilon;  \max_{\bphi \in \Phi} \| \bphi \|_{\bA(\bLambda)^{-1}}^2).
\end{align*}
However, this implies 
\begin{align*}
\frac{1}{2} f(\bLamst) \le \epsilon \cdot  \Nst(\epsilon;  \max_{\bphi \in \Phi} \| \bphi \|_{\bA(\bLambda)^{-1}}^2),
\end{align*}
so $(\bLamst, 2\Nst(\epsilon;  \max_{\bphi \in \Phi} \| \bphi \|_{\bA(\bLambda)^{-1}}^2))$ is a feasible solution to the optimization \eqref{eq:Nst_opt} for $f$. As $\Nst(\epsilon;f)$ is the minimum solution, it follows that $\Nst(\epsilon; f) \le 2 \Nst(\epsilon;  \max_{\bphi \in \Phi} \| \bphi \|_{\bA(\bLambda)^{-1}}^2)$.
\end{proof}

\subsection{Bounding the Smoothness}
\begin{lem}\label{lem:smooth_gopt_fun}
$f(\bLambda) = \Gopts(\bLambda)$ satisfies all conditions of \Cref{def:smooth_exp_des_fun} with 
\begin{align*}
& L = \| \bLambda_0^{-1} \|_\op^2, \quad \beta = 2 \| \bLambda_0^{-1}\|_\op^3 ( 1 + \eta \| \bLambda_0^{-1}\|_\op), \quad M = \| \bLambda_0^{-1} \|_\op^2 
\end{align*}
\begin{align*}
& \nabla_{\bLambda} f(\bLambda) = \left ( \sum_{\bphi \in \Phi} e^{\eta \| \bphi \|_{\bA(\bLambda)^{-1}}^2} \right )^{-1} \cdot \sum_{\bphi \in \Phi} e^{\eta \| \bphi \|_{\bA(\bLambda)^{-1}}^2} \bA(\bLambda)^{-1} \bphi \bphi^\top \bA(\bLambda)^{-1} =: \Xi_{\bLambda}.
\end{align*}
\end{lem}
\begin{proof}
Using \Cref{lem:inverse_mat_derivative}, the gradient of $f(\bLambda)$ with respect to $\bLambda_{ij}$ is
\begin{align*}
\nabla_{\bLambda_{ij}} f(\bLambda) = - \left ( \sum_{\bphi \in \Phi} e^{\eta \| \bphi \|_{\bA(\bLambda)^{-1}}^2} \right )^{-1} \cdot \sum_{\bphi \in \Phi} e^{\eta \| \bphi \|_{\bA(\bLambda)^{-1}}^2} \bphi^\top \bA(\bLambda)^{-1} \be_i \be_j^\top \bA(\bLambda)^{-1} \bphi
\end{align*}
from which the expression for $\nabla_{\bLambda} f(\bLambda)$ follows directly.

To bound the Lipschitz constant of $f$, by the Mean Value Theorem it suffices to bound
\begin{align*}
\sup_{\bLambda,\bLamtil \succeq 0, \| \bLambda \|_\op \le 1, \| \bLamtil \|_\op \le 1} | \tr(\nabla f(\bLambda)^\top \bLamtil)  | & \le \left ( \sum_{\bphi \in \Phi} e^{\eta \| \bphi \|_{\bA(\bLambda)^{-1}}^2} \right )^{-1} \cdot \sum_{\bphi \in \Phi} e^{\eta \| \bphi \|_{\bA(\bLambda)^{-1}}^2} \| \bA(\bLambda)^{-1} \|_\op^2 \| \bLamtil \|_\op \\
& \le \| \bLambda_0^{-1} \|_\op^{2}
\end{align*}
where the last inequality follows since $\bA(\bLambda) \succeq \bLambda_0$ for all $\bLambda$. This also suffices as a bound on $M$.

To bound the smoothness, again by the Mean Value Theorem it suffices to bound the operator norm of the Hessian. Standard calculus gives that, using $\nabla^2 f(\bLambda)[\bLamtil,\bLambar]$ to denote the Hessian of $f$ in direction $(\bLamtil,\bLambar)$:
\begin{align*}
&  \nabla^2 f(\bLambda)[\bLamtil,\bLambar]=  -\frac{d}{dt} \left ( \sum_{\bphi \in \Phi} e^{\eta \| \bphi \|_{\bA(\bLambda + t \bLambar)^{-1}}^2} \right )^{-1} \cdot \sum_{\bphi \in \Phi} e^{\eta \| \bphi \|_{\bA(\bLambda + t \bLambar)^{-1}}^2}  \bphi^\top \bA(\bLambda + t \bLambar)^{-1} \bLamtil \bA(\bLambda + t \bLambar)^{-1} \bphi \\
& = -\eta \left ( \sum_{\bphi \in \Phi} e^{\eta \| \bphi \|_{\bA(\bLambda)^{-1}}^2} \right )^{-2}  \left ( \sum_{\bphi \in \Phi} e^{\eta \| \bphi \|_{\bA(\bLambda )^{-1}}^2}  \bphi^\top \bA(\bLambda )^{-1} \bLambar \bA(\bLambda)^{-1} \bphi  \right ) \left ( \sum_{\bphi \in \Phi} e^{\eta \| \bphi \|_{\bA(\bLambda )^{-1}}^2}  \bphi^\top \bA(\bLambda )^{-1} \bLamtil \bA(\bLambda)^{-1} \bphi \right ) \\
& \qquad + \eta \left ( \sum_{\bphi \in \Phi} e^{\eta \| \bphi \|_{\bA(\bLambda)^{-1}}^2} \right )^{-1} \sum_{\bphi \in \Phi} e^{\eta \| \bphi \|_{\bA(\bLambda)^{-1}}^2}  \left (\bphi^\top \bA(\bLambda)^{-1} \bLambar \bA(\bLambda)^{-1} \bphi \right ) \left (\bphi^\top \bA(\bLambda)^{-1} \bLamtil \bA(\bLambda)^{-1} \bphi \right ) \\
& \qquad + \left ( \sum_{\bphi \in \Phi} e^{\eta \| \bphi \|_{\bA(\bLambda)^{-1}}^2} \right )^{-1} \sum_{\bphi \in \Phi} e^{\eta \| \bphi \|_{\bA(\bLambda)^{-1}}^2} \bphi^\top \bA(\bLambda)^{-1} \bLambar \bA(\bLambda)^{-1} \bLamtil \bA(\bLambda)^{-1} \bphi \\
& \qquad + \left ( \sum_{\bphi \in \Phi} e^{\eta \| \bphi \|_{\bA(\bLambda)^{-1}}^2} \right )^{-1} \sum_{\bphi \in \Phi} e^{\eta \| \bphi \|_{\bA(\bLambda)^{-1}}^2} \bphi^\top \bA(\bLambda)^{-1} \bLamtil \bA(\bLambda)^{-1} \bLambar \bA(\bLambda)^{-1} \bphi .
\end{align*}
We can bound this as
\begin{align*}
\sup_{\bLambda,\bLamtil, \bLambar \succeq 0, \| \bLambda \|_\op \le 1, \| \bLamtil \|_\op \le 1, \| \bLambar \|_\op \le 1} |\nabla^2 f(\bLambda)[\bLamtil,\bLambar]| & \le 2\eta \| \bLambda_0^{-1} \|_\op^4 + 2 \| \bLambda_0^{-1} \|_\op^3.
\end{align*}
Convexity of $f(\bLambda)$ follows since it is the composition of a convex function with a strictly increasing convex function, so it is itself convex.
\end{proof}

\begin{lem}\label{lem:inverse_mat_derivative}
For $\bLambda$ invertible, $\frac{d}{dt} (\bLambda + t \be_i \be_j^\top)^{-1} = - \bLambda^{-1} \be_i \be_j^\top \bLambda^{-1}$.
\end{lem}
\begin{proof}
We can compute the gradient as
\begin{align*}
\frac{d}{dt} (\bLambda + t \be_i \be_j^\top)^{-1}  = \lim_{t\rightarrow 0} \frac{(\bLambda + t \be_i \be_j^\top)^{-1} - \bLambda^{-1}}{t}.
\end{align*}
By the Sherman-Morrison formula,
\begin{align*}
(\bLambda + t \be_i \be_j^\top)^{-1} = \bLambda^{-1} - \frac{t \bLambda^{-1} \be_i \be_j^\top \bLambda^{-1}}{1 + t \be_j^\top \bLambda^{-1} \be_i}
\end{align*}
so as $t \rightarrow 0$, 
\begin{align*}
(\bLambda + t \be_i \be_j^\top)^{-1} \rightarrow \bLambda^{-1} - t \bLambda^{-1} \be_i \be_j^\top \bLambda^{-1}
\end{align*}
Thus,
\begin{align*}
\lim_{t\rightarrow 0} \frac{(\bLambda + t \be_i \be_j^\top)^{-1} - \bLambda^{-1}}{t} & = \lim_{t\rightarrow 0} \frac{\bLambda^{-1} - t \bLambda^{-1} \be_i \be_j^\top \bLambda^{-1}  - \bLambda^{-1}}{t} = - \bLambda^{-1} \be_i \be_j^\top \bLambda^{-1}.
\end{align*}
\end{proof}

\subsection{Obtaining Well-Conditioned Covariates}

\begin{algorithm}[h]
\begin{algorithmic}[1]
\State \textbf{input}: Scale $N$, minimum eigenvalue $\lamun$, confidence $\delta$
	\For{$j = 1,2,3,\ldots$}\label{line:collect_well_cond_cov}
		\State $T_j \leftarrow \poly(2^j,d,H,\log 1/\delta)$ 
		\State $\epsilon_j \leftarrow 2^{-j}$, $\gamma_j^2 \leftarrow \frac{2^{-j}}{\max \{ 12544 d \log \frac{2N(2+32T_j)}{\delta}, \lamun \}}$, $\delta_j \leftarrow \delta/(4j^2)$
		\State Run Algorithm 5 of \cite{wagenmaker2022reward} with parameters $(\epsilon_j,\gamma_j^2,\delta_j)$, obtain covariates $\bSigtil$ and store policies run as $\Pitil$
		\If{$\lambda_{\min}(\bLamtil) \ge \max \{ 12544 d \log \frac{2N(2+32T_j)}{\delta}, \lamun \}$} \label{line:min_eig_suff}
			\State \textbf{break}
		\EndIf
	\EndFor
	\State Rerun every policy $\pi \in \Pitil$ $\lceil N / |\Pitil | \rceil$ times, collect covariates $\bSigbar$
	\State \textbf{return} $\bSigtil + \bSigbar$
\end{algorithmic}
\caption{Collect Well-Conditioned Covariates (\condcov)}
\label{alg:cov_min_eig}
\end{algorithm}

\begin{lem}\label{lem:rerun_cov}
Consider running policies $(\pi_\tau)_{\tau=1}^T$, for $\pi_\tau$ $\cF_{\tau-1}$-measurable, and collecting covariance $\bSigma_T = \sum_{\tau=1}^T \bphi_\tau \bphi_\tau^\top$. Then as long as
\begin{align*}
\lambda_{\min}(\bSigma_T) \ge 12544 d \log \frac{2+32T}{\delta}.
\end{align*}
with probability at least $1-\delta$, if we rerun each $(\pi_\tau)_{\tau=1}^T$, we will collect covariates $\bSigtil_T$ such that
\begin{align*}
\lambda_{\min}(\bSigtil_T) \ge \frac{1}{2} \lambda_{\min}(\bSigma_T).
\end{align*}
\end{lem}
\begin{proof}
Let $\cN$ be an $\frac{1}{8 T}$-net of $\cS^{d-1}$. Let $\bSigma \succeq 0$ be any matrix with $\| \bSigma \|_\op \le T$ and let $\bv$ be the minimum eigenvalue of $\bSigma$. Let $\bvtil \in \cN$ be the element of $\cN$ closest to $\bv$ in the $\ell_2$ norm. Then:
\begin{align*}
\lammin(\bSigma) = \bv^\top \bSigma \bv & = \bvtil^\top \bSigma \bvtil + ( \bv^\top \bSigma \bv - \bvtil^\top \bSigma \bvtil) \\
& \ge \bvtil^\top \bSigma \bvtil - | \bv^\top \bSigma \bv - \bv^\top \bSigma \bvtil| - | \bv^\top \bSigma \bvtil - \bvtil^\top \bSigma \bvtil | \\
& \ge \bvtil^\top \bSigma \bvtil - 2 \| \bSigma \|_\op \| \bvtil - \bv \|_2.
\end{align*}
By the construction of $\cN$ and since $\| \bSigma \|_\op \le T$, we can bound $2 \| \bSigma \|_\op \| \bvtil - \bv \|_2 \le 1 / 4$, so 
\begin{align*}
\bvtil^\top \bSigma \bvtil - 2 \| \bSigma \|_\op \| \bvtil - \bv \|_2 \ge \bvtil^\top \bSigma \bvtil - 1/4
\end{align*}
which implies
\begin{align}\label{eq:min_eig_cover}
\lammin(\bSigma) + 1/4 \ge \bvtil^\top \bSigma \bvtil \ge \min_{\bvtil \in \cN} \bvtil^\top \bSigma \bvtil.
\end{align}
By \Cref{lem:euc_ball_cover}, we can bound $|\cN|  \le (1 + 16 T )^d$. 

Note that $\Var[\bv^\top \bphi_\tau | \cF_{\tau-1}] \le \Exp_{\pi_\tau}[(\bv^\top \bphi_\tau)^2]$ so $\sum_{\tau=1}^T \Var[\bv^\top \bphi_\tau | \cF_{\tau-1}]  \le \bv^\top \Exp[\bSigma_T | \pi_1,\ldots, \pi_T] \bv$ for $\Exp[\bSigma_T | \pi_1,\ldots, \pi_T] = \sum_{\tau = 1}^T \Exp_{\pi_\tau} [\bphi_\tau \bphi_\tau^\top]$. By Freedman's Inequality (\Cref{lem:freedmans}), for all $\bv \in \cN$ simultaneously, we will have, with probability at least $1-\delta$,
\begin{align}
& \left | \bv^\top \bSigma_T \bv - \bv^\top \Exp[\bSigma_T | \pi_1,\ldots, \pi_T] \bv \right | \le 2\sqrt{ \bv^\top \Exp[\bSigma_T | \pi_1,\ldots, \pi_T]\bv \log \frac{2|\cN|}{\delta}} +  \log \frac{2|\cN|}{\delta} \label{eq:min_eig_cov_conc1} \\
& \left | \bv^\top \bSigtil_T \bv - \bv^\top \Exp[\bSigma_T | \pi_1,\ldots, \pi_T] \bv \right | \le 2\sqrt{\bv^\top \Exp[\bSigma_T | \pi_1,\ldots, \pi_T]\bv \log \frac{2|\cN|}{\delta}} +  \log \frac{2|\cN|}{\delta} .\label{eq:min_eig_cov_conc2}
\end{align}
Rearranging \eqref{eq:min_eig_cov_conc1}, some algebra shows that
\begin{align*}
\bv^\top \Exp[\bSigma_T | \pi_1,\ldots, \pi_T] \bv & \le \bv^\top \bSigma_T \bv + 3 \log \frac{2|\cN|}{\delta} + 2 \sqrt{\bv^\top \bSigma_T \bv \log \frac{2|\cN|}{\delta} + 2 \log^2 \frac{2|\cN|}{\delta}} \\
& \le \bv^\top \bSigma_T \bv + 6 \log \frac{2|\cN|}{\delta} + 2 \sqrt{\bv^\top \bSigma_T \bv \log \frac{2|\cN|}{\delta} } \\
& \le 3 \bv^\top \bSigma_T \bv + 8 \log \frac{2|\cN|}{\delta}
\end{align*}
where the last inequality uses $\sqrt{ab} \le \max \{ a, b \}$. 
Thus, if \eqref{eq:min_eig_cov_conc1} and \eqref{eq:min_eig_cov_conc2} hold, we have
\begin{align*}
\bv^\top \bSigtil_T \bv & \ge \bv^\top \bSigma_T \bv - 4 \sqrt{ \bv^\top \Exp[\bSigma_T | \pi_1,\ldots, \pi_T] \bv \log \frac{2|\cN|}{\delta} } - 2 \log \frac{2|\cN|}{\delta} \\
& \ge \bv^\top \bSigma_T \bv - 4 \sqrt{ 3 \bv^\top \bSigma_T \bv \log \frac{2|\cN|}{\delta} } - 14 \log \frac{2|\cN|}{\delta}
\end{align*}
Therefore, as long as
\begin{align*}
\bv^\top \bSigma_T \bv \ge 12544 \log \frac{2|\cN|}{\delta}, 
\end{align*}
we can lower bound
\begin{align*}
\bv^\top \bSigma_T \bv - 4 \sqrt{ 3 \bv^\top \bSigma_T \bv \log \frac{2|\cN|}{\delta} } - 14 \log \frac{2|\cN|}{\delta} \ge \frac{3}{4} \bv^\top \bSigma_T \bv \ge \frac{3}{4} \lammin(\bSigma_T)
\end{align*}
so, for all $\bv \in \cN$,
\begin{align*}
\bv^\top \bSigtil_T \bv \ge \frac{3}{4} \lammin(\bSigma_T).
\end{align*}
By assumption, $\lammin(\bSigma_T) \ge 12544 d \log \frac{2+32T}{\delta}$, which implies, since $|\cN| \le (1 + 16 T )^d$, that for all $\bv \in \cS^{d-1}$, $\bv^\top \bSigma_T \bv \ge 12544 \log \frac{2|\cN|}{\delta}$, so the above condition will be met. 

Since $\| \bSigtil_T \|_\op \le T$, we can apply \eqref{eq:min_eig_cover} to then get that
\begin{align*}
\lammin(\bSigtil_T) \ge \frac{3}{4} \lammin(\bSigma_T) - 1/4 \ge \frac{1}{2} \lammin(\bSigma_T) + \frac{1}{4} (\lammin(\bSigma_T) -1).
\end{align*}
Since we have already establishes that $\lammin(\bSigma_T) \ge 12544 d \log \frac{2+32T}{\delta}$, we have $\lammin(\bSigma_T) -1 \ge 0$, so we can lower bound
\begin{align*}
\lammin(\bSigtil_T) \ge \frac{1}{2} \lammin(\bSigma_T).
\end{align*}
\end{proof}

\begin{lem}\label{lem:well_cond_cov}
With probability at least $1-\delta$, \Cref{alg:cov_min_eig} will terminate after at most
\begin{align*}
N +  \poly\log \left ( \frac{1}{\sup_\pi \lambda_{\min}(\bSigma_\pi)}, d, H , \lamun, \log \frac{N}{\delta} \right ) \cdot \left ( \frac{d \max \{ d \log \frac{N}{\delta}, \lamun \}}{\sup_\pi \lambda_{\min}(\bSigma_\pi)^2} + \frac{d^4 H^3 \log^{7/2} \frac{1}{\delta}}{\sup_\pi \lambda_{\min}(\bSigma_\pi)}\right ) 
\end{align*}
episodes, and will return covariates $\bSigma$ such that
\begin{align*}
\lammin(\bSigma) & \ge N  \cdot \min \left \{ \frac{\sup_\pi \lambda_{\min}(\bSigma_\pi)^2}{d}, \frac{\sup_\pi \lambda_{\min}(\bSigma_\pi)}{d^3 H^3 \log^{7/2} 1/\delta} \right \} \cdot  \poly\log \left ( \frac{1}{\sup_\pi \lambda_{\min}(\bSigma_\pi)}, d, H , \lamun, \log \frac{N}{\delta} \right )^{-1} \\
& \qquad + \max \{ d \log 1/\delta, \lamun \} 
\end{align*}
and
\begin{align*}
\| \bSigma \|_\op \le N +  \poly\log \left ( \frac{1}{\sup_\pi \lambda_{\min}(\bSigma_\pi)}, d, H , \lamun, \log \frac{N}{\delta} \right ) \cdot \left ( \frac{d \max \{ d \log \frac{N}{\delta}, \lamun \}}{\sup_\pi \lambda_{\min}(\bSigma_\pi)^2} + \frac{d^4 H^3 \log^{7/2} \frac{1}{\delta}}{\sup_\pi \lambda_{\min}(\bSigma_\pi)}\right )  .
\end{align*}
\end{lem}
\begin{proof}
By Theorem 4 of \cite{wagenmaker2022reward}, as long as Algorithm 5 of \cite{wagenmaker2022reward} is run with parameters $\epsilon$ and $\gamma^2$, it will terminate after at most
\begin{align*}
c_1 \cdot \frac{1}{\epsilon} \max \left \{ \frac{d m}{\gamma^2} \log \frac{d m}{\epsilon \gamma^2},  d^4 H^3 m^{7/2} \log^{3/2}(d/\gamma^2) \log^{7/2} \frac{c_2 m d H \log (d/\gamma^2)}{\delta} \right \} 
\end{align*}
episodes for $m = \lceil \log(2/\epsilon) \rceil$ (to get the slightly more precise bound on the number of episodes collected than that given in Theorem 4 of \cite{wagenmaker2022reward}, we use the precise definition of $K_i$ given at the start of Appendix B). Furthermore, if $\epsilon \le \sup_\pi \lambda_{\min}(\bSigma_\pi)$, with probability at least $1-\delta$ it will collect covariates $\bSigtil$ satisfying $\lambda_{\min}(\bSigtil) \ge \epsilon/\gamma^2$. 

It follows that, by our choice of $\epsilon_j = 2^{-j}$, $\gamma_j^2 = \frac{2^{-j}}{\max \{ 12544 d \log \frac{2N(2+32T_j)}{\delta}, \lamun \}}$, and $\delta_j = \delta/(4j^2)$, for every $j$ we will collect at most
\begin{align*}
c_1 \cdot 2^j \max \left \{ 2^j d j^2 \max \{  d \log \frac{2N(2+32T_j)}{\delta}, \lamun \} \log (d j a_j),  d^4 H^3 j^{5} \log^{3/2}(d a_j) \log^{7/2} \frac{c_2 j^4 d H \log (d a_j)}{\delta} \right \} 
\end{align*}
episodes, where we denote $a_j := \max \{ 12544 d \log \frac{2N(2+32T_j)}{\delta}, \lamun \}$. Note that $T_j$ is an upper bound on this complexity.
Furthermore, once $j$ is large enough that $2^{-j} \le \sup_\pi \lambda_{\min}(\bSigma_\pi)$, Theorem 4 of \cite{wagenmaker2022reward} implies that the condition $\lambda_{\min}(\bSigtil) \ge \epsilon_j/\gamma_j^2$ will be met. By our choice of $\gamma_j^2$ and $\epsilon_j$, it follows that the if condition on \Cref{line:min_eig_suff} will be met once $2^{-j} \le \sup_\pi \lambda_{\min}(\bSigma_\pi)$. Since $2^{-j}$ decreases by a factor of 2 each time, it follows that the if statement on \Cref{line:min_eig_suff} will have terminated once $2^{-j} \ge \sup_\pi \lambda_{\min}(\bSigma_\pi)/2$. This implies that the total number of episodes collected before the if statement on \Cref{line:min_eig_suff} is met is bounded as
\begin{align}\label{eq:lammin_collect_comple_up}
 \poly\log \left ( \frac{1}{\sup_\pi \lambda_{\min}(\bSigma_\pi)}, d, H , \lamun, \log \frac{N}{\delta} \right ) \cdot \left ( \frac{d \max \{ d \log \frac{N}{\delta}, \lamun \}}{\sup_\pi \lambda_{\min}(\bSigma_\pi)^2} + \frac{d^4 H^3 \log^{7/2} \frac{1}{\delta}}{\sup_\pi \lambda_{\min}(\bSigma_\pi)}\right ) 
\end{align}
By \Cref{lem:rerun_cov}, since $\lammin(\bSigtil) \ge \max \{ 12544 d \log \frac{2N(2+32T_j)}{\delta}, \lamun \}$ and $T_j$ is an upper bound on the number of episodes run at epoch $j$, every time we run all policies $\pi \in \Pitil$, with probability at least $1-\delta/(2N)$, we will collect covariates $\bSigma$ such that 
\begin{align*}
\lammin(\bSigma) \ge \lammin(\bSigtil)/2 \ge \frac{1}{2}  \max \{ 12544 d \log \frac{2N(2+32T_j)}{\delta}, \lamun \}.
\end{align*}
Thus, if we rerun every policy $\lceil N / | \Pitil | \rceil$ times to create covariates $\bSigbar$, with probability at least $1-\delta/2$, we have
\begin{align*}
\lammin(\bSigbar) \ge \frac{N}{2 | \Pitil|}  \max \{ 12544 d \log \frac{2N(2+32T_j)}{\delta}, \lamun \}.
\end{align*}
Note that this procedure will complete after at most $N + |\Pitil|$ episodes. Furthermore, $|\Pitil| \le \eqref{eq:lammin_collect_comple_up}$, so we can lower bound
\begin{align*}
\lammin(\bSigbar) \ge N  \cdot \min \left \{ \frac{\sup_\pi \lambda_{\min}(\bSigma_\pi)^2}{d}, \frac{\sup_\pi \lambda_{\min}(\bSigma_\pi)}{d^3 H^3 \log^{7/2} 1/\delta} \right \} \cdot  \poly\log \left ( \frac{1}{\sup_\pi \lambda_{\min}(\bSigma_\pi)}, d, H , \lamun, \log \frac{N}{\delta} \right )^{-1}.
\end{align*}
The final lower bound on the returned covariates follows since we return $\bSigbar + \bSigtil$, and we know that $\lammin(\bSigtil) \ge \max \{ 12544 d \log \frac{2N(2+32T_j)}{\delta}, \lamun \}$. The upper bound on $\| \bSigbar + \bSigtil \|_\op$ follows since every feature vector encountered has norm of at most 1. 

The failure probability of each call to Algorithm 5 of \cite{wagenmaker2022reward} is $\delta/(4 j^2)$, so the total failure probability of \Cref{alg:cov_min_eig} is
\begin{align*}
\sum_{j=1}^{\infty} \frac{\delta}{4j^2} = \frac{\pi^2}{24} \delta \le \delta/2. 
\end{align*}

\end{proof}

\subsection{Online $\mathsf{XY}$-Optimal Design}

\begin{thm}[Full version of \Cref{cor:g_opt_exp_design}]\label{cor:gopt_exp_design}
Consider running \optcov with some $\epsilon > 0$ and functions
\begin{align*}
f_i(\bLambda) \leftarrow \Gopts(\bLambda)
\end{align*}
for $\bLambda_0 \leftarrow (T_i K_i)^{-1} \bSigma_i =: \bLambda_i$ and  
\begin{align*}
& \eta_i = \frac{2}{\gamphi} \cdot ( 1 + \| \bLambda_i \|_{\op} ) \cdot \log | \Phi | \\
& L_i = \| \bLambda_i^{-1} \|_\op^2, \quad \beta_i = 2 \| \bLambda_i^{-1}\|_\op^3 ( 1 + \eta_i \| \bLambda_i^{-1}\|_\op), \quad M_i = \| \bLambda_i^{-1} \|_\op^2 
\end{align*}
where $\bSigma_i$ is the matrix returned by running \condcov with $N \leftarrow T_i K_i$, $\delta \leftarrow \delta/(2i^2)$, and some $\lamun \ge 0$.
Then with probability $1-2\delta$, this procedure will collect at most
\begin{align*}
20 \cdot \frac{ \inf_{\bLambda \in \bOmega}  \max_{\bphi \in \Phi} \| \bphi \|_{\bA(\bLambda)^{-1}}^2}{\epsexp} + \poly \left (d, H, \log 1/\delta, \frac{1}{\lamminst}, \frac{1}{\gamphi}, \lamun, \log | \Phi |, \log \frac{1}{\epsexp} \right )
\end{align*}
episodes, where 
$$\bA(\bLambda) = \bLambda +  \min \left \{ \frac{(\lamminst)^2}{d}, \frac{\lamminst}{d^3 H^3 \log^{7/2} 1/\delta} \right \} \cdot  \poly\log \left ( \frac{1}{\lamminst}, d, H , \lamun, \log \frac{1}{\delta} \right )^{-1} \cdot I, $$ 
and will produce covariates $\bSighat + \bSigma_i$ such that
\begin{align*}
\max_{\bphi \in \Phi} \| \bphi \|_{(\bSighat + \bSigma_i)^{-1}}^2 \le \epsexp
\end{align*}
and
\begin{align*}
\lammin(\bSighat + \bSigma_i) \ge \max \{ d \log 1/\delta, \lamun \}.
\end{align*}
\end{thm}
\begin{proof}
Note that the total failure probability of our calls to \condcov is at most
\begin{align*}
\sum_{i=1}^\infty \frac{\delta}{2i^2} = \frac{\pi^2}{12} \delta \le \delta.
\end{align*}
For the remainder of the proof, we will then assume that we are on the success event of \condcov, as defined in \Cref{lem:well_cond_cov}.

By \Cref{lem:smooth_gopt_fun}, $f_i(\bLambda)$ satisfies \Cref{def:smooth_exp_des_fun} with constants
\begin{align*}
& L_i = \| \bLambda_i^{-1} \|_\op^2, \quad \beta_i = 2 \| \bLambda_i^{-1}\|_\op^3 ( 1 + \eta_i \| \bLambda_i^{-1}\|_\op), \quad M_i = \| \bLambda_i^{-1} \|_\op^2 
\end{align*}
for $\bLambda_i \leftarrow (T_i K_i)^{-1} \bSigma_i$. 

By \Cref{lem:well_cond_cov}, on the success event of \Cref{lem:well_cond_cov} we have that 
\begin{align*}
\lammin(\bLambda_i) \ge \min \left \{ \frac{(\lamminst)^2}{d}, \frac{\lamminst}{d^3 H^3 \log^{7/2} 1/\delta} \right \} \cdot  \poly\log \left ( \frac{1}{\lamminst}, d, H , \lamun, i, \log \frac{1}{\delta} \right )^{-1}
\end{align*}
(note that the $\poly \log(i)^{-1}$ dependence arises because we take $N \leftarrow T_i K_i = 2^{4i}$).
Thus, we can bound, for all $i$ (using the upper bound on $\| \bSigma_i \|_\op$ given in \Cref{lem:well_cond_cov} to upper bound $\eta_i$),
\begin{align*}
& L_i = M_i \le  \max \left \{ \frac{d^2}{(\lamminst)^4}, \frac{d^6 H^6 \log^{7} 1/\delta} {(\lamminst)^2}\right \} \cdot  \poly\log \left ( \frac{1}{\lamminst}, d, H , \lamun, i, \log \frac{1}{\delta} \right ), \\
& \beta_i \le   \poly \left ( d, H, \log 1/\delta, \frac{1}{\lamminst},\frac{1}{\gamphi}, \lamun, i, \log | \Phi | \right ) .
\end{align*}
Assume that the termination condition of \optcov for $\ihat$ satisfying
\begin{align}\label{eq:xy_design_ihat_bound}
\ihat \le \log \left ( \poly \left ( \frac{1}{\epsexp}, d, H, \log 1/\delta, \frac{1}{\lamminst}, \frac{1}{\gamphi}, \lamun, \log |\Phi| \right ) \right ).
\end{align}
We assume this holds and justify it at the conclusion of the proof. For notational convenience, define
\begin{align*}
\iota := \poly \left ( \log \frac{1}{\epsexp}, d, H, \log 1/\delta, \frac{1}{\lamminst}, \frac{1}{\gamphi}, \lamun, \log |\Phi| \right ) .
\end{align*}
Given this upper bound on $\ihat$, set 
\begin{align*}
& L = M := \max \left \{ \frac{d^2}{(\lamminst)^4}, \frac{d^6 H^6 \log^{7} 1/\delta} {(\lamminst)^2}\right \} \cdot  \poly\log \iota, \qquad \beta := \iota.
\end{align*}
With this choice of $L,M,\beta$, we have $L_i \le L, M_i \le M,\beta_i \le \beta$ for all $i \le \ihat$.

Now take $f(\bLambda) \leftarrow \Gopts(\bLambda; \eta, \bLambda_0)$ with 
\begin{align}\label{eq:gopt_f_A0}
\bLambda_0 \leftarrow \min \left \{ \frac{(\lamminst)^2}{d}, \frac{\lamminst}{d^3 H^3 \log^{7/2} 1/\delta} \right \} \cdot  \frac{1}{\poly \log \iota} \cdot I 
\end{align}
and 
\begin{align*}
\eta =  \frac{2 \log | \Phi |}{\gamphi} \cdot  \left ( 1 + \min \left \{ \frac{(\lamminst)^2}{d}, \frac{\lamminst}{d^3 H^3 \log^{7/2} 1/\delta} \right \} \cdot  \frac{1}{\poly\log \iota} \right ) .
\end{align*}
Note that in this case, we have $\| \bLambda_0 \|_\op \le \lammin(\bLambda_i)$ for all $i$, so $\bLambda_0 \preceq \bLambda_i$ and $\eta \le \eta_i$. By the construction of $\Gopts$ and \Cref{lem:eta_ordering}, it follows that $f(\bLambda) \ge f_i(\bLambda)$ for all $\bLambda \succeq 0$, so this is a valid choice of $f$, as required by \Cref{thm:regret_data_fw}. Furthermore, we can set $R = 2$, since $\| \Exp_{(s,a) \sim \omega} [\bphi(s,a) \bphi(s,a)^\top] \|_{\fro} \le 1$ for all $\omega \in \simplex_{\cS \times \cA}$.

To apply \Cref{thm:regret_data_fw}, it remains only to find a suitable value of $\fmin$. By \Cref{lem:xy_approx_error} and \Cref{lem:xy_minimum_val}, we can lower bound $f_i$ by $\frac{\gamphi}{1+ \| \bLambda_i \|_\op}$. By \Cref{lem:well_cond_cov}, we can lower bound
\begin{align*}
\frac{\gamphi}{1+ \| \bLambda_i \|_\op} \ge \frac{\gamphi}{2 +  \poly\log \iota \cdot \left ( \frac{d \max \{ d \log \frac{1}{\delta}, \lamun \}}{(\lamminst)^2} + \frac{d^4 H^3 \log^{7/2} \frac{1}{\delta}}{\lamminst}\right )}.
\end{align*}
We then take this as our choice of $\fmin$. 

We can now apply \Cref{thm:regret_data_fw}, using the complexity for \optcov instantiated with \force given in \Cref{cor:Nst_opt_force}, and get that with probability at least $1-\delta$, \optcov will terminate in
\begin{align*}
N \le 5 \Nst \left ( \epsexp/2 ; f \right ) + \iota
\end{align*}
episodes, and will return (time-normalized) covariates $\bLamhat$ such that
\begin{align*}
f_{\ihat}(\bLamhat) \le N \epsexp.
\end{align*}
By \Cref{lem:nst_smooth_to_orig}, our choice of $\eta$ and $\bLambda_0$, we can upper bound
\begin{align*}
\Nst \left ( \epsexp/2 ; f \right )  \le 2\Nst \left (\epsexp/2; \Gopt \right )  = \frac{4 \inf_{\bLambda \in \bOmega}  \max_{\bphi \in \Phi} \| \bphi \|_{\bA(\bLambda)^{-1}}^2}{\epsexp}
\end{align*}
where here $\bA(\bLambda) = \bLambda + \bLambda_0$ for $\bLambda_0$ as in \eqref{eq:gopt_f_A0}. Furthermore, by \Cref{lem:xy_approx_error} we have
\begin{align*}
\max_{\bphi \in \Phi} \| \bphi \|^2_{(\bSighat + \bLambda_0)^{-1}} \le f_{\ihat}(\bLamhat).
\end{align*}

The final upper bound on the number of episodes collected and the lower bound on the minimum eigenvalue of the covariates follows from \Cref{lem:well_cond_cov}.

It remains to justify our bound on $\ihat$, \eqref{eq:xy_design_ihat_bound}. Note that by definition of $\optcov$, if we run for a total of $\bar{N}$ episodes, we can bound $\ihat \le \frac{1}{4} \log_2(\bar{N})$. However, we see that the bound on $\ihat$ given in \eqref{eq:xy_design_ihat_bound} upper bounds $\frac{1}{4} \log_2(\bar{N})$ for $\bar{N}$ the upper bound on the number of samples collected by \optcov stated above. Thus, our bound on $\ihat$ is valid. 
\end{proof}


\newcommand{\Thetaalt}{\Theta_{\mathrm{alt}}}
\newcommand{\rhost}{\rho^\star}
\newcommand{\simptil}{\widetilde{\simplex}}
\newcommand{\bz}{\bm{z}}
\newcommand{\bAtil}{\widetilde{\bA}}
\newcommand{\byz}{\by_{\bz}}
\newcommand{\zhat}{\widehat{\bz}}
\newcommand{\Thetalb}{\Theta_{\mathrm{lb}}}

\section{Suboptimality of Low-Regret Algorithms}

\newcommand{\omegahat}{\widehat{\omega}}

\subsection{Linear Bandit Construction}\label{sec:lin_bandit_hard}
In the linear bandit setting, at each time step $t$, the learner chooses some $\bz_t \in \cZ$, and observes $y_t$. We will consider the case when the noise is Bernoulli so that $y_t \sim \bern(\inner{\bthetast}{\bz_t} + 1/2)$, and will set
\begin{align*}
& \bthetast = \be_1, \quad \cZ = \{ \xi \be_1, \be_2, \ldots, \be_d, -\xi \be_1, -\be_2, \ldots, -\be_d, \bx_2, \ldots, \bx_d \}, \quad \bx_i = (\xi - \Delta) \be_1 + \gamma \be_i 
\end{align*}
for some $\xi$ and $\Delta$ to be chosen. In this setting, the optimal arm is $\zst = \xi \be_1$, and $\Delta(\be_i) = \xi, i \ge 2$, $\Delta(\bx_i) = \Delta$. 

We will assume:
\begin{align}\label{eq:lb_instance_constraints}
\frac{1}{52d} \ge \xi \ge \max \{ \gamma/\sqrt{d}, \sqrt{\Delta} \}, \quad \max \left \{ \zeta :=  \frac{\xi \cC_1}{(d/200\Delta^2)^{1-\alpha}} + \frac{200 \xi \cC_2 \Delta^2}{d}, \Delta \right \} \le \gamma^2 .
\end{align}
We provide explicit values for $\xi,\Delta,$ and $\gamma$ that satisfy this in \Cref{lem:lin_bandit_param}.

\begin{defn}[$(\epsilon,\delta)$-correct Stopping Rule]
We say a stopping rule $\tau$ is $(\epsilon,\delta)$-correct over $\Theta$ if for all $\btheta \in \Theta$, $\Pr_{\btheta}[\btheta^\top \Exp_{\bz \sim \omegahat_\tau}[\bz] \ge \max_{\bz \in \cZ} \btheta^\top \bz - \epsilon] \ge 1 - \delta$, where $\omegahat_\tau \in \simplex_{\cZ}$ is a distribution over arms recommended at time $\tau$. 
\end{defn}

We will say that some $\omega \in \simplex_{\cZ}$ is $\epsilon$-optimal for $\btheta$ if $\btheta^\top \Exp_{\bz \sim \omega}[\bz] \ge \max_{\bz \in \cZ} \btheta^\top \bz - \epsilon$. We have the following. 

\begin{lem}\label{lem:lin_band_lower_bound}
Consider running some low-regret algorithm satisfying \Cref{def:low_regret} on the linear bandit instance described above and let $\tau$ be some stopping rule. Then if $\tau$ is $(\epsilon,\delta)$-correct on the set of instances $\Thetalb := \{ \btheta \in \R^d \ : \ \| \btheta \|_2 \le \frac{1}{4} ( \sqrt{d} - 2), |\inner{\btheta}{\bz}| \le 1/2, \forall \bz \in \cZ \}$ for $\delta \in (0,1/10)$, $\epsilon < \min \{ \Delta/2, \xi \}$, $d \ge 2116$, and where $\xi,\gamma,$ and $\Delta$ are chosen as in \Cref{lem:lin_bandit_param}, we must have that
\begin{align*}
\Exp_{\bthetast}[\tau] \ge \frac{d}{200 \Delta^2} \cdot \log \frac{1}{2.4 \delta}.
\end{align*}
\end{lem}
\begin{proof}
This proof follows closely the proof of Theorem 1 of \cite{fiez2019sequential} and relies on the Transportation Lemma of \cite{kaufmann2016complexity}.

\paragraph{Bounding the number of pulls to $\{ \pm \be_2,\ldots,\pm \be_d\}$.}
By assumption, we collect data with a low-regret algorithm satisfying \Cref{def:low_regret}. Every time we pull $\pm \be_i, i \ge 2$, we incur a loss of $\xi$. Thus, we can lower bound
\begin{align*}
\Exp[\Vst_0 - V_0^{\pi_k}] \ge \xi \sum_{i=2}^d \Exp[ \Pr_{\pi_k}[ \{\bz_k = \be_i \} \cup \{ \bz_k = - \be_i \}]]
\end{align*}
so, letting $T(\be_i)$ denote the total number of pulls to $\be_i$, we have
\begin{align}\label{eq:T_regret_constraint}
\cC_1 K^\alpha + \cC_2 \ge \sum_{k=1}^K \Exp[\Vst_0 - V_0^{\pi_k}] \ge \xi \sum_{k=1}^K\sum_{i=2}^d \Exp[ \Pr_{\pi_k}[ \{ \bz_k = \be_i \} \cup \{ \bz_k = -\be_i \}]] = \xi \sum_{i=2}^d \Exp[T(\be_i) + T(-\be_i)].
\end{align}

\paragraph{Applying the Transportation Lemma.}
Let $\Thetaalt \subseteq \Thetalb$ denote the set of $\btheta$ vectors such that the set of $\epsilon$-optimal distributions on $\btheta$ is disjoint from that on $\bthetast$. 
By the Transportation Lemma of \cite{kaufmann2016complexity} (Lemma 1 of \cite{kaufmann2016complexity}), for any $\btheta \in \Thetaalt$, assuming our stopping rule is $(\epsilon,\delta)$-correct, we then have
\begin{align*}
\sum_{\bz \in \cZ} \Exp[T(\bz)] \KL(\nu_{\bthetast,\bz} || \nu_{\btheta,\bz}) \ge \log \frac{1}{2.4 \delta}.
\end{align*}
Combining this with our constraint \eqref{eq:T_regret_constraint}, it follows that $\sum_{\bz \in \cZ} \Exp[T(\bz)] \ge \sum_{\bz \in \cZ} t_{\bz}$ for any $(t_{\bz})_{\bz \in \cZ}$ that is a feasible solution to
\begin{align}\label{eq:lb_opt1}
\min \sum_{\bz \in \cZ} t_{\bz} \quad \text{s.t.} \quad \min_{\btheta \in \Thetaalt} \sum_{\bz \in \cZ} t_{\bz} \KL(\nu_{\bthetast,\bz} || \nu_{\btheta,\bz}) \ge \log \frac{1}{2.4 \delta}, \cC_1 (\sum_{\bz \in \cZ} t_{\bz})^\alpha + \cC_2 \ge \xi \sum_{i=2}^d (t_{\be_i} + t_{-\be_i}).
\end{align}
We can rearrange the second constraint to
\begin{align*}
\frac{\xi \cC_1}{(\sum_{\bz \in \cZ} t_{\bz})^{1-\alpha}} + \frac{\xi \cC_2}{\sum_{\bz \in \cZ} t_{\bz}} \ge \frac{\sum_{i=2}^d (t_{\be_i} + t_{-\be_i})}{\sum_{\bz \in \cZ} t_{\bz}}.
\end{align*}
Assume that the optimal solution to \eqref{eq:lb_opt1} satisfies $\sum_{\bz \in \cZ} t_{\bz} \ge \frac{d}{200\Delta^2}$, then this constraint can be weakened to 
\begin{align*}
\zeta :=  \frac{\xi \cC_1}{(d/200\Delta^2)^{1-\alpha}} + \frac{200\xi \cC_2 \Delta^2}{d} \ge \frac{\sum_{i=2}^d t_{\bx_i}}{\sum_{\bz \in \cZ} t_{\bz}}.
\end{align*}
It follows then that if the optimal value to 
\begin{align}\label{eq:lb_opt2}
\min \sum_{\bz \in \cZ} t_{\bz} \quad \text{s.t.} \quad \min_{\btheta \in \Thetaalt} \sum_{\bz \in \cZ} t_{\bz} \KL(\nu_{\bthetast,\bz} || \nu_{\btheta,\bz}) \ge \log \frac{1}{2.4 \delta}, \zeta \ge \frac{\sum_{i=2}^d (t_{\be_i} + t_{-\be_i})}{\sum_{\bz \in \cZ} t_{\bz}}
\end{align}
is at least $\frac{d}{200 \Delta^2}$, then the optimal value to \eqref{eq:lb_opt1} is also at least $\frac{d}{200 \Delta^2}$, so our assumption that $\sum_{\bz \in \cZ} t_{\bz} \ge \frac{d}{200\Delta^2}$ will be justified.

For $\bz \neq \zst$, let $\btheta_{\bz}(\epsilon,t)$ denote the instance
\begin{align*}
\bthetast - \frac{(\by_{\bz}^\top \bthetast + 2\epsilon) \bAtil(t)^{-1} \by_{\bz}}{\by_{\bz}^\top \bAtil(t)^{-1} \by_{\bz}}
\end{align*}
for $\by_{\bz} = \zst - \bz$, $\bAtil(t) =  \sum_{\bz \in \cZ} \frac{t_{\bz}}{\sum_{\bz' \in \cZ} t_{\bz'}} \bz \bz^\top + \diag([\xi^2,\gamma^2/d,\ldots,\gamma^2/d])$. Note that $\by_{\bz}^\top \btheta_{\bz}(\epsilon,t) = - 2\epsilon < 0$. Since $\epsilon < \Delta/2$ by assumption, any $\epsilon$-optimal distribution $\omega$ for $\bthetast$ must place mass strictly more than $1/2$ on $\zst$. Since $\by_{\bz}^\top \btheta_{\bz}(\epsilon,t) = - 2\epsilon$, any $\epsilon$-optimal distribution on $\btheta_{\bz}(\epsilon,t)$ must place mass less than or equal to $1/2$ on $\zst$. Thus, it follows that the $\epsilon$-optimal distributions on $\bthetast$ and $\btheta_{\bz}(\epsilon,t)$ are disjoint. 
Furthermore, we have:
\begin{claim}\label{claim:alt_norm_bound}
For all $\bz \in \{ \bx_2, \ldots, \bx_d \}$ and $t$, assuming that $\epsilon < \Delta$ and $\xi$ and $\Delta$ are chosen as in \Cref{lem:lin_bandit_param}, we can bound $\| \btheta_{\bz}(\epsilon,t) \|_2 \le 11$. Furthermore, for any $\bv \in \cZ$, we have $|\inner{\btheta_{\bz}(\epsilon,t)}{\bv}| \le 1/4$.
\end{claim}
Since we have chosen $\Thetalb = \{ \btheta \in \R^d \ : \ \| \btheta \|_2 \le \frac{1}{4} ( \sqrt{d} - 2), |\inner{\btheta}{\bz}| \le 1/2, \forall \bz \in \cZ \}$ and assumed $d \ge 2116$, we have that $\btheta_{\bz}(\epsilon,t) \in \Thetalb$, so $\btheta_{\bz}(\epsilon,t) \in \Thetaalt$.
We can also bound:
\begin{claim}\label{claim:lb_kl_bound}
For all $\bz, \bv \in \cZ$ and $t$,
\begin{align*}
\KL(\nu_{\bthetast,\bv} || \nu_{\btheta_{\bz}(\epsilon,t),\bv}) \le 16 (\byz^\top \bthetast + \epsilon)^2 \frac{\byz^\top \bAtil(t)^{-1} \bv \bv^\top \bAtil(t)^{-1} \byz}{(\byz^\top \bAtil(t)^{-1} \byz)^2}.
\end{align*}
\end{claim}
This implies that, for any $t$,
\begin{align*}
 \sum_{\bv \in \cZ} t_{\bv} \KL(\nu_{\bthetast,\bv} || \nu_{\btheta_{\bz}(\epsilon,t),\bv}) & \le 16 \sum_{\bv \in \cZ} t_{\bv} (\byz^\top \bthetast + \epsilon)^2 \frac{\byz^\top \bAtil(t)^{-1} \bv \bv^\top \bAtil(t)^{-1} \byz}{(\byz^\top \bAtil(t)^{-1} \byz)^2} \\
 & = 16\sum_{\bv \in \cZ} t_{\bv} \cdot   (\byz^\top \bthetast + \epsilon)^2 \frac{\byz^\top \bAtil(t)^{-1} (\sum_{\bv \in \cZ} \frac{t_{\bv}}{\sum_{\bv' \in \cZ} t_{\bv'}} \bv \bv^\top) \bAtil(t)^{-1} \byz}{(\byz^\top \bAtil(t)^{-1} \byz)^2} \\
  & \le 16\sum_{\bv \in \cZ} t_{\bv} \cdot   (\byz^\top \bthetast + \epsilon)^2 \frac{\byz^\top \bAtil(t)^{-1} \bAtil(t) \bAtil(t)^{-1} \byz}{(\byz^\top \bAtil(t)^{-1} \byz)^2} \\
  & = \sum_{\bv \in \cZ} t_{\bv} \cdot   \frac{16 (\byz^\top \bthetast + \epsilon)^2}{\| \byz \|_{\bAtil(t)^{-1}}^2}
\end{align*}

Thus:
\begin{align*}
\eqref{eq:lb_opt2} & \ge \min \sum_{\bv \in \cZ} t_{\bv} \quad \text{s.t.} \quad \min_{\bz \in \{ \bx_1, \ldots, \bx_d \} } \sum_{\bv \in \cZ} t_{\bv} \KL(\nu_{\bthetast,\bv} || \nu_{\btheta_{\bz}(\epsilon,t),\bv}) \ge \log \frac{1}{2.4 \delta}, \zeta \ge \frac{\sum_{i=2}^d (t_{\be_i} + t_{-\be_i})}{\sum_{\bv \in \cZ} t_{\bv}} \\
& \ge \min \sum_{\bv \in \cZ} t_{\bv} \quad \text{s.t.} \quad \sum_{\bv \in \cZ} t_{\bv} \ge   \max_{\bz \in \{ \bx_1, \ldots, \bx_d \}}\frac{\| \byz \|_{\bAtil(t)^{-1}}^2}{16 (\byz^\top \bthetast + \epsilon)^2} \cdot \log \frac{1}{2.4 \delta}, \zeta \ge \frac{\sum_{i=2}^d (t_{\be_i} + t_{-\be_i})}{\sum_{\bv \in \cZ} t_{\bv}} \\
& = \inf_{\lambda \in \simptil}  \max_{\bz \in \{ \bx_1, \ldots, \bx_d \}}\frac{\| \byz \|_{\bAtil(\lambda)^{-1}}^2}{16 (\byz^\top \bthetast + \epsilon)^2} \cdot \log \frac{1}{2.4 \delta}
\end{align*}
where $\bAtil(\lambda) = \sum_{\bz \in \cZ} \lambda_{\bz} \bz \bz^\top$ and $\simptil = \{ \lambda \in \simplex_\cZ \ : \ \zeta \ge \sum_{i=2}^d (\lambda_{\be_i} + \lambda_{-\be_i}) \}$. We can further lower bound this by
\begin{align*}
& \ge \inf_{\lambda \in \simptil}  \max_{i \ge 2}\frac{\| \zst - \bx_i \|_{\bAtil(\lambda)^{-1}}^2}{16 ((\zst - \bx_i)^\top \bthetast + \epsilon)^2} \cdot \log \frac{1}{2.4 \delta} \\
& = \inf_{\lambda \in \simptil}  \max_{i \ge 2}\frac{\| \Delta \be_1 - \gamma \be_i \|_{\bAtil(\lambda)^{-1}}^2}{16 (\Delta + \epsilon)^2} \cdot \log \frac{1}{2.4 \delta} .
\end{align*}
By \Cref{lem:diag_design_order}, we have
\begin{align*}
\inf_{\lambda \in \simptil} & \max_{i \ge 2} \| \Delta \be_1 - \gamma \be_i \|_{\bAtil(\lambda)^{-1}}^2 \\
& \ge \inf_{\lambda \in \simplex_{d}}  \max_{i \ge 2} (\Delta \be_1 - \gamma \be_i)^\top  \Big ( 2 \xi^2 \be_1 \be_1^\top + 2 \max \{\zeta,\gamma^2\} \lambda_i \be_i \be_i^\top + \diag([\xi^2,\gamma^2/d,\ldots,\gamma^2/d]) \Big )^{-1} (\Delta \be_1 - \gamma \be_i) \\
& \ge \frac{\Delta^2}{3 \xi^2} +  \inf_{\lambda \in \simplex_{d}} \max_{i \ge 2} \frac{1}{2 \lambda_i + 1/d}
\end{align*}
where in the final equality we have used that $\zeta \le \gamma^2$. However, this is clearly minimized by choosing $\lambda_i = 1/(d-1)$, which gives a lower bound of 
\begin{align*}
\frac{1}{2/(d-1) + 1/d} \ge \frac{d-1}{3} .
\end{align*}
Putting all of this together, we have shown that any feasibly solution $(t_{\bz})_{\bz \in \cZ}$ to \eqref{eq:lb_opt1} must satisfy
\begin{align*}
\sum_{\bz \in \cZ} t_{\bz} \ge \frac{d-1}{48 (\Delta + \epsilon)^2} \cdot \log \frac{1}{2.4 \delta} \ge \frac{d}{200 \Delta^2} \cdot \log \frac{1}{2.4 \delta}
\end{align*}
where the second inequality uses that $\epsilon \le \Delta$ and $d \ge 2116$.
Using that any feasible solution to \eqref{eq:lb_opt1} lower bounds $\sum_{\bz \in \cZ} \Exp[T(\bz)]$ and noting that our lower bound has justified our assumption that $\sum_{\bz \in \cZ} t_{\bz} \ge d/200 \Delta^2$, gives the result.
\end{proof}

\begin{lem}\label{lem:lin_bandit_param}
Take some $\Delta > 0$ satisfying:
\begin{align*}
\Delta \le \min \left \{ \frac{1}{10000 d^4}, \sqrt{\frac{1}{4 \cdot 10816 \cC_2}}, \frac{1}{\sqrt{200}} \left ( \frac{1}{10816 d^\alpha \cC_1} \right )^{\frac{1}{2(1-\alpha)}} \right \}
\end{align*}
and set
\begin{align*}
\xi = \frac{1}{52 d}, \qquad \gamma = \frac{1}{52 \sqrt{d}} .
\end{align*}
Then this choice of $\xi,\gamma,\delta$ satisfies \eqref{eq:lb_instance_constraints} and, furthermore, $\| \bz \|_2 \le 1$ for all $\bz \in \cZ$.  
\end{lem}
\begin{proof}
We first fix $\xi = \frac{1}{52 d}$. To satisfy \eqref{eq:lb_instance_constraints}, it then suffices to choose $\gamma$ and $\Delta$ satisfying
\begin{align*}
\frac{\gamma}{\sqrt{d}} \le \frac{1}{52d}, \quad \sqrt{\Delta} \le \frac{1}{52d}, \quad \frac{\cC_1}{52 d (d/200 \Delta^2)^{1-\alpha}} + \frac{4 \cC_2 \Delta^2}{d^2} \le \gamma^2, \quad \Delta \le \gamma^2. 
\end{align*}
Thus, if
\begin{align*}
& \Delta \le \frac{1}{10000 d^4}, \quad \Delta \le \sqrt{\frac{1}{4 \cdot 10816 \cC_2}}, \quad \Delta \le \frac{1}{\sqrt{200}} \left ( \frac{1}{10816 d^\alpha \cC_1} \right )^{\frac{1}{2(1-\alpha)}},
\end{align*}
some algebra shows that it suffices that we take $\gamma = \frac{1}{52 \sqrt{d}}$.
The norm bound follows by our choice of $\xi$ and since $\xi \ge \gamma/\sqrt{d} \ge \sqrt{\Delta}$.
\end{proof}

\subsubsection{Additional Proofs}

\begin{proof}[Proof of \Cref{claim:alt_norm_bound}]
We first show that $| \inner{\btheta_{\bz}(\epsilon,t)}{\bv} |  \le 1/4$ for all $\bz \in \{ x_2,\ldots,x_d \}$ and $\bv \in \cZ$.

Let $\simptil = \{ \lambda \in \simplex_\cZ \ : \ \zeta \ge \sum_{i=2}^d (\lambda_{\be_i} + \lambda_{-\be_i}) \}$. By \Cref{lem:diag_design_order},
\begin{align*}
\byz^\top \bAtil(t)^{-1} \byz & \ge \inf_{\lambda \in \simptil} \byz^\top \Big ( 2 \sum_{\bz' \in \cZ} \lambda_{\bz'} \diag([(\bz')^2]) + \diag([\xi^2,\gamma^2/d,\ldots,\gamma^2/d]) \Big )^{-1} \byz \\
& \ge  \byz^\top \Big ( 2 \xi^2 \be_1 \be_1^\top + 2 \max \{\zeta,\gamma^2\} \be_i \be_i^\top + \diag([\xi^2,\gamma^2/d,\ldots,\gamma^2/d]) \Big )^{-1} \byz \\
& \ge \Delta^2 \frac{3}{\xi^2}  + \frac{\gamma^2}{2 \max \{ \zeta, \gamma^2 \} + \gamma^2/d} \\
& \ge \frac{1}{3}
\end{align*}
where the last inequality follows from our assumption that $\zeta \le \gamma^2$.
We can also upper bound
\begin{align*}
\bv^\top \bAtil(t)^{-1} \by_{\bz} \le \bv^\top \diag([\xi^2,\gamma^2/d,\ldots,\gamma^2/d])^{-1} \byz \le \frac{\Delta}{\xi^2} + \frac{d \gamma}{\gamma^2} \le \frac{2d}{\gamma}.
\end{align*}
This gives an upper bound of
\begin{align*}
| \inner{\btheta_{\bz}(\epsilon,t)}{\bv} | \le \xi + (\Delta + \epsilon) \frac{6 d}{\gamma} \le \xi + \frac{12 d \Delta}{\sqrt{\Delta}} \le 13 d \xi.
\end{align*}
By our assumption that $\xi \le \frac{1}{52d}$, it follows that $| \inner{\btheta_{\bz}(\epsilon,t)}{\bv} |  \le 1/4$ for all $\bv \in \cZ$.

To prove the norm bound on $\btheta_{\bz}(\epsilon,t)$, consider $\bv = \be_i$, $i \neq 1$. Similar to the above calculation, we have
\begin{align*}
| \inner{\btheta_{\bz}(\epsilon,t)}{\bv}| = |[\btheta_{\bz}(\epsilon,t)]_i| \le (\Delta + \epsilon) \cdot \frac{6 d}{\gamma} \le 12 d \sqrt{\Delta}
\end{align*}
where the last inequality follows by taking $\epsilon < \Delta$, and using that by assumption $\gamma \ge \sqrt{\Delta}$. Now letting $\bv = \xi \be_1 \in \cZ$, we have
\begin{align*}
| \inner{\btheta_{\bz}(\epsilon,t)}{\bv}| = \xi |[\btheta_{\bz}(\epsilon,t)]_1| \le \xi + 12 d \sqrt{\Delta}.
\end{align*}
Thus, it follows that
\begin{align*}
\| \btheta_{\bz}(\epsilon,t) \|_2 & = \sqrt{\sum_{i=1}^2 [\btheta_{\bz}(\epsilon,t)]_i^2} \\
& \le \sqrt{ 2 + 288 d^2 \Delta/\xi^2 + 144 d^3 \Delta)} \\
& \le \sqrt{2} + \frac{d}{\xi} \sqrt{288 \Delta} + d^{3/2} \sqrt{144 \Delta} .
\end{align*}
Now, if $\xi$ and $\Delta$ are chosen as in \Cref{lem:lin_bandit_param}, we have
\begin{align*}
\| \btheta_{\bz}(\epsilon,t) \|_2 & \le \sqrt{2} + 52 d^2 \sqrt{288 \Delta} + d^{3/2} \sqrt{144 \Delta}  \\
& \le \sqrt{2} + 52 \sqrt{288/10000} + \sqrt{144/10000} \\
& \le 11.
\end{align*}

\end{proof}

\begin{proof}[Proof of \Cref{claim:lb_kl_bound}]

By Lemma D.2 of \cite{wagenmaker2022reward}, as long as $\inner{\btheta_{\bz}(\epsilon,t)}{\bv} + 1/2 \in (0,1)$ and $\inner{\bthetast}{\bv} + 1/2 \in (0,1)$, which will be the case by the definition of $\bthetast$ and since $| \inner{\btheta_{\bz}(\epsilon,t)}{\bv} |  \le 1/4$ as noted above, we have 
\begin{align*}
\KL(\nu_{\bthetast,\bv} || \nu_{\btheta_{\bz}(\epsilon,t),\bv}) \le \frac{\inner{\btheta_{\bz}(\epsilon,t) - \bthetast}{\bv}^2}{ (\inner{\btheta_{\bz}(\epsilon,t)}{\bv} + 1/2)(1/2 - \inner{\btheta_{\bz}(\epsilon,t)}{\bv})}.
\end{align*}
Using what we have just shown, we can upper bound this as
\begin{align*}
 \frac{\inner{\btheta_{\bz}(\epsilon,t) - \bthetast}{\bv}^2}{ (\inner{\btheta_{\bz}(\epsilon,t)}{\bv} + 1/2)(1/2 - \inner{\btheta_{\bz}(\epsilon,t)}{\bv})} & \le \frac{\inner{\btheta_{\bz}(\epsilon,t) - \bthetast}{\bv}^2}{ (-1/4 + 1/2)(1/2 - 1/4)} \\
& = 16 \inner{\btheta_{\bz}(\epsilon,t) - \bthetast}{\bv}^2.
\end{align*}
By our choice of $\btheta_{\bz}(\epsilon,t)$, this is equal to:
\begin{align*}
16 (\byz^\top \bthetast + \epsilon)^2 \frac{\byz^\top \bAtil(t)^{-1} \bv \bv^\top \bAtil(t)^{-1} \byz}{(\byz^\top \bAtil(t)^{-1} \byz)^2}
\end{align*}
which completes the proof.
\end{proof}

\begin{lem}\label{lem:diag_design_order}
\begin{align*}
\sum_{\bz \in \cZ}\lambda_{\bz} \bz \bz^\top \preceq 2 \sum_{\bz \in \cZ} \lambda_{\bz} \diag(\bz^2).
\end{align*}
\end{lem}
\begin{proof}
This follows since every $\bz \in \cZ$ has at most two non-zero entries, and since $(a \bx + b \by) (a \bx + b \by)^\top \preceq 2 a^2 \bx \bx^\top + 2 b^2 \by \by^\top$. 
\end{proof}

\subsection{Mapping to Linear MDPs}
We can map this linear bandit (with parameters chose as in \Cref{lem:lin_bandit_param}) to a linear MDP with state space $\cS = \{ s_0, s_1, \sbar_2, \ldots, \sbar_{d+1} \}$, action space $\cA = \cZ \cup \{ \be_{d+1}/2 \}$, parameters
\begin{align*}
\btheta_1 = \bm{0}, & \quad \btheta_2 =  \be_1 \\
\bmu_1(s_1) =  [2 \btheta, 1], & \quad \bmu_1(\sbar_i) = \frac{1}{d} [-2 \btheta, 1 ],
\end{align*}
for some $\btheta$ (in particular, we are interested in $\btheta = \bthetast$),
 and feature vectors
\begin{align*}
& \bphi(s_0,\be_{d+1}) = \be_{d+1}/2, \quad \bphi(s_0,\bz) = [\bz/2,1/2], \quad \forall \bz \in \cZ \\
& \bphi(s_1, \bz) = \be_1 ,  \quad \bphi(\sbar_i,\bz) =  \be_i, i \ge 2, \quad \forall \bz \in \cA.
\end{align*}
Note that, if we take action $\bz$ in state $s_0$, our expected episode reward is
\begin{align*}
P_1(s_1|s_0,\bz) \cdot 1  + \sum_{i=2}^{d+1}  P_1(\sbar_i|s_0,\bz) \cdot 0 = \inner{\btheta}{\bz} + 1/2
\end{align*}
since we always acquire a reward of 1 in any state $s_1$, and a reward of 0 in any state $\sbar_i$, and the reward distribution is Bernoulli.

\begin{lem}\label{lem:lb_valid_mdp}
For any $\btheta \in \Thetalb$,
the MDP constructed above is a valid linear MDP as defined in \Cref{defn:linear_mdp}.
\end{lem}
\begin{proof}
For $\bz \in \cZ$ we have,
\begin{align*}
P_1(s_1 | s_0, \bz) & = \inner{\bphi(s_0,\bz)}{\bmu_1(s_1)} = \inner{\btheta}{\bz} + 1/2 \ge 0 \\
P_1(\sbar_i | s_0, \bz) & = \inner{\bphi(s_0,\bz)}{\bmu_1(\sbar_i)} = \frac{1}{d} ( -\inner{\btheta}{\bz} + 1/2) \ge 0
\end{align*}
where the inequality follows since $|\inner{\btheta}{\bz}| \le 1/2$ for all $\bz \in \cZ$ by the definition of $\Thetalb$.
In addition,
\begin{align*}
P_1(s_1|s_0,\bz)  + \sum_{i=2}^{d+1}  P_1(\sbar_i|s_0,\bz) =   \inner{\btheta}{\bz} + 1/2 + d \cdot \frac{1}{d} ( -\inner{\btheta}{\bz} + 1/2) = 1
\end{align*}
Thus, $P_1(\cdot|s_0,\bz)$ is a valid probability distribution for $\bz \in \cZ$. A similar calculation shows the same for $\bz = \be_{d+1}/2$. 

It remains to check the normalization bounds. Clearly, by our construction of $\cZ$, $\| \bphi(s,a) \|_2 \le 1$ for all $s$ and $a$. It is also obvious that $\| \btheta_0 \|_2 \le \sqrt{d}$ and $\| \btheta_1 \|_2 \le \sqrt{d}$. Finally,
\begin{align*}
\| | \bmu_1(\cS) | \|_2 = \left \| \sum_{s \in \cS \backslash s_0} |\bmu_1(s)| \right \|_2 = \| [2 \btheta, 1] + d \cdot \frac{1}{d} [2 \btheta, 1] \|_2 \le 4 \| \btheta \|_2 + 2 \le \sqrt{d}
\end{align*}
where the last inequality follows from the definition of $\Thetalb$. 
Thus, all normalization bounds are met, so this is a valid linear MDP. 
\end{proof}

\begin{proof}[Proof of \Cref{prop:easy_instance_upper}]
If we assume that the learner has prior access to the feature vectors, and also knows this is a linear MDP, then, even with no knowledge of the dynamics, we can guarantee an optimal policy is contained in the set of policies $\pi^{\bz,\bz'}$ defined as:
\begin{align*}
\pi^{\bz,\bz'}_1(s_0) = \bz, \pi^{\bz,\bz'}_2(s_0) = \bz', \pi^{\bz,\bz'}_h(s_1) = \xi \be_1, \pi^{\bz,\bz'}_h(\sbar_i) = \xi \be_1
\end{align*}
This holds because in states $s_1$ and $\sbar_i$, the performance of each action is identical since the feature vectors are identical, so it doesn't matter which action we choose in these states. In this case, we can bound $|\Pi| \le |\cZ|^2 \le 4 d^2$. 

Now, for $\bz \in \cA$, $\bz \neq \be_{d+1}/2$, we have
\begin{align*}
& \bphi_{\pi^{\bz,\bz'},1} = [\bz/2,1/2] \\
& \bphi_{\pi^{\bz,\bz'},2} = (\inner{\bthetast}{\bz} + 1/2)\be_1 + \frac{1}{d} (-\inner{\bthetast}{\bz} + 1/2) \sum_{i \ge 2} \be_i
\end{align*}
and if $\bz = \be_{d+1}/2$, $\bphi_{\pi^{\bz,\bz'},1}  = \be_{d+1}/2$, $\bphi_{\pi^{\bz,\bz'},2} = \be_1/2 + \frac{1}{2d} \sum_{i \ge 2} \be_i$. 
Let $\piexp$ be the policy that plays action $\be_2$ in state $s_0$ at step $h=1$. Then,
\begin{align*}
\bLambda_{\piexp,2} = \frac{1}{2} \be_1 \be_1  + \frac{1}{2d} \sum_{i \ge 3}  \be_i \be_i^\top .
\end{align*}
Since $\inner{\bthetast}{\bz} \le \cO(1/d)$ and $[\bz]_1 \le \cO(1/d)$ for all $\bz$ by construction, it follows that we can bound, for all $\bz,\bz'$,
\begin{align*}
\|  \bphi_{\pi^{\bz,\bz'},2} \|_{\bLambda_{\piexp,2}^{-1}}^2 = \cO \left ( 1 + \sum_{i \ge 2} \frac{1}{d^2} \cdot d \right ) = \cO(1)
\end{align*}
so 
\begin{align*}
\inf_{\piexp} \max_{\pi \in \Pi} \frac{ \| \bphi_{\pi,2} \|_{\bSigma_{\piexp,2}^{-1}}^2}{\max \{ \Vst_0 - \Vpi_0, \Delmin^\Pi, \epsilon \}^2} \le \cO(1/\epsilon^2).
\end{align*}

Now let $\piexp$ be the policy that, at step $h=1$, plays $\xi \be_1$ with probability 1/4, $\be_{d+1}$ with probability 1/4, and plays $\be_i$ with probability $\frac{1}{4(d-1)}$ for $i \ge \{2,\ldots,d\}$. In this setting, we have 
\begin{align*}
\bLambda_{\piexp,1} = \frac{1}{4} \xi^2 \be_1 \be_1^\top + \frac{1}{4} \be_{d+1} \be_{d+1} + \frac{1}{4(d-1)} \sum_{i \in \{ 2,\ldots,d \}}  \be_i \be_i^\top.
\end{align*}

Note that $\Vst_0 - V_0^{\pi^{\bz,\bz'}} = \xi - \inner{\bthetast}{\bz}$, so for $\bz = \be_2,\ldots,\be_{d+1}$, we have $\Vst_0 - V_0^{\pi^{\bz,\bz'}} = \xi = \cO(1/d)$, while for $\bz = \xi \be_1, \bx_2,\ldots,\bx_d$, we have $\Vst_0 - V_0^{\pi^{\bz,\bz'}} = \Delta$. 

It's easy to see that for $\bz = \be_2,\ldots,\be_{d+1}$, we have $\| \bphi_{\pi^{\bz,\bz'},1} \|_{\bLambda_{\piexp,1}^{-1}} \le \cO(d)$, and for $\bz = \xi \be_1, \bx_2,\ldots,\bx_d$, $\| \bphi_{\pi^{\bz,\bz'},1} \|_{\bLambda_{\piexp,1}^{-1}} \le \cO(1 + d \gamma^2) = \cO(1)$. Combining these bounds with the gap values, we conclude that 
\begin{align*}
\inf_{\piexp} \max_{\pi \in \Pi} \frac{ \| \bphi_{\pi,1} \|_{\bSigma_{\piexp,1}^{-1}}^2}{\max \{ \Vst_0 - \Vpi_0, \Delmin^\Pi, \epsilon \}^2} \le \cO(1/\epsilon^2 + \poly(d)).
\end{align*}
The result then follows by \Cref{thm:complexity} and \Cref{lem:lb_valid_mdp}. 
\end{proof}

\paragraph{Lower bounding the performance of low-regret algorithms.}

Assume that we have access to the linear bandit instance constructed in \Cref{sec:lin_bandit_hard} with parameters chosen as in \Cref{lem:lin_bandit_param}. That is, at every timestep $t$ we can choose an arm $\bz_t \in \cZ$ and obtain and observe reward $y_t \sim \bern(\inner{\bthetast}{\bz_t} + 1/2)$. Using the mapping up, we can use this bandit to simulate a linear MDP as follows:
\begin{enumerate}
\item Start in state $s_0$ and choose any action $\bz_t \in \cA$
\item Play action $\bz_t$ in our linear bandit. If reward obtained is $y_t = 1$, then in MDP transition to any of the states $s_1$. If reward obtained is $y_t = 0$ transition to any of the states $\sbar_2,\ldots,\sbar_{d+1}$, each with probability $1/d$. If the chosen action was $\bz_t = \be_{d+1}/2$, then play any action in the linear bandit and transition to state $s_1$ with probability 1/2 and $\sbar_2,\ldots,\sbar_{d+1}$ with probability $1/2d$, regardless of $y_t$
\item Take any action in the state in which you end up, and receive reward of 1 if you are in $s_1$, and reward of 0 if you are in $\sbar_2,\ldots,\sbar_{d+1}$.
\end{enumerate}
Note that this MDP has precisely the transition and reward structure as the MDP constructed above.

\begin{lem}\label{lem:lb_to_lmdp_alg}
Any algorithm that is $(\epsilon,\delta)$-PAC for linear MDPs satisfying \Cref{defn:linear_mdp} is also $(\epsilon,\delta)$-PAC for linear bandits constructed as above, for any $\btheta \in \Thetalb$.
\end{lem}
\begin{proof}
Consider some policy $\pi$ that is $\epsilon$-optimal on the linear MDP constructed as above. 
Note that, for all $\btheta$, we have $\max_{\bz \in \cZ} \inner{\bz}{\btheta} \ge 0$, so that the optimal action in the linear MDP is never $\be_{d+1}/2$.
If $\pi$ is $\epsilon$-optimal in the linear MDP, this then implies that
\begin{align*}
& \sum_{\bz \in \cZ} \pi_1(\bz | s_0) ( \inner{\bz}{\btheta} + 1/2) + \pi_1(\be_{d+1}/2|s_0)/2 \ge \max_{\bz \in \cZ} \inner{\bz}{\btheta} + 1/2 - \epsilon \\
\implies & \inner{ \sum_{\bz \in \cZ} \pi_1(\bz | s_0) \bz}{\btheta} \ge \max_{\bz \in \cZ} \inner{\bz}{\btheta} - \epsilon.
\end{align*}
Thus, it follows that the distribution $\pi_1(\bz|s_0)$ is $\epsilon$-optimal for the linear bandit with parameter $\btheta$, implying that any policy $\epsilon$-optimal on the linear MDP constructed above can be used to identify an $\epsilon$-optimal solution to the linear bandit.

By \Cref{lem:lb_valid_mdp}, for any $\btheta \in \Thetalb$, the linear MDP constructed above with parameter $\btheta$ is a linear MDP satisfying \Cref{defn:linear_mdp}. Thus, it follows that any algorithm $(\epsilon,\delta)$-PAC on linear MDPs satisfying \Cref{defn:linear_mdp} must also be $(\epsilon,\delta)$-PAC on our linear bandit.

\end{proof}

\begin{proof}[Proof of \Cref{prop:easy_instance_lower}]
By \Cref{lem:lb_to_lmdp_alg}, we know that any algorithm which is $(\epsilon,\delta)$-PAC for linear MDPs satisfying \Cref{defn:linear_mdp} is also $(\epsilon,\delta)$-PAC on the set of linear bandit instances with arm set $\cZ$ as constructed above, and parameter in $\Thetalb$. As the lower bound given in \Cref{lem:lin_band_lower_bound} applies to any algorithm that is $(\epsilon,\delta)$-PAC on linear bandits with parameters in $\Thetalb$, it follows that it also applies to a $(\epsilon,\delta)$-PAC linear MDP algorithm, which proves the result.
\end{proof}

\end{document}